\documentclass[11pt, oneside]{article}
\usepackage{geometry}
\usepackage{graphicx}	
\usepackage{amssymb,graphics}
\usepackage{amsfonts,latexsym,amsthm,amssymb,amsmath,amscd,euscript}
\usepackage{bbm} 
\usepackage{framed}
\usepackage{fullpage}
\usepackage{mathrsfs}
\usepackage{hyperref}
\usepackage{comment}
\usepackage{xspace}
\usepackage{algpseudocode}
\usepackage{algorithm}

\hypersetup{
  colorlinks=true,
  linkcolor=blue,
  filecolor=blue,
  citecolor = black,      
  urlcolor=cyan,
}
\usepackage{accents}
\usepackage{setspace}
\hypersetup{colorlinks=true,citecolor=blue,urlcolor=black,linkbordercolor={1 0 0}}
\usepackage{tikz-cd}

\usepackage{kz_style}

\newcommand{\kz}[1]{}
\newcommand{\noah}[1]{}
\newcommand{\cd}[1]{}

\makeatletter
\newtheorem*{rep@theorem}{\rep@title}
\newcommand{\newreptheorem}[2]{%
\newenvironment{rep#1}[1]{%
 \def\rep@title{#2 \ref{##1}}%
 \begin{rep@theorem}}%
 {\end{rep@theorem}}}
\makeatother

\makeatletter
\newcommand\xlabel[2][]{\phantomsection\def\@currentlabelname{#1}\label{#2}}
\makeatother

\theoremstyle{plain}
\newtheorem{theorem}{Theorem}
\newreptheorem{theorem}{Theorem}
\newtheorem{lemma}[theorem]{Lemma}
\newtheorem{corollary}[theorem]{Corollary}
\newtheorem{conjecture}[theorem]{Conjecture}

\newtheorem{claim}[theorem]{Claim}

\theoremstyle{definition}
\newtheorem{definition}{Definition}
\newtheorem{defn}[definition]{Definition}

\numberwithin{theorem}{section}
\numberwithin{definition}{section}

\newcommand{\nc}{\newcommand}
\nc{\DMO}{\DeclareMathOperator}
\newcount\Comments  %
\DeclareMathOperator*{\argmax}{arg\,max}

\DMO{\prox}{prox}
\DMO{\Span}{span}
\DMO{\UCB}{UCB}
\DMO{\LCB}{LCB}
\nc{\dg}{\dagger}
\nc{\WP}{\texttt{W}}
\nc{\VP}{\texttt{V}}
\nc{\KMAX}{K(S+1)}

\nc{\unif}{\mu_{\rm unif}}
\nc{\cover}{{\rm cover}}
\nc{\algname}{\texttt{SPoCMAR}\xspace}
\nc{\PWSNE}{\text{PWSNE}}
\nc{\WSNE}{\text{WSNE}}

\nc{\gamvec}{\gamma}
\nc{\til}{\widetilde}
\nc{\td}{\tilde}
\nc{\wh}{\widehat}
\nc{\todo}[1]{\ifnum\Comments=1 {\color{red}  [TODO: #1]}\fi}
\nc{\old}[1]{\ifnum\Comments=1 {\color{brown}  [OLD: #1]}\fi}
\nc{\BP}{\mathbb{P}}
\nc{\BR}{\mathbb{R}}
\nc{\PPAD}{\textsf{PPAD}\xspace}
\nc{\TFNP}{\textsf{TFNP}\xspace}
\nc{\PSPACE}{\textsf{PSPACE}\xspace}
\nc{\NP}{\textsf{NP}\xspace} 
\nc{\coNP}{\textsf{co-NP}\xspace} 
\nc{\EOTL}{\texttt{EOTL}\xspace}
\nc{\EST}{\texttt{EstVisitation}\xspace}
\nc{\sinkz}{s_{\rm sink}^0}
\nc{\sinko}{s_{\rm sink}^1}

\nc{\fools}[3]{\MF_{#3}({#1}, {#2})}
\nc{\fool}[2]{\MF({#1},{#2})}
\nc{\clip}[2]{{\rm clip}\left[ \left. {#1} \right| {#2} \right]}
\nc{\imax}{\omega}
\DMO{\conv}{conv}
\nc{\st}{\star}
\nc{\MI}{\mathcal{I}}
\nc{\lng}{\langle}
\nc{\MH}{\mathcal{H}}
\nc{\rng}{\rangle}
\DMO{\OOPT}{opt}
\nc{\dopt}[2]{\ell_{\OOPT}({#1},{#2})}
\nc{\grad}{\nabla}
\nc{\MG}{\mathcal{G}}
\nc{\MP}{\mathcal{P}}
\nc{\MC}{\mathcal{C}}
\nc{\TT}{\mathbb{T}}
\nc{\TTmax}{\TT_{\max}}
\DMO{\Reg}{Reg}
\DMO{\Ham}{Ham}
\DMO{\Gap}{Gap}
\DMO{\GD}{GD}
\DMO{\GDA}{GDA}
\DMO{\EG}{EG}
\DMO{\OGDA}{OGDA}
\DMO{\Unif}{Unif}
\nc{\ul}{\underline}
\nc{\ol}{\overline}
\nc{\Qu}{\ul{Q}}
\nc{\Qo}{\ol{Q}}
\nc{\Ro}{\ol{R}}
\nc{\Vu}{\ul{V}}
\nc{\Vo}{\ol{V}}
\nc{\RanQ}{\Delta Q}
\nc{\RanV}{\Delta V}
\nc{\clipQ}{\Delta \breve{Q}}
\nc{\frzQ}{\Delta \mathring{Q}}
\nc{\clipV}{\Delta \breve{V}}
\nc{\clipdelta}{\breve{\delta}}
\nc{\cliptheta}{\breve{\theta}}
\nc{\delmin}{\Delta_{{\rm min}}}
\nc{\delmins}[1]{\Delta_{{\rm min},{#1}}}
\nc{\gapfinal}[1]{\max \left\{ \frac{\frzQ_{{#1}}^{k^\st}(x,a)}{2H}, \frac{\delmin}{4H} \right\}}
\nc{\post}[2]{R({#1}; {#2})}
\nc{\posts}[3]{R_{#3}({#1}; {#2})}

\nc{\algnst}[1]{\begin{align*}#1\end{align*}}
\nc{\algn}[1]{\begin{align}#1\end{align}}
\nc{\matx}[1]{\left(\begin{matrix}#1\end{matrix}\right)}

\nc{\nuu}{\nu}

\nc{\bel}[1]{\mathbf{b}({#1})}
\nc{\nbel}[1]{\bar{\mathbf{b}}({#1})}
\nc{\sbel}[2]{\mathbf{b}'_{#1}({#2})}
\nc{\nsbel}[2]{\bar{\mathbf{b}}'_{#1}({#2})}

\nc{\bone}{\mathbf{1}}
\nc{\sy}{y}
\nc{\sx}{x}

\nc{\MO}{\mathcal O}
\nc{\MU}{\mathcal{U}}
\nc{\ME}{\mathcal{E}}
\nc{\MN}{\mathcal{N}}
\nc{\MK}{\mathcal{K}}
\nc{\MV}{\mathcal{V}}
\nc{\MS}{\mathcal{S}}
\nc{\MT}{\mathcal{T}}
\nc{\BF}{\mathbb F}
\nc{\BQ}{\mathbb Q}
\nc{\MX}{\mathcal{X}}
\nc{\MA}{\mathcal{A}}
\nc{\MD}{\mathcal{D}}
\nc{\MB}{\mathcal{B}}
\nc{\MZ}{\mathcal{Z}}
\nc{\MJ}{\mathcal{J}}
\nc{\MW}{\mathcal{W}}
\nc{\MR}{\mathcal{R}}
\nc{\MY}{\mathcal{Y}}
\nc{\BZ}{\mathbb Z}
\nc{\BN}{\mathbb N}
\nc{\ep}{\epsilon}
\nc{\gapfn}[1]{\varepsilon_{#1}}
\nc{\ggapfn}[2]{\varphi_{#1}({#2})}
\nc{\epsahk}{\gapfn{0}}
\nc{\BH}{\mathbb H}
\nc{\BG}{\mathbb{G}}
\nc{\MF}{\mathcal{F}}
\nc{\One}{\mathbbm{1}}
\nc{\bOne}{\mathbf{1}}
\nc{\Aopt}{\mathcal{A}^{\rm opt}}
\nc{\Amul}{\mathcal{A}^{\rm mul}}

\nc{\SP}{\mathsf P}
\nc{\SQ}{\mathsf Q}
\nc{\Gcircuit}{\texttt{GCircuit}\xspace}

\nc{\DO}{\accentset{\circ}{\D}}
\nc{\mf}{\mathfrak}
\nc{\mfp}{\mathfrak{p}}
\nc{\mfq}{\mf{q}}
\nc{\Sp}{\mbox{Spec}}
\nc{\Spm}{\mbox{Specm}}
\nc{\hookuparrow}{\mathrel{\rotatebox[origin=c]{90}{$\hookrightarrow$}}}
\nc{\hookdownarrow}{\mathrel{\rotatebox[origin=c]{-90}{$\hookrightarrow$}}}
\nc{\hra}{\hookrightarrow}
\nc{\tra}{\twoheadrightarrow}
\nc{\sgn}{{\rm sgn}}
\nc{\aut}{{\rm Aut}}
\nc{\Hom}{{\rm Hom}}
\nc{\img}{{\rm Im}}
\DMO{\id}{Id}
\DMO{\supp}{supp}
\DMO{\KL}{KL}
\nc{\kld}[2]{\KL({#1}||{#2})}
\nc{\ren}[2]{D_2({#1}||{#2})}
\nc{\chisq}[2]{\chi^2({#1}||{#2})}
\nc{\tvd}[2]{\left\| {#1} - {#2} \right\|_1}
\nc{\hell}[2]{H^2({#1}, {#2})}
\DMO{\BSS}{BSS}
\DMO{\BES}{BES}
\DMO{\BGS}{BGS}
\DMO{\poly}{poly}
\DMO{\sink}{sink} 
\nc{\fp}[1]{\MP_1({#1})}
\nc{\BO}{\mathbb{O}}
\nc{\BT}{\mathbb{T}}

\nc{\Gradient}{\nabla}
\DMO{\diag}{diag}
\nc{\norm}[1]{\left \lVert #1 \right \rVert}
\nc{\EE}{\mathbb{E}}
\nc{\contr}[1]{{\texttt{cr}({#1})}}
\nc{\ag}{\contr}

\renewcommand{\Pr}{\mathbb{P}}
\nc{\E}{\mathbb{E}}
\renewcommand{\ra}{\rightarrow}

\title{The Complexity of Markov Equilibrium in Stochastic Games}
\author{Constantinos Daskalakis\thanks{Email: \texttt{costis@csail.mit.edu}.  Supported by NSF Awards CCF-1901292, DMS-2022448 and DMS2134108, by a Simons Investigator Award, by the Simons Collaboration on the Theory of Algorithmic Fairness, by a DSTA grant, and by the DOE PhILMs project (No. DE-AC05-76RL01830).} \and Noah Golowich\thanks{Email: \texttt{nzg@mit.edu}. Supported by  Fannie \& John Hertz Foundation Fellowship and NSF Graduate Fellowship.} \and Kaiqing Zhang\thanks{Email: \texttt{kaiqing@mit.edu}. Supported by Simons-Berkeley Fellowship and  MIT-DSTA grant 031017-00016.}}

\begin{document}
\maketitle

\begin{abstract}
We show that computing approximate stationary Markov coarse correlated equilibria (CCE) in general-sum stochastic games is computationally intractable, even when there are two players, the game is turn-based, the discount factor is an absolute constant, and the approximation is an absolute constant. Our intractability results stand in sharp contrast to normal-form games where exact CCEs are efficiently computable. A fortiori, our results imply that there are no efficient algorithms for learning stationary Markov CCE policies in multi-agent reinforcement learning (MARL), even when the interaction is two-player and turn-based, and both the discount factor and the desired approximation of the learned policies is an absolute constant. In turn, these results stand in sharp contrast to single-agent reinforcement learning (RL) where near-optimal stationary Markov policies can be efficiently learned. Complementing our intractability results for \emph{stationary} Markov CCEs, we provide a decentralized algorithm (assuming shared randomness among players)  for learning a \emph{nonstationary} Markov CCE policy with polynomial time and sample complexity in all problem parameters. Previous work for learning Markov CCE policies all required exponential time and sample complexity in the number of players.

  \end{abstract}

  \noah{TODOs:
    \begin{enumerate}
    \item CE proof (+ prelims regarding policy modifications +swap regret bandit alg)
    \end{enumerate}
    }
    \thispagestyle{empty}
    \newpage
    \clearpage
    \setcounter{page}{1}
\section{Introduction}\label{sec:intro}

Learning in multi-agent, dynamic environments lies at the heart of many recent developments and multiple  significant  challenges in Artificial Intelligence, from playing Go~\cite{silver2016mastering} and Poker~\cite{BrownS19}, to improving autonomous driving algorithms~\cite{shalev2016safe}, and evaluating the outcomes of economic policies~\cite{zheng2020ai}. On this frontier, a prominent and general  learning framework studied in this paper is that of {\em multi-agent reinforcement learning (MARL)}, the multi-agent analog of reinforcement learning (RL) \cite{BusoniuBDS08,ZhangYB21}.  In the same way that RL is mathematically grounded on the model of {Markov Decision Processes}, MARL is grounded on the model of {\em Stochastic Games (SGs)}, the multi-agent analog of MDPs introduced in the seminal work of Shapley~\cite{Shapley53}.

In a stochastic game, several agents interact in an environment over multiple steps: at each step, each agent takes an action, and then the environment transitions to a new state and each agent receives a reward. The rewards and transitions  depend  on both the current state and the profile of actions chosen by the players at the current step.  In contrast to the study of MDPs, where a standard goal is to learn a near-\emph{optimal} policy, a standard goal in the study of SGs is for the agents to learn a near-\emph{equilibrium} policy by interacting. Since their introduction by Shapley, SGs have received extensive study in game theory \cite{NeymanS03,SolanV15}, machine learning \cite{Littman94,HuW03,BusoniuBDS08,ZhangYB21}, and various other fields, due to their broad applicability. %

When there is a single state and the interaction lasts for a single step, SGs degenerate to {\it normal-form} games. In this  case, our understanding of equilibrium existence, computational complexity, and learnability is quite advanced. If the game is two-player and zero-sum, Nash equilibria are identical to minimax equilibria~\cite{VN28}, which can be computed efficiently using linear programming~\cite{Dantzig1951}, and a large number of  (decentralized) learning algorithms have been discovered which converge to the equilibrium when employed by the agents  to iteratively update their strategies, even when the game is a priori unknown to them; see e.g.,~\cite{CBL06,bubeck2012regret} for overviews. Beyond two-player zero-sum games, it is known that computing a Nash equilibrium is intractable in general~\cite{daskalakis2009complexity,ChenDT09}, but correlated and coarse correlated equilibria can be computed efficiently using linear programming, or decentralized learning~\cite{CBL06,bubeck2012regret}.

When there are more states and steps, questions of equilibrium existence, computation and learning become much more intricate, occupying many works in the literature; see e.g.,~\cite{SolanV15,ZhangYB21}. When the players interact over multiple steps, strategic behavior might be {\it history-dependent}, giving rise to notions of equilibrium that are also history-dependent and, thus, extremely complex. Circumventing this complexity, a compelling type of strategic behavior, introduced by Shapley and studied in much of the game theory and machine learning literature, is {\em Markovian}, i.e.,~strategic behavior wherein  the actions chosen by the players at every step of the game depend on the {\it current state} (and potentially the step count), but not the history of states visited and actions played so far. Indeed, under broad and natural conditions, e.g.,~future payoff discounting, there exist Markov Nash equilibria  that are also {\em stationary}, i.e.,~the actions played at every state are also step-count independent; see e.g.,~\cite{Shapley53,Takahashi62,Fink64,SolanV15}. 

On the computation and learning front, most of the progress has been on efficient computation and learning of (approximate) Nash equilibria in {\it two-player zero-sum}  stochastic games; see e.g.,~\cite{brafman2002r,WeiHL17,XieCWY20,ZhangKBY20,SidfordWYY20,BaiJ20,DaskalakisFG20,BaiJY20,LiuTBJ21,JinLWY21}.  Indeed, some of these works provide time- and sample-efficient  learning algorithms for computing Nash equilibria that are also Markovian. 
Beyond the two-player zero-sum case, however, our understanding is lagging. On the one hand, Nash equilibria are computationally intractable, as SGs are more expressive than normal-form games. On the other hand, the complexity and learnability of (coarse) correlated equilibria  are not well-understood in SGs. It  is easy to see that approximate {\it nonstationary} Markov (coarse) correlated equilibria can be computed efficiently via {\it backward induction}, but the complexity of {\it stationary} Markov (coarse) correlated equilibria remained unknown prior to this work. At the same time, a flurry of recent work has provided \emph{learning} algorithms for nonstationary (coarse) correlated equilibria in  finite-horizon episodic  SGs~\cite{LiuTBJ21,MaoB21,SongMB21,JinLWY21}. However, each of these algorithms suffers from one of two shortfalls: either they cannot output Markov equilibria \cite{MaoB21,SongMB21,JinLWY21}, or else they require exponentially many samples in the number of agents \cite{LiuTBJ21}, i.e.,~suffering from the {\it curse of multi-agents} (see \cite{JinLWY21} for additional explanation). 
We defer a more detailed literature review to Appendix \ref{sec:append_related_work}. 

\paragraph{Overview of results.} In this work, we settle the complexity of computing stationary Markov (coarse) correlated equilibria, showing that they are intractable; we then  complement these results with time- and sample-efficient decentralized learning algorithms for computing nonstationary Markov coarse correlated equilibria. In particular, we show the following results (which are summarized and compared to existing results in Table \ref{table:summary}): 
\begin{itemize}
\item In Theorems \ref{thm:perfect-hardness} and \ref{thm:non-perfect-hardness}, we establish intractability of computing stationary Markov coarse correlated equilibria (CCE) in 2-player, discounted general-sum stochastic games. In particular, a notion of stationary Markov CCE called \emph{perfect CCE} (Definition \ref{def:cce}) are \PPAD-hard to approximate up to a constant, and a relaxed notion (\emph{stationary CCE}) are \PPAD-hard to approximate up to a constant assuming the ``PCP for \PPAD'' conjecture (Conjecture \ref{con:pcp-ppad}).
\item To circumvent the above intractability results, we then consider the  computation of \emph{Markov nonstationary CCE}, a relaxation of stationary CCE. While it is  trivial to {\it compute} the  approximate Markov nonstationary CCE using backward induction, the learning problem, in which the SG is {\it unknown} and agents must employ exploratory policies to learn its transitions and rewards, is more challenging. In Theorem \ref{thm:main-ub}, we establish the first guarantee for learning a Markov nonstationary CCE which avoids the curse of multi-agents suffered by prior work \cite{LiuTBJ21}. In particular, the  sample complexity of our algorithm (\algname, Algorithm \ref{alg:main}) is linear in the number of agents. We also show that \algname can be implemented in a decentralized manner (Section \ref{sec:decentralized}), assuming that agents have access to shared common randomness. 
\end{itemize}

{
\begin{table}[!t]\label{table:summary}
\centering
\caption{\small Complexity results of finding CCE in general-sum stochastic games. The two rows correspond to computational and sample complexities, respectively. %
  {\tt Polynomial} means computational and sample costs with polynomial dependence on all problem parameters,  and  {\tt Exponential} means computational and sample costs which are exponential in the number of players. %
}
\vspace{7pt}
\resizebox{\columnwidth}{!}{
\begin{tabular}{c |c c|  c}
\toprule 
 & \multicolumn{2}{c|}{\textbf{Markovian}} & \multirow{2}{*}{\textbf{Non-Markovian}}\\
 \cmidrule{2-3} 
 & \textbf{Stationary} & \textbf{Nonstationary} & \\ 
 \midrule
 \textbf{Computation} & \shortstack{\tt{\PPAD-hard} \\ \textbf{(Theorems \ref{thm:perfect-hardness}, \ref{thm:non-perfect-hardness})}} & \shortstack{\tt{Polynomial}  \\ (Folklore, via backward  induction)} & \multirow{2}{*}{\shortstack{{\tt{Polynomial}} \\ \cite{SongMB21,MaoB21,JinLWY21}}}   \\
 \cmidrule{1-3}
 \textbf{Learning} & \shortstack{\tt{\PPAD-hard} \\ \textbf{(Theorems \ref{thm:perfect-hardness}, \ref{thm:non-perfect-hardness})}}  & 
 \shortstack{{\tt{Exponential}} \cite{LiuTBJ21}; \\
 {\tt{Polynomial} } \textbf{(Theorem \ref{thm:main-ub})} }&  \\
 \bottomrule
\end{tabular}}
\end{table}}

\paragraph{Notation.} We use $\perp$ to denote a null element.  %
For $m \in \BN$, $[m]$ denotes the set $\{1,\cdots,m\}$. For a finite set $\MT$, $\Delta(\MT)$ denotes the space of distributions over $\MT$, $|\MT|$ denotes the cardinality of the set. Let $\sqcup$ denote the disjoint union of sets. For $x\in\RR$, $\sign(x)\in \{\pm 1\}$ denotes the sign of $x$.

\section{Problem Formulation \& Preliminaries}\label{sec:prelim}

\subsection{Stochastic games}
\label{sec:sgs-prelim}
  We begin with some background regarding the terminology and equilibrium concepts in general-sum stochastic games.
  Formally, for some $m \in \BN$, an infinite-horizon  discounted $m$-player\footnote{Hereafter, we use ``player'' and ``agent'' interchangeably.}  stochastic game $\BG$ is defined to be a
  tuple $(\MS, (\MA_i)_{i\in [m]}, \BP, (r_i)_{i \in [m]}, \gamma, \mu)$, where:
  \begin{itemize}
  \item $\MS$ denotes the (finite) \emph{state space}, and we denote  $S=|\MS|$.  
  \item $\MA_i$ denotes the (finite) action space of each player $i \in[m]$, and we denote  $A_i=|\MA_i|$.  We will write $\MA := \MA_1 \times \cdots \times \MA_m$ to denote the \emph{joint action space}, and, for $i \in [m]$, $\MA_{-i} := \prod_{i' \neq i} \MA_{i'}$. %
  \item $\gamma \in [0,1)$ denotes the \emph{discount factor}.
  \item $\mu \in \Delta(\MS)$ denotes the distribution over initial states.
  \item   $r_i : \MS \times \MA \ra [-1,1]$ denotes the \emph{reward function} for player $i$.
  \item  $\BP : \MS \times \MA \ra \Delta(\MS)$ denotes the \emph{transition kernel}: $\BP(\cdot | s, \ba) \in \Delta(\MS)$ denotes the distribution over the next state if joint action profile is played at a state.
  \end{itemize}

  We denote joint action profiles $\ba \in \MA$ with boldface; to denote the action  of some agent $i \in [m]$ when the joint action profile is $\ba$, we write $a_i \in \MA_i$. Similarly, we denote a joint reward profile as $\br \in \BR^m$, with $r_i \in \BR$ denoting the reward to agent $i$.

  \paragraph{Policies: Stationary and  nonstationary.} We primarily consider two types of policies in this paper, namely \emph{stationary} Markov policies and \emph{nonstationary} Markov policies: a \emph{stationary Markov policy} for some player $i$ is a mapping $\pi_i : \MS \ra \Delta(\MA_i)$, and a \emph{nonstationary Markov policy}  for player $i$ is a sequence of maps $\pi_{i,1}, \pi_{i,2},  \ldots : \MS \ra \Delta(\MA_i)$, which we denote by $\pi_i = (\pi_{i,1}, \pi_{i,2},  \ldots)$.  
  A stationary Markov policy $\pi_i$ maps each state $s$ to  a distribution over actions $\pi_i(s) \in \Delta(\MA_i)$ for player $i$; in the nonstationary case, the distribution over actions taken, $\pi_{i,h}( s)$, depends also on the current step $h$. %
  Furthermore, we will often write $\pi_{i}(a_i | s)$ to denote the probability of taking $a_i$ under the distribution $\pi_i(s)$. We denote the set of all stationary Markov  policies of player $i$ by $\Delta(\MA_i)^{\MS}$, and the set of all nonstationary Markov  policies of player $i$ by $\Delta(\MA_i)^{\BN \times \MS}$.\footnote{  Notice that it takes infinite space to specify a general nonstationary policy: to obtain efficient algorithms which output nonstationary policies, we fully specify the policy for some number $H$ of steps and then specify a fixed policy (e.g., playing a uniform action) for all remaining steps. As long as $H \gg \frac{1}{1-\gamma}$, any suboptimality of the policy played at steps $h > H$ incurs only a small approximation error. }

  \emph{Joint} Markov policies are defined analogously to policies for individual players, except they prescribe a distribution over {\it joint} actions at each state: in particular a joint stationary Markov policy is a mapping $\pi : \MS \ra \Delta(\MA)$, and a joint nonstationary Markov policy with horizon $H$ is a sequence $\pi = (\pi_1, \pi_2, \ldots)$, where each $\pi_h : \MS \ra \Delta(\MA)$. With slight abuse of terminology, we will drop ``Markov'' and ``joint'' from our terminology when discussing policies if the context is clear. %
  We say that the stationary policy $\pi : \MS \ra \Delta(\MA)$ is a \emph{product policy} if there are policies $\pi_i : \MS \ra \Delta(\MA_i)$ so that $\pi(s) = \pi_1(s) \times \cdots \times \pi_m(s)$ for all $s\in\cS$. A  nonstationary policy is a product policy if each of its constituent  policies $\pi_h : \MS \ra \Delta(\MA)$ is a product policy. 
  Given a stationary policy $\pi : \MS \ra \Delta(\MA)$ and a player $i \in [m]$, let $\pi_{-i} : \MS \ra \Delta(\MA_{-i})$ denote the joint policy which at each state $s$ outputs the marginal distribution of $\pi(s)$ over $\MA_{-i}$.  For a joint nonstationary policy $\pi \in \Delta(\MA)^{\BN \times \MS}$, we write $\pi_{-i,h} := (\pi_h)_{-i} : \MS \ra \Delta(\MA_{-i})$, and define $\pi_{-i}$ to be the sequence $(\pi_{-i,1}, \pi_{-i,2}, \ldots )$.

\paragraph{Value functions.}  Consider first a joint stationary policy $\pi$. The evolution of the stochastic game $\BG$ proceeds as follows: the system starts at some $s_1 \in \MS$, drawn according to $\mu$, and at each step $h \geq 1$, all players observe $s_h$, draw a joint action $\ba_h \sim \pi(s_h)$, and then the system transitions to some $s_{h+1} \sim \BP(\cdot | s_h, \ba_h)$. We call the tuple $(s_1, \ba_1, s_2, \ba_2, \ldots)$ a \emph{trajectory}, and will write $(s_1, \ba_1, s_2, \ba_2, \ldots) \sim (\BG, \pi)$ to denote a trajectory drawn in this manner.  For any agent $i \in [m]$, their value function $V_i^\pi : \MS \ra [-1,1]$ is defined as the expected $\gamma$-discounted cumulative reward that player $i$ receives if the game starts at state $s_1=s$ and the players act according to $\pi$: %
\vspace{-0.2cm}
  \begin{align}
V_i^\pi(s) := (1-\gamma) \cdot \E_{(s_1, \ba_1, s_2, \ba_2, \ldots) \sim (\BG, \pi)}\left[ \sum_{h=1}^\infty \gamma^{h-1} \cdot r_i(s_{h}, \ba_{h})\bigggiven s_1=s\right]\nonumber.
  \end{align}
  Furthermore, set $V_i^\pi(\mu) := \E_{s \sim \mu} [V_i^\pi(s)]$. 
  The value function is defined similarly for a joint \emph{nonstationary} policy  $\pi \in \Delta(\MA)^{\BN \times \MS}$, except that, due to nonstationarity, it is useful to define separate value functions at each step $h \geq 1$: thus, we write, for $h \geq 1$,
  \begin{align}
V_{i,h}^\pi(s) = (1-\gamma) \cdot \E_{(s_h, \ba_h, s_{h+1}, \ba_{h+1}, \ldots) \sim (\BG, \pi)} \left[ \sum_{h'=h}^\infty \gamma^{h'-h} \cdot r_i(s_{h'}, \ba_{h'}) \Biggiven s_{h} = s \right]\label{eq:vihpi},
  \end{align}
  and for simplicity write $V_i^\pi(s) = V_{i,1}^\pi(s)$. In the expectation in (\ref{eq:vihpi}), for $h' \geq h$ the action $\ba_{h'}$ is drawn from $\pi_{h'}(s_{h'})$. Similarly to above, we define $V_{i,h}^\pi(\mu) := \E_{s \sim \mu} \left[ V_{i,h}^\pi(s) \right]$.

  \subsection{Equilibrium notions} 
  
  To define the equilibrium notions we work with, we begin by introducing best-response policies.
  
\paragraph{Best-response policies.}  For any $i \in [m]$ and for stationary Markov policies $\pi_i : \MS \ra \Delta(\MA_i),\ \pi_{-i} : \MS \ra \Delta(\MA_{-i})$, we let $\pi_i \times \pi_{-i}$ refer to the policy which at each state $s$, samples an action profile according to the product distribution $\pi_i(s) \times \pi_{-i}(s)$.
  Fix any $i \in [m]$, and consider any joint stationary policy $\pi_{-i} : \MS \ra \Delta(\MA_{-i})$ of all players except player $i$. There is a stationary policy of the $i$th player, $\pi_i^\dagger(\pi_{-i}) : \MS \ra \Delta(\MA_i)$, so that
$ 
    V_i^{\pi_i^\dagger(\pi_{-i}) \times \pi_{-i}}(s) = \sup_{\pi_i' : \MS \ra \Delta(\MA_i)} V_i^{\pi_i' \times \pi_{-i}}(s)%
$  %
for all $s \in \MS$. The policy $\pi_i^\dagger(\pi_{-i})$ is called the \emph{best-response policy} of player $i$, and we will write $V_i^{\dagger, \pi_{-i}}(s) := V_i^{\pi_i^\dagger(\pi_{-i}) \times \pi_{-i}}(s)$, and $V_i^{\dagger, \pi_{-i}}(\mu) := \EE_{s \sim \mu} \left[ V_{i}^{\dagger, \pi_{-i}}(s) \right]$.\footnote{It is well-known that when $\pi_{-i}$ is Markov (as is assumed here), the best response amongst all \emph{history-dependent} policies is Markovian (and is in fact deterministic), as it reduces to a single-agent Markov decision process problem; thus it is without loss of generality to constrain ourselves to Markov policies $\pi_i'$ above; an analogous fact also holds for nonstationary policies.}

  Best-response policies for nonstationary policies are defined similarly:  for a nonstationary policy $\pi_{-i} \in \Delta(\MA_{-i})^{\BN \times \MS}$, there is a nonstationary \emph{best-response} policy of the $i$th player $\pi_i^\dagger(\pi_{-i}) \in \Delta(\MA_i)^{\BN \times \MS}$ so that for all $(h,s) \in \BN \times \MS$, $V_{i,h}^{\pi_i^\dagger(\pi_{-i}), \pi_{-i}}(s) = \sup_{\pi_i' \in \Delta(\MA_i)^{\BN \times S}} V_{i,h}^{\pi_i' \times \pi_{-i}}(s)$. As above we write $V_{i,h}^{\dagger, \pi_{-i}}(s) := V_{i,h}^{\pi_i^\dagger(\pi_{-i}) \times \pi_{-i}}(s)$ and $V_{i,h}^{\dagger, \pi_{-i}}(\mu) := \EE_{s \sim \mu} \left[ V_{i,h}^{\dagger, \pi_{-i}}(s) \right]$.

  \paragraph{Coarse correlated equilibrium.}
  We first define approximate Markov coarse correlated equilibria in stochastic games.
  \vspace{-0.2cm}
\begin{defn}[Coarse correlated equilibrium]
  \label{def:cce}
  For $\ep > 0$:
  \begin{itemize}
  \item A stationary policy $\pi \in \Delta(\MA)^\MS$ is an \emph{$\ep$-approximate stationary Markov coarse correlated equilibrium} (abbreviated \emph{$\ep$-stationary CCE}) if $\max_{i \in [m]} \left\{ V_i^{\dagger, \pi_{-i}}(\mu) - V_i^\pi(\mu) \right\} \leq \ep$.
  \item A nonstationary policy $\pi \in \Delta(\MA)^{\BN \times \MS}$ is an \emph{$\ep$-approximate nonstationary Markov coarse correlated equilibrium}  (abbreviated \emph{$\ep$-nonstationary CCE}) if $\max_{i \in [m]} \left\{ V_{i}^{\dagger, \pi_{-i}}(\mu) - V_{i}^\pi(\mu) \right\} \leq \ep$.
  \item  A stationary policy $\pi \in \Delta(\MA)^\MS$ is an \emph{$\ep$-approximate perfect Markov coarse correlated equilibrium} (abbreviated \emph{$\ep$-perfect CCE}) if it holds that $\max_{i \in [m], s \in \MS} \left\{ V_{i}^{\dagger, \pi_{-i}}(s) - V_{i}^\pi(s) \right\} \leq \ep$. 
  \end{itemize}
\end{defn}
It is also possible to define $\ep$-perfect nonstationary CCE in the natural way, but we will not need to do so (as such equilibria are easily seen to be computationally feasible to compute, yet also  impossible to learn in the  model of PAC learning of stochastic games we consider, see Section \ref{sec:pac-rl-sg}). %
  When stationarity (or lack thereof) of $\pi$ is clear from context, we will drop the words ``stationary'' and ``nonstationary'' from the above definitions; furthermore, we will drop the word ``Markov'' when referring to the above definitions since all equilibria we consider are Markovian.

  \paragraph{Nash equilibrium.} %
  We next define approximate Markov Nash equilibria in stochastic games. %

 \begin{defn}[Nash equilibrium]
   \label{def:ne}
   For $\ep > 0$, the notions:
   \begin{itemize}
   \item  \emph{$\ep$-approximate stationary Markov Nash equilibrium} (abbreviated \emph{$\ep$-stationary NE})
   \item \emph{$\ep$-approximate nonstationary Markov Nash equilibrium} (abbreviated \emph{$\ep$-nonstationary NE})
     \item  \emph{$\ep$-approximate perfect Markov Nash equilibrium} (abbreviated \emph{$\ep$-perfect NE})
     \end{itemize}
     are defined to be $\ep$-stationary CCE, $\ep$-nonstationary CCE, and $\ep$-perfect CCE, respectively, which are also product policies.
  \end{defn}
  In the literature, perfect NE are also referred to as \emph{Markov perfect equilibria} \cite{maskin1988theory} for stochastic games. It is  known that perfect a Markov NE always {\it exists}  for discounted SGs \cite{Shapley53,Fink64}, thus so do the stationary and nonstationary NE, and the corresponding CCE counterparts.

  \subsection{The PAC-RL model for stochastic games}
  \label{sec:pac-rl-sg}
  Our results on {\it learning SGs} (Section \ref{sec:ub-res}) operate in the   \emph{probably approximately correct (PAC) learning model} of RL, which is standard in the literature \cite{kakade2003sample,azar2017minimax}. In particular, at the onset of the algorithm, the agents have no information about the transitions $\BP$, the reward functions $r_i$, or the initial state distribution $\mu$; %
  only the parameters $S, \gamma$ are known to all agents, and each agent $i$ knows $A_i$. The agents' only access to the SG is through the ability to repeatedly choose some joint (perhaps nonstationary) policy $\pi$ and then sample a trajectory  $(s_1, \ba_1, \br_1, s_2, \ba_2, \br_2, \ldots) \sim (\BG, \pi)$. In the \emph{centralized setting} (studied, for instance, in \cite{LiuTBJ21}), all agents may communicate with a central coordinator who may choose $\pi$ and observe the entire trajectory. Our algorithm may in fact be implemented in the stricter \emph{decentralized setting} with public randomness, which is discussed in Section \ref{sec:decentralized}. Note that in either case, the agents need to efficiently {\it explore} the environment, as the trajectory data they access might not visit all the state-action pairs with large rewards often enough, usually  leading to  a poor sample complexity.
  
  In order to have computationally efficient algorithms, each trajectory must be truncated at some step $H \in \BN$, after which another trajectory is started anew. In the infinite-horizon discounted setting, we therefore assume that agents can choose to stop playing the SG at some point: in particular for a desired error parameter $\ep > 0$, all agents will truncate after $H:=\frac{\log 1/\ep}{1-\gamma}$ steps, incurring only $\ep$ loss from steps $h > H$.

\subsection{Turn-based stochastic games}
A stochastic game $\BG$ is called a \emph{turn-based stochastic  game} if, at each state $s \in \MS$, there is a single player $i \in [m]$ (called the  \emph{controller} of state $s$, and denoted $i = \ag{s}$) whose action at $s$ entirely determines the reward and the transition to the next state. Formally,   for all $j \in [m]$  there is some function $r_j' : \MS \times (\MA_1 \sqcup \cdots \sqcup \MA_m) \ra [-1,1]$   and some transition kernel $\BP' : \MS \times (\MA_1 \sqcup \cdots \sqcup \MA_m) \ra \Delta(\MS)$ so that $r_j(s, \ba) = r_j'(s, a_{\ag{s}})$ for all $s \in \MS$, $j \in [m]$, $\ba = (a_1, \ldots, a_m) \in\MA$, and so that  $\BP(\cdot | s, \ba) = \BP'(\cdot | s, a_{\ag{s}})$ for all $s \in \MS$, $\ba \in \MA$. It is evident that in turn-based stochastic  games, the notions of $\ep$-CCE and $\ep$-NE are equivalent, both for stationary and nonstationary policies (and the same holds for the perfect versions of the equilibria), since in such games we may restrict to product policies without loss of generality.  %
  
\subsection{\PPAD and the generalized circuit problem}
The problems of computing equilibria of the types defined in Definitions \ref{def:cce} and \ref{def:ne}  are instances of \emph{total search problems} \cite{megiddo1991total}. In particular, they lie in the class
\TFNP, which is the class of binary relations $\MP \subset \{0,1\}^\st \times \{0,1\}^\st$ so that for all $\bx, \by \in \{0,1\}^\st$, there is a polynomial-time algorithm that can determine whether $\MP(\bx, \by)$ holds, and so that for all $\bx \in \{0,1\}^\st$, there is some $\by \in \{0,1\}^\st$ with $|\by| \leq \poly(|\bx|)$ so that $\MP(\bx, \by)$ holds. Equilibrium computation of stochastic games is seen to be in \TFNP as follows: $\bx$ represents the description of the stochastic game, $\by$ represents a proposed equilibrium policy, and $\MP(\bx, \by)$ holds if $\by$ is a (approximate) equilibrium of $\bx$. For all notions of equilibria we have defined, an equilibrium always exists \cite{Fink64,SolanV15} and it may be efficiently checked whether a proposed equilibrium $\by$ is indeed an equilibrium for the game represented by $\bx$.

The class $\PPAD$ \cite{Papadimitriou1994} is defined as the class of all problems in $\TFNP$ which have a polynomial-time reduction to the \textsf{End-of-the-Line} (\EOTL) problem; we refer the reader to \cite{Papadimitriou1994,R16} for a description of \EOTL, as we do not need to directly use its definition. To establish our hardness results (i.e.,   \PPAD-completeness for computing approximate stationary CCE), we instead use the fact, proven in \cite{R16}, that the $\ep$-\Gcircuit problem (Definition \ref{def:gcircuit}) is \PPAD-complete (Theorem \ref{thm:gcircuit-ppad}). Additional preliminaries regarding \PPAD and $\ep$-\Gcircuit are presented in Section \ref{sec:ppad-prelim}.

\section{\PPAD-hardness results for stationary equilibria}\label{sec:lb-res}
We next state our main lower bounds, which establish hardness for finding stationary CCE in infinite-horizon discounted  SGs. To do so, we first prove hardness for finding stationary NE in the special case of turn-based discounted SGs, and then note that in such games, stationary Nash equilibria and stationary coarse correlated equilibria coincide (as do the perfect versions). %

\begin{theorem}[\PPAD-hardness for perfect equilibria]%
  \label{thm:perfect-hardness}
  There is a constant $\ep > 0$ so that the problem of computing $\ep$-perfect NE in 2-player, $1/2$-discounted turn-based stochastic games is \PPAD-hard.

  Thus, computing $\ep$-perfect CCE in 2-player, $1/2$-discounted stochastic games is \PPAD-hard.
\end{theorem}

We next turn to the  intractability results for  computing the weaker ``non-perfect'' notions of equilibria, namely $\ep$-stationary NE in turn-based stochastic games, and more generally, $\ep$-stationary CCE in stochastic games. Motivation for such results is two-fold: first, they are a natural target in terms of extending the reach of intractability results; second, if the initial state distribution $\mu$ is not sufficiently {\it exploratory}, in the PAC-RL model it is impossible to \emph{learn} notions of equilibria which are perfect in that they require a condition to hold for {\it each state}. %
Thus, in the learning setting (to be discussed in Section \ref{sec:ub-res}),  our algorithm can only find the (non-perfect) notions of $\ep$-nonstationary NE in turn-based games and $\ep$-nonstationary CCE in general stochastic games. It is thus desirable  to have corresponding lower bounds for stationary NE in turn-based games and stationary CCE in general games. 
In Theorem \ref{thm:non-perfect-hardness} below, we prove such lower bounds: %
\begin{theorem}[\PPAD-hardness for non-perfect equilibria]
  \label{thm:non-perfect-hardness}
  There are constants $\ep, c > 0$ so that:
  \begin{itemize}
  \item The problem of computing $c/S$-stationary NE in 2-player, 1/2-discounted turn-based stochastic games is \PPAD-hard; and thus so is the problem of computing $c/S$-stationary CCE in 2-player, 1/2-discounted stochastic games.
  \item Under the ``PCP for \PPAD conjecture'' (Conjecture \ref{con:pcp-ppad}),  the problem of computing $\ep$-stationary NE in 2-player, 1/2-discounted turn-based stochastic games is \PPAD-hard; and thus so is the problem of computing $\ep$-stationary CCE in 2-player, 1/2-discounted stochastic games.
  \end{itemize}
\end{theorem}

\subsection{Proof overview for Theorems \ref{thm:perfect-hardness} and \ref{thm:non-perfect-hardness}}
The proofs of Theorems \ref{thm:perfect-hardness} and \ref{thm:non-perfect-hardness} proceed by reducing the $(\ep, \delta)$-\Gcircuit problem, introduced in Section \ref{sec:ppad-prelim}, to the problem of finding approximate stationary Nash equilibria in 2-player general-sum turn-based stochastic games. We overview here the proof of Theorem \ref{thm:perfect-hardness}, which uses \PPAD-hardness of the $(\ep, 0)$-\Gcircuit problem (Theorem \ref{thm:gcircuit-ppad}); the proof of Theorem \ref{thm:non-perfect-hardness} is similar, except that the second part of the theorem relies on the \PPAD-hardness of the $(\ep, \delta)$-\Gcircuit problem for some constants $\ep, \delta > 0$, which is not yet known but is rather the content of Conjecture \ref{con:pcp-ppad}.

Given an instance $\MC$ of the $(\ep, 0)$-\Gcircuit problem, we wish to construct a 2-player, 1/2-discounted turn-based stochastic game $\BG$ so that given an $\ep'$-perfect NE of $\BG$, for some constant $\ep' > 0$, we can compute an assignment of values to the nodes of the generalized circuit $\MC$ which $\ep$-satisfies all gates. To do so, we construct $\BG$ by creating a number of ``gadgets'' (Section \ref{sec:game-gadgets}), each of which implements one gate in the generalized circuit instance $\MC$. Each such gadget consists of a constant number of states of $\BG$, each of which is controlled by a single player who can take one of two actions, say $\{0,1\}$, at the state. A stationary policy of $\BG$ then is equivalent to a mapping $\pi : \MS \ra [0,1]$, where $\MS$ is the state space of $\BG$ and, for $s \in \MS$, $\pi(s)$ denotes the probability that the agent controlling $s$ chooses action 1 at state $s$. We then define transitions and rewards for the states in each gadget in a way that forces any $\ep'$-perfect NE $\pi$ of $\BG$ to have the property that the restriction of $\pi$ to the gadget $\ep$-approximately satisfies the gate corresponding to the gadget. By defining transitions between the gadgets in a way that mirrors the structure of the circuit $\MC$, we may achieve the desired reduction. The gadgets we use are somewhat reminiscient of the game gadgets used in \cite{daskalakis2009complexity} to show \PPAD-hardness of computing Nash equilibria in graphical games.

The reduction as described above suffices to show \PPAD-hardness of computing $\ep'$-perfect NE of $S$-player stochastic turn-based games, namely games in which a different player controls each state. To establish hardness for {\it 2-player}  games, care must be taken to ensure that for each player, the rewards assigned to them from different gadgets do not conflict with each other. In fact, as we describe in Section \ref{sec:valid-colorings}, conflicts could arise for the gadgets that we use. To overcome this issue, we show (in Lemma \ref{lem:construct-valid-coloring}) how to map the given generalized circuit instance $\MC$ to an equivalent one, $\MC'$, which has the property that conflicts as described above \emph{cannot} arise. In particular, we introduce the notion of \emph{valid colorings} (Definition \ref{def:validity}) to establish a formal condition on the circuit instance $\MC'$ which guarantees that there will be no conflicts. We then show, using our game gadgets, how to map the instance $\MC'$ (equipped with a valid coloring) to a 2-player stochastic game $\BG$ whose $\ep'$-perfect NE yield an assignment of $\MC$ that $\ep$-satisfies all gates.

The above description omits some details; for instance, rather than reducing directly to the problem of $\ep'$-perfect NE in stochastic games, we instead reduce $\ep$-\Gcircuit to the problem of computing \emph{perfect well-supported Nash equilibria in stage games} (Definition \ref{def:wsne-sg}), which in turn reduces to perfect NE (Lemma \ref{lem:reduce-to-sg}). The full proofs for our lower bounds are in Section \ref{sec:proof-lb}. 

\section{A Nearly Decentralized MARL Algorithm}\label{sec:ub-res}

Given the intractability results discussed in the previous section, it is natural to relax the notion of (Markov) {\it stationary}  equilibria. A very natural relaxation, and indeed one considered in a number of recent works, is to drop the requirement of \emph{stationarity} of the equilibrium policy. While computing approximate  nonstationary (coarse) correlated equilibria in general-sum discounted stochastic games is straightforward via backward induction, %
the learning problem, in which the stochastic game is {\it unknown} and the players must employ {\it exploratory} policies to learn an equilibrium, is significantly less trivial. All prior work (on finite-horizon episodic SGs) for learning nonstationary equilibria   either requires a number of samples exponential in the number of players \cite{BaiJ20,LiuTBJ21}, or else does not compute Markov policies \cite{MaoB21,JinLWY21,SongMB21}. In this section, we state Theorem \ref{thm:main-ub}, which establishes a nearly decentralized learning algorithm that computes a Markov nonstationary equilibrium in time polynomial in the number of players.

\paragraph{Reduction to the finite-horizon case.} 
We address the problem of learning $\ep$-nonstationary CCE   in infinite-horizon discounted games by using a standard reduction to the problem of computing $\ep$-nonstationary CCE  in \emph{finite-horizon undiscounted} games, which is the setting studied by the majority of the aforementioned work \cite{MaoB21,JinLWY21,SongMB21,LiuTBJ21}. A finite-horizon stochastic game $\BG$ is defined identically to the infinite-horizon case (Section \ref{sec:sgs-prelim}), except that the discount factor $\gamma$ is replaced by an integer $H \in \BN$, denoting the \emph{horizon}; as such, the total reward is no longer discounted, but is summed from steps $h =1$ to $H$.

In Section \ref{sec:finite-horizon-prelim},  we formally define policies, value functions, and equilibria in finite-horizon SGs, which closely follow the corresponding notions for the discounted case in Section \ref{sec:prelim}. We then discuss the simple reduction from episodic learning of an infinite-horizon discounted game with discount factor $\gamma$ to episodic learning of a finite-horizon game with horizon $H := \frac{\log 1/\ep}{1-\gamma}$. Owing to the fact that $\gamma^H \leq \ep$, this reduction preserves nonstationary equilibria up to an additive approximation of $\ep$. Thus, in the remainder of this section, we focus on a  finite-horizon stochastic game $\BG = (\MS, (\MA_i)_{i \in [m]},\BP,(r_i)_{i \in [m]}, H, \mu)$, and aim to learn an $\ep$-CCE in the episodic PAC-RL model, see Section \ref{sec:pac-rl-sg}. An additional motivation for studying the finite-horizon case is to make a fair comparison with recent works, e.g.,  \cite{MaoB21,JinLWY21,SongMB21,LiuTBJ21}, with the goal of learning a Markov CCE in a decentralized fashion, with polynomial dependence on all  parameters. %

\subsection{Preliminaries for adversarial bandit routine}
\label{sec:adv-bandit}
Our learning algorithm will require all players to choose their actions at certain steps of each episode according to a {\it no-regret} bandit algorithm at each state. Therefore, we briefly overview the setup and guarantees of adversarial no-regret bandit learning. 
Consider the following setting involving a bandit learner and an adversary interacting over $T$ rounds. The learner has access to a finite set $\MB$ of arms, with $B := |\MB|$.  For each time step $t \in [T]$:
  \begin{enumerate}
  \item The learner picks a distribution $p_t \in \Delta(\MB)$;
  \item The adversary chooses a loss vector $\ell_t \in [0,1]^\MB$, depending on the arms chosen by the learner at previous steps, as well as a vector $\til \ell_t \in [0,1]^\MB$, so that, if $\MF_t$ denotes the sigma-algebra generated by all random variables in the adversary's view up to time $t$ (including $\ell_t$), for all $b \in \MB$,we have  $\E[\til \ell_t(b) |  \MF_t] = \ell_t(b)$. %
  \item The learner takes action $b_t \sim p_t$ and sees  $\til \ell_t(b_t)$, which satisfies $\E[\til \ell_t(b_t) | b_t, \MF_t] = \ell_t(b_t)$.  
  The learner uses the pair $(b_t, \til \ell_t(b_t))$ to update its distribution. 
  \end{enumerate}

Theorem \ref{thm:adv-bandit-external} below gives a high-probability regret guarantee for an adversarial bandit algorithm, which may be taken to be \texttt{Exp3-IX} \cite{neu2015explore}; the statement of the theorem differs slightly from that in prior works, so we explain how to derive it from \cite{neu2015explore} in Section \ref{sec:bandit-derivations}. Note that any such adversarial bandit algorithms with high-probability guarantees would be sufficient for us later, and we choose  \texttt{Exp3-IX} mainly for its simplicity and reprensentativeness.

\begin{theorem}[\cite{neu2015explore}, Theorem 1]
  \label{thm:adv-bandit-external}
  There is an algorithm for the above adversarial bandit setting which obtains the following regret guarantee and runs in $\poly(B)$ time at each step: for any $T_0 \in \BN$, $\delta \in (0,1)$, we have that, with probability at least $1-\delta$, for all $T \leq T_0$,
  \begin{align}
\max_{b \in \MB} \sum_{t=1}^T \left( \ell_t(b_t) - \ell_t(b) \right) \leq O \left ( \sqrt{TB } \cdot \log(T_0 B/\delta) \right)\nonumber.
  \end{align}
\end{theorem}

\subsection{The \algname algorithm}

  We are ready to introduce our main algorithm, called \algname (\textbf{S}tage-based \textbf{Po}licy \textbf{C}over for \textbf{M}ulti-\textbf{A}gent Learning with \textbf{R}max), presented in full in Algorithm \ref{alg:main}. 
  The \algname algorithm combines multiple tools from the literature in order to learn a Markov equilibrium while beating the curse of multi-agents: it uses an adversarial bandit routine at each state (Section \ref{sec:adv-bandit}), similar to the recent works  \cite{MaoB21,SongMB21,JinLWY21}, optimistic rewards similar to those in the \texttt{Rmax} algorithm \cite{brafman2002r} to induce exploration, as well as a \emph{policy cover} (see, e.g., \cite{agarwal2020pcpg,foster2021statistical,jin2020reward}) to ensure that exploratory policies learned in the past are not forgotten.

  In more detail, the algorithm proceeds as follows: it takes as input the parameters $m, \MS, \MA, H$ of the finite-horizon stochastic game, as well as additional parameters $K, N_{\rm visit} \in \BN$, $p \in (0,1)$ whose interpretations will be explained below. \algname's computation proceeds in a number of \emph{stages} $q$ (step \ref{it:loop-q}). At a high level, the algorithm will attempt to find, for each pair $(h,s)$, a policy $\pi_{h,s}^\cover$ which visits $(h,s)$ with nontrivial probability (namely, at least $p$; see step \ref{it:set-pi-cover}). For each stage $q$, the algorithm loops over $h = H, H-1, \ldots, 1$ (step \ref{it:loop-h}), and for each such value of $h$, the algorithm first re-initializes all adversarial bandit instances according to a bandit algorithm satisfying the guarantee of Theorem \ref{thm:adv-bandit-external} (step \ref{it:bandit-init}) and then loops over all policies in a \emph{policy cover} set $\Pi_h^q$ (step \ref{it:loop-over-pi}), which is the set of current non-null cover policies $\pi_{h,s}^\cover$. To deal with the case that $\Pi_h^q$ is empty (e.g., if $q = 1$), in step \ref{it:loop-over-pi} the algorithm also loops over the policy $\pi^\MU$ which prescribes all agents to choose their action uniformly at random at each state. %

  For each $\pi \in \Pi_h^q \cup \pi^\MU$, the algorithm executes a policy $\ol \pi$ (step \ref{it:define-bar-pi}), which is identical to $\pi$ except that at step $h$ each agent plays according to her  adversarial bandit instance. The policy $\ol \pi$ is executed for a total of $K$ \emph{episodes} (step \ref{it:K-episodes}). The algorithm then uses the trajectory data drawn from $\ol \pi$ at each episode to update the adversarial bandit instances at step $h$ (steps \ref{it:bandit-begin} to \ref{it:bandit-end}). Using the data collected from all $K$ episodes for each cover policy in $\Pi_h^q \cup \pi^\MU$, the algorithm then computes a function $\Vo_{i,h}^q : \MS \ra \BR$, representing a value function estimate for a coarse correlated equilibrium, in steps \ref{it:define-vbar} through \ref{it:define-vbar-real}. Crucially, the estimates $\Vo_{i,h}^q$ depend on $\Vo_{i,h+1}^q$, necessitating the backward loop over $h$ in step \ref{it:loop-h}.

  After this backward  loop over $h$ has completed, the algorithm constructs a policy $\til \pi^q$ (step \ref{it:define-tilpiq}) representing an estimate for an approximate CCE given the data collected at stage $q$. By drawing a total of $N_{\rm visit}$ additional trajectories from $\til \pi^q$, in step \ref{it:state-visitation-estimate} it uses the sub-procedure \EST (Algorithm \ref{alg:visitation-est}) to estimate the state visitation probabilities for $\til \pi^q$. If $\til \pi^q$ does not visit any new pairs $(h,s)$ with significant probability, \algname terminates, outputting $\til \pi^q$; otherwise, it sets $\pi_{h,s}^\cover \gets \til \pi^q$ for newly visited pairs $(h,s)$, and then proceeds to the following stage. 
  
  \begin{algorithm}
  \caption{\algname (\textbf{S}tage-based \textbf{Po}licy \textbf{C}over for \textbf{M}ulti-\textbf{A}gent Learning with \textbf{R}max)}\label{alg:main}
  \begin{algorithmic}[1]
    \Procedure{\algname}{$m,\MS, \MA, H,K,N_{\rm visit},p$}
\State Set $\MV = \emptyset$.  (\emph{$\MV$  denotes the set of ``well-visited'' states, updated at each stage.})%
\State For each $h \in [H],\ s \in \MS$, set $\pi_{h,s}^\cover = \perp$.  (\emph{$\pi_{h,s}^\cover$ will be set to a joint policy in $\Delta(\MA)^{[H] \times \MS}$.}) %
\For{$q \geq 1$ and while $\tau = 0$}\label{it:loop-q}
  \State Set $\tau = 1$   (\emph{$\tau$ is a bit indicating whether we should terminate at the current stage}).
  \State Set $\Pi_h^q := \{ \pi_{h,s}^\cover : s \in \MS\}$ for each $h \in [H]$. (\emph{Note that $|\Pi_h^q| \leq S$ for each $h$.})
  \For{$h = H, H-1, \ldots, 1$}\label{it:loop-h}
    \State Set $k = 0$, and $\Vo_{i,H+1}^q(s) = 0$ for all $s \in \MS$ and $i \in [m]$.
    \State Each player $i$ initializes an adversarial bandit instance at each state $s \in \MS$ for the  step $h$, according to some algorithm satisfying the guarantee of Theorem \ref{thm:adv-bandit-external}.\label{it:bandit-init}
    \For{each  $\pi \in \Pi_h^q \cup \pi^\MU$} (\emph{$\pi^\MU$ chooses actions uniformly at random})\label{it:loop-over-pi}
  \For{a total of $K$ times}\label{it:K-episodes}
    \State Increment $k$ by 1.
    \State Let $\ol \pi$ be the policy which follows $\pi$ for the first $h-1$ steps and plays according to the bandit algorithm for the state visited at step $h$ (and acts arbitrarily for steps $h' > h$).\label{it:define-bar-pi}
    \State Draw a joint trajectory $(s_{1,k}, \ba_{1,k}, \br_{1,k}, \ldots, s_{H,k}, \ba_{H,k}, \br_{H,k})$ from $\ol \pi$.\label{it:draw-trajectory} %
    \If{$(h, s_{h,k}) \in \MV$}\label{it:bandit-begin}
    \State Each $i$ updates its bandit alg.~at $(h,s_{h,k})$ w/ $( a_{i,h,k}, \frac{H - r_{i,h,k} - \Vo_{i,h+1}^q(s_{h+1,k})}{H} )$.
    \Else
    \State Each $i$ updates its bandit alg.~at $(h,s_{h,k})$ w/ $(a_{i,h,k}, \frac{H - (H+1-h)}{H} )$. 
    \EndIf\label{it:bandit-end}
    \EndFor
    \EndFor
    \State For each $s \in \MS$, and $j \geq 1$, let $k_{j,h,s} \in [\KMAX+1]$ denote the $j$th smallest value of $k$ so that $s_{h,k} = s$, or $\KMAX+1$ if such a $j$th smallest value does not exist. \label{it:define-vbar}
      \State For each $s \in \MS$, let $J_{h,s}$ denote the largest integer $j$ so that $k_{j,h,s} \leq \KMAX$. \label{it:define-jhs}
    
   \State \label{it:end-k-for} Define $\til \pi^q_h \in \Delta(\MA)^\MS$ to be the 1-step policy: %
$ \til \pi_h^q(\ba | s) = \frac{1}{J_{h,s}} \sum_{j=1}^{J_{h,s}} \One[\ba = \ba_{h,k_{j,h,s}}].$

    \State Set\label{it:define-vbar-real} %
    \begin{align}
      \Vo_{i,h}^q(s) := \begin{cases}
        \frac{1}{J_{h,s}} \sum_{j=1}^{J_{h,s}} \left( r_{i,h,k_{j,h,s}} + \Vo_{i,h+1}^q(s_{h+1,k_{j,h,s}})\right) & : (h,s) \in \MV \\
        (H+1-h) &: (h,s) \not \in \MV.
        \end{cases}\label{eq:define-vbar}
    \end{align}
  \EndFor
\State  Define the joint policy $\til \pi^{q}$,  which follows $\til \pi_{h'}^q$ at each step $h' \in [H]$. \label{it:define-tilpiq}
\State  \label{it:state-visitation-estimate} Call $\EST(\til \pi^q, N_{\rm visit})$ (Alg.~\ref{alg:visitation-est}) to obtain estimates $\wh d_{h'}^{q} \in \Delta(\MS)$ for each $h' \in [H]$. %
  \For{each $s \in \MS$ and $h' \in [H]$} %
    \If{$\wh d_{h'}^{q}(s) \geq p$ and $(h', s) \not \in \MV$}  %
      \State Set $\pi_{h',s}^\cover \gets  \til \pi^q $.\label{it:set-pi-cover}
      \State Add $(h',s)$ to $\MV$. %
      \State Set $\tau \gets 0$.
      \EndIf
    \EndFor
\EndFor
\State \Return the policy $\wh \pi := \til \pi^q$. %
    \EndProcedure
    \end{algorithmic}
  \end{algorithm}

  \begin{algorithm}
    \caption{\EST}\label{alg:visitation-est}
    \begin{algorithmic}[1]
      \Procedure{\EST}{$\pi, N$}
      \For{$1 \leq n \leq N$}
      \State Draw a trajectory from $\pi$, and let $(s_1^n, \ldots, s_H^n)$ denote the sequence of states observed.
      \EndFor
      \For{$h \in [H]$}
      \State Let $\wh d_h \in \Delta(\MS)$ denote the empirical distribution over $(s_h^1, \ldots, s_h^N)$.
      \EndFor
      \State\Return $(\wh d_1, \ldots, \wh d_H)$.
      \EndProcedure
    \end{algorithmic}
  \end{algorithm}
  
  \subsection{Guarantee for \algname}
  In Theorem \ref{thm:main-ub} we state the main guarantee for \algname, which shows that for finite-horizon general-sum stochastic games, \algname obtains sample and computational complexity polynomial in all relevant parameters, including the number of players. 
  \begin{theorem}
    \label{thm:main-ub}
    Fix any $\ep, \delta > 0$. For appropriate settings of the parameters $N_{\rm visit}, K, p$ (specified in Section \ref{sec:conc-ineq}), \algname outputs an $\ep$-nonstationary CCE with probability at least $1-\delta$ after sampling at most $O \left( \frac{H^{10} S^3 \iota^2 \max_{i \in [m]} A_i}{\ep^3} \right)$ trajectories, where $\iota = \log \left( \frac{SH \max_i A_i}{\ep \delta} \right)$. The computational complexity of \algname is polynomial in $H, S, \max_i A_i, 1/\ep, \log 1/\delta$. 
  \end{theorem}
  Combining \algname with the reduction from infinite-horizon discounted games to finite-horizon (undiscounted) games described in Section \ref{sec:finite-horizon-prelim}, we obtain the following as an immediate corollary of Theorem \ref{thm:main-ub}:
  \begin{corollary}
    \label{cor:main-discounted}
    There is a polynomial-time algorithm which learns a  $\ep$-nonstationary CCE in $\gamma$-discounted general-sum stochastic games using $\til O \left( \frac{S^3  \max_{i \in [m]} A_i}{(1-\gamma)^{10} \ep^3 }\right)$ trajectories.
  \end{corollary}

  \subsection{Proof overview for Theorem \ref{thm:main-ub}}
  We now overview the proof of Theorem \ref{thm:main-ub}. Let $\wh \MV$ denote the value of $\MV$ at termination of \algname and $\wh q$ denote the value of the final stage of \algname. The main tool in the proof is to construct an intermediate game, denoted $\BG_{\wh \MV}$ (Section \ref{sec:intermed-game}): we will first show that the output policy of \algname is an $\ep$-CCE with respect to the game $\BG_{\wh \MV}$, and then, using the termination criterion of \algname, we will show that this implies that $\wh \pi$ is an $\ep$-CCE with respect to the true game $\BG$.

  The game $\BG_{\wh \MV}$ is constructed in a similar way to an intermediate MDP used in the analysis of the \texttt{Rmax} algorithm \cite{brafman2002r,jin2020reward}. For tuples $(h,s) \not \in \wh \MV$, $\BG_{\wh \MV}$ transitions, at $(h,s)$, to a special sink state at which all agents receive reward 1 (the maximum possible reward) at all future steps; for all $(h,s) \in \wh\MV$, the rewards and transitions of $\BG_{\wh \MV}$ at $(h,s)$ are identical to those of $\BG$. By ensuring that the parameter $K$ passed to \algname is sufficiently large, we may guarantee that, during stage $\wh q$, \algname visits all $(h,s) \in \wh \MV$ sufficiently many times to compute accurate estimates of $V_{i,h}^{\BG_{\wh \MV}, \wh \pi}(s)$ for such $(h,s) \in \wh \MV$. Since, for all $(h,s) \not \in \wh \MV$, we have $V_{i,h}^{\BG_{\wh \MV}, \wh \pi}(s) = H+1-h$, it is possible to show (Lemma \ref{lem:conc-onpolicy}) that $\left|\Vo_{i,h}^{\wh q}(s) - V_{i,h}^{\BG_{\wh \MV}, \wh \pi}(s)\right|$ is small for \emph{all} $(h,s) \in [H] \times \MS$ and all $i \in [m]$. %

  Using the no-regret property of the adversarial bandit instances used by each player for each $(h,s)$, we then obtain (Lemmas \ref{lem:bandit-no-reg} and \ref{lem:no-gprime-regret}) that $\wh \pi$ is an $\ep$-CCE of $\BG_{\wh \MV}$. To derive such a guarantee for the true SG $\BG$, we use two facts: first, by the optimistic nature of the rewards of $\BG_{\wh \MV}$ the value function of $\BG_{\wh \MV}$ is always an \emph{upper bound} on the value function of $\BG$, and second, by the termination criterion of \algname, the probability that a trajectory $(s_1, s_2, \ldots, s_H) \sim (\BG, \wh \pi)$ visits any state $s_h \not \in \wh \MV$ is small (Lemma \ref{lem:escape-prob}). These arguments are worked out in  detail in Lemma \ref{lem:cce-almost}. %

  \subsection{Implementing \algname in a decentralized manner}
  \label{sec:decentralized}
  So far, we have described \algname as a centralized algorithm, declining to make distinctions between the computations performed by each agent. We now proceed to explain how \algname can be implemented in a \emph{decentralized} way, namely in the following setting:
  \begin{enumerate}
  \item All agents know $m, \MS, \MA, H$, as well as $\ep, \delta$, so that they may each compute the additional parameters $K, N_{\rm visit}, p$ passed to \algname.
  \item For each trajectory of $\BG$ sampled from a policy $\pi$, each agent $i \in [m]$ sees only the states, their actions, and their rewards.
  \item The agents may access a common string of uniformly random bits during the course of the algorithm, but no communication between agents is allowed.\label{it:common-randomness} %
  \item The agents are required to be able to sample from the output policy $\wh \pi$, again using only common randomness (and no communication). 
  \end{enumerate}
The only existing decentralized learning algorithm for multi-player general-sum stochastic games, V-learning \cite{SongMB21,MaoB21,JinLWY21}, shares all requirements above except item  \ref{it:common-randomness}. In particular, while V-learning requires public random bits to sample a trajectory from its output CCE policy, such bits are not used in the process of learning the policy. 
  
  To implement \algname in a decentralized manner, we first describe how agents can sample trajectories from $\ol \pi$ (defined in step \ref{it:define-bar-pi}) without communicating: note that the first $h-1$ steps of $\ol \pi$ are given by $(\til \pi^q, \ldots, \til \pi_{h-1}^q)$, for some stage $q$: furthermore, $\til \pi_{h'}^q(\cdot |s)$ is a uniform mixture over some number $J_{h',s}^q$ of joint action profiles (step \ref{it:end-k-for}; we denote the parameter $J_{h',s}$ at stage $q$ by $J_{h',s}^q$). Thus, if each agent stores its action taken in each of the $J_{h',s}^q$ such steps for all $s, h', q$, the agents may draw an action sampled from $\til \pi_{h'}^q$ by using the public randomness to sample a uniformly random element of $[J_{h',s}^q]$. Noting that $J_{h',s}^q \leq \KMAX$ for all $h,s$, we see that the total number of common random bits needed to execute $\ol \pi$ is $O(H^3 S^2 K \log(SK))$.

  It is straightforward that the bandit updates in steps \ref{it:bandit-begin} through \ref{it:bandit-end} as well as the computation of $\Vo_{i,h}^q$ in steps \ref{it:define-vbar} through \ref{it:define-vbar-real} may be implemented in a decentralized way (in particular, each agent $i$ only computes its own value estimate $\Vo_{i,h}^q$). Finally, the procedure \EST allows each agent $i \in [m]$ to compute their own estimates of $\wh d_h^q$, for all $h,q$, which all coincide since the states drawn from each trajectory are common knowledge. In order to play the policy $\til \pi^q$ passed to \EST, the same strategy as described above may be used, which requires a total of $O(H^2 S \log(SK))$ bits of common randomness over all stages $q \geq 1$.

  Summarizing, executing \algname in a decentralized way requires $O(H^3 S^2 K \log(SK))$ bits of common randomness. For the $K$ described in Section \ref{sec:conc-ineq}, this leads to $\til O\left( \frac{H^7 S^3 \max_{i \in [m]} A_i}{\ep^3} \right)$ bits.

\section{Proofs for Section \ref{sec:lb-res}}\label{sec:proof-lb}

In this section, we prove Theorems \ref{thm:perfect-hardness} and \ref{thm:non-perfect-hardness}. To this end, we  show how to reduce solving an instance of $\ep$-\Gcircuit (or $(\ep, \delta)$-\Gcircuit; defined in Section \ref{sec:ppad-prelim} below) to computing the appropriate notion of equilibria in a 2-player stochastic game. The proof is organized into the following parts:
\begin{itemize}
\item In Section \ref{sec:ppad-prelim} we introduce some additional preliminaries regarding \PPAD and $(\ep, \delta)$-\Gcircuit, and in Section \ref{sec:add_prelims} we introduce some additional preliminaries regarding stochastic games.
\item In Section \ref{sec:equilibria-reduction}, we introduce a notion of Nash equilibrium in turn-based games, namely \emph{well-supported Nash equilibrium in stage games (WSNE-SG)}. We show (roughly speaking) that computing approximate WSNE-SG reduces to computing approximate NE in turn-based games, i.e., it suffices to show \PPAD-hardness for computing approximate WSNE-SG.
\item In Section \ref{sec:game-gadgets}, we show how each of the gates in Definition \ref{def:gcircuit} can be implemented via a \emph{gadget} with a constant number of states, transitions, and rewards in a turn-based stochastic game; these gadgets are similar in nature to those introduced by  \cite{daskalakis2009complexity} 
in showing that computing Nash equilibria in graphical games is \PPAD-complete.
\item In Section \ref{sec:valid-colorings}, we show how to combine the gadgets from Section \ref{sec:game-gadgets} to construct a turn-based stochastic game whose approximate WSNE-SG correspond to approximate assignments to a given \Gcircuit instance. 
\end{itemize}

Unless otherwise stated, the policies considered in this section are stationary Markov policies.

\subsection{Additional preliminaries for \PPAD}
\label{sec:ppad-prelim}
  For some $\ep > 0$ and reals $x,y$, we use $x = y \pm \ep$ to denote $x \in [y - \ep, y + \ep]$ throughout this section.\footnote{We remark that in \cite{R16}, $x \pm \ep$ was used to denote the fact that $x \in (y-\ep, y+\ep)$; this difference does not materially change any of the hardness results from \cite{R16}, since $\ep$ may be scaled down by any constant factor.} %

  \begin{defn}[Generalized circuit]
A generalized circuit $\MC = (V, \MG)$ is a finite set of \emph{nodes} $V$ and \emph{gates} $\MG$. Each gate in $\cG$  is characterized as $G(\ell | v_1, v_2 | v)$, where $G \in \{ G_{\gets}, G_{\times, +}, G_<\}$ denotes a gate type, $\ell \in \BR^\st$ is a vector of real parameters (perhaps of length 0), 
$v_1, v_2\in V \cup \{ \perp \}$ denote the gate's \emph{input nodes}, and $v \in V$ denotes the gate's \emph{output node}. The collection of gates $\MG$ satisfies the following property: for every two gates $G(\ell | v_1, v_2 | v)$ and $G'(\ell' | v_1' ,v_2' | v')$, it holds that $v \neq v'$ (i.e., each gate computes a distinct, and thus well-defined, output node). 
\end{defn}
For the purposes of proving hardness results, it is without loss of generality to assume that each node $v \in V$ is the output node of some gate in $\MG$: for each node which is not the output node of some gate, we can add a gate specifying that the node is equal to $1 \pm \ep$. The resulting circuit is still a valid instance of the generalized circuit problem, and it has a solution by Brouwer's fixed point theorem. Any such solution is certainly a valid solution to the original generalized circuit instance.

\begin{defn}[$(\ep, \delta)$-\Gcircuit]
  \label{def:gcircuit}
  Fix $\ep, \delta \in (0,1)$. 
  Given a generalized circuit $\MC = (V,\MG)$, an assignment $\pi : V \ra [0,1]$   is said to \emph{$\ep$-approximately satisfy} some gate $G \in \MG$, if the following holds: %
  \begin{itemize}
  \item If for some constant $\zeta \in \{0,1\}$, the  gate $G$ is of the form $G_{\gets}(\zeta || v)$, then  we have $\pi(v) = \zeta$;
  \item If for some constants $\xi, \zeta \in [-1,1]$, the gate $G$ is of the form $G_{\times, +}(\xi, \zeta | v_1, v_2 | v_3)$, then  we have $$\pi( v_3) = \max\left\{ \min \left\{ \xi \cdot \pi( v_1) + \zeta \cdot \pi(v_2), 1 \right\}, 0 \right\} \pm \ep.$$
  \item If the gate $G$ is of the form $G_<(|v_1, v_2  | v_3)$, then we have $\pi(v_3) = \begin{cases}
      1 \pm \ep, \quad \pi(v_1) \leq \pi(v_2) - \ep \\
      0 \pm \ep, \quad \pi(v_1) \geq \pi(v_2) + \ep.
    \end{cases}$
  \end{itemize}
  We stress that in the gates $G_{\times, +}(\xi, \zeta | v_1, v_2 | v_3)$ and in $G_<(|v_1, v_2 | v_3)$, we allow for $v_1 = v_2$. 
  The problem $(\ep, \delta)$-\Gcircuit is the following: {\it Given a generalized circuit $\MC = (V, \MG)$, find an assignment $\pi : V \ra [0,1]$ (represented in binary) which $\ep$-approximately satisfies all but a $\delta$-fraction of the gates in $\MG$. We then define the $\ep$-\Gcircuit problem to be the $(\ep, 0)$-\Gcircuit problem.} 
\end{defn}
The following theorem will be crucial for our hardness results. 
\begin{theorem}[\cite{rubinstein2018inapproximability}]
  \label{thm:gcircuit-ppad}
There is a constant $\ep > 0$ so that $\ep$-\Gcircuit is \PPAD-complete.
\end{theorem}
The problem of $(\ep, \delta)$-\Gcircuit for positive $\delta$ is not (yet) known to be \PPAD-hard, but it has been conjectured to be so:
\begin{conjecture}[PCP for \PPAD conjecture  \cite{babichenko2015can}]
  \label{con:pcp-ppad}
There are constants $\ep, \delta > 0$ so that \EOTL has a polynomial-time reduction to $(\ep, \delta)$-\Gcircuit. 
\end{conjecture}
We remark that Conjecture \ref{con:pcp-ppad} is slightly weaker (i.e., more plausible) than \cite[Conjecture 2]{babichenko2015can},  which states that the reduction be \emph{quasilinear}. 

Further, we remark that the definition of $\ep$-\Gcircuit in \cite{rubinstein2018inapproximability}  uses some additional gates; it is straightforward to see that these gates may be implemented using the gates in Definition \ref{def:gcircuit}, meaning that the $\ep$-\Gcircuit problem with the set of gates listed above is still \PPAD-complete for constant $\ep$ (and Conjecture \ref{con:pcp-ppad} is implied by \cite[Conjecture 2]{babichenko2015can}); for completeness, we have presented the details of this reduction in Appendix \ref{sec:gcircuit-gates}.

\subsection{Additional preliminaries for stochastic games}\label{sec:add_prelims}

In this section, we introduce some additional definitions and lemmas which will be helpful in our proofs. Consider a  infinite-horizon discounted  stochastic game $\BG$ and a stationary policy $\pi \in \Delta(\MA)^\MS$.  Then one can define the state-action value function $Q_i^\pi: \MS\times\MA \ra [-1,1]$ under policy $\pi$ as
  \$
  Q_i^\pi(s,\ba) := (1-\gamma) \cdot \E_{(s_1, \ba_1, s_2, \ba_2, \ldots) \sim (\BG, \pi)}\left[ \sum_{h=1}^\infty \gamma^{h-1} r_i(s_{h}, \ba_{h})\bigggiven s_1=s,\ba_1=\ba\right]\nonumber,
  \$ 
  which corresponds to the $\gamma$-discounted cumulative reward starting from $(s,\ba)$.

  The \emph{state-visitation distribution} under the policy $\pi$ given that the initial state is $s$ is defined as follows:
\begin{align}
d_s^\pi(s') := (1-\gamma) \cdot \sum_{h=1}^\infty \gamma^{h-1} \cdot \Pr_{s_{h} \sim (\BG, \pi)} \left[ s_{h} = s' | s_1 = s\right]\nonumber.
\end{align}
Note that $d_s^\pi$ is a valid distribution over $\cS$, i.e., $\sum_{s' \in \MS} d_s^\pi(s') = 1$.

\begin{lemma}[Performance difference lemma \cite{Kakade02approximatelyoptimal}]
  \label{lem:pd}
  Consider any two stationary policies $\pi : \MS \ra \Delta(\MA),\ \pi' : \MS \ra \Delta(\MA)$. Then for all $s_1 \in \MS$ and $i \in [m]$,
  \begin{align}
V_i^\pi(s_1) - V_i^{\pi'}(s_1) = \frac{1}{1-\gamma} \cdot \E_{s \sim d_{s_1}^\pi} \E_{\ba \sim \pi(s)} [Q_i^{\pi'}(s,\ba) - V_i^{\pi'}(s)]\nonumber.
  \end{align}
\end{lemma}

The following lemma is standard but we give a proof for completeness:
\begin{lemma}
  \label{lem:lip-dirpar}
  For policies $\pi, \pi' \in \Delta(\MA)^\MS$, it holds that, for all states $s \in \MS$, joint actions $\ba \in \MA$, and agents $i \in [m]$,
  \begin{align}
    \left| V_i^\pi(s)  - V_i^{\pi'}(s) \right| &\leq \frac{1}{1-\gamma} \cdot \max_{s' \in \MS} \| \pi(\cdot | s') - \pi'(\cdot | s') \|_1\nonumber\\
    \left| Q_i^\pi(s,\ba) - Q_i^{\pi'}(s,\ba) \right| &\leq  \frac{\gamma}{1-\gamma} \cdot \max_{s' \in \MS} \| \pi(\cdot | s') - \pi'(\cdot | s') \|_1\nonumber.
  \end{align}
\end{lemma}
\begin{proof}
  Let us view $\pi \in \Delta(\MA)^\MS \subset \BR^{S \times A}$ as a vector whose components are $\pi(\ba | s)$, for all $(s,\ba) \in \MS \times \MA$. Then by policy gradient theorem \cite{sutton1999policy},  we have that for all $\pi$ and states $s_1 \in \MS$,
  \begin{align}
    \grad_\pi V^\pi(s_1) =& \frac{1}{1-\gamma} \cdot \E_{s \sim d_{s_1}^\pi} \E_{\ba \sim \pi(\cdot | s)} \left[ \grad_\pi \log \pi(\ba|s) \cdot Q^\pi(s,\ba) \right]\nonumber\\
    =& \frac{1}{1-\gamma} \cdot \E_{s \sim d_{s_1}^\pi} \left[ \sum_{\ba \in \MA} \grad_\pi \pi(\ba|s) \cdot Q^\pi(s,\ba) \right]\nonumber\\
    =& \frac{1}{1-\gamma} \cdot \E_{s \sim d_{s_1}^\pi} \left[ \sum_{\ba \in \MA} e_{(s,\ba)} \cdot Q^\pi(s,\ba) \right]\nonumber,
  \end{align}
  where $e_{(s,\ba)}$ denotes an all-zero vector except that the $(s,\ba)$ component is one. 
  For a vector $v \in \BR^{S \times A}$ and $s \in \MS$, write $v_s := (v_{s,\ba})_{\ba \in \MA} \in \BR^A$. 
  Since $|Q^\pi(s,\ba)|\leq 1$ for all $s,\ba$ and $\sum_s d_{s_1}^\pi(s) = 1$, it holds that $\sum_{s \in \MS}\| (\grad_\pi V^\pi(s_1))_s \|_\infty \leq 1/(1-\gamma)$. Furthermore, note that for any vectors $v, w \in \BR^{S \times A}$, we have
  \begin{align}
| \lng v, w \rng | \leq \sum_{s \in \MS}|  \lng v_s, w_s \rng | \leq \sum_{s \in \MS} \| v_s \|_\infty \cdot \| w_s \|_1 \leq \max_{s \in \MS} \{\|w_s\|_1\} \cdot \sum_{s \in \MS} \| v_s \|_\infty\nonumber.
  \end{align}

  It follows that for all $s \in \MS$ and $i \in [m]$, 
  \begin{align}
| V_i^\pi(s) - V_i^{\pi'}(s) | \leq \max_{\til \pi \in \Delta(\MA)^\MS} | \lng \pi - \pi', \grad_{\til \pi} V^{\til \pi}(s) \rng | \leq \frac{1}{1-\gamma} \cdot \max_{s' \in \MS} \| \pi(\cdot | s') - \pi'(\cdot | s') \|_1\nonumber,
  \end{align}
  verifying the first claim of the lemma. The second claim follows as a consequence of the first and the fact that for all policies $\pi$, $Q^\pi_i(s,\ba) =  (1-\gamma) \cdot r_i(s,\ba) + \gamma \cdot \E_{s' \sim \PP(\cdot | s,\ba)} [V_i^\pi(s')]$. 
\end{proof}

\subsection{Reductions between notions of equilibria}
\label{sec:equilibria-reduction}
We begin by introducing a variant of Nash equilibrium which holds with respect to each stage game:
\begin{defn}
  \label{def:ne-sg}
  We say that a product Markov policy $\pi : \MS \ra \Delta(\MA)$ is an \emph{$\ep$-perfect NE in stage games} (abbreviated \emph{$\ep$-PNE-SG}) if for all states $s \in \MS$ and all agents $i \in [m]$,
    \begin{align}
    \label{eq:stage-condition}
\max_{a_i \in \MA_i} \E_{\ba_{-i} \sim \pi_{-i}(s)} \left[ Q^\pi_i(s, (a_i, \ba_{-i})) \right] - V_i^\pi(s) \leq \ep.
    \end{align}
    We also say that $\pi$ is an \emph{$\ep$-NE in stage games} (abbreviated \emph{$\ep$-NE-SG}) if for all agents $i \in [m]$,
        \begin{align}
    \label{eq:stage-condition-weak}
 \E_{s \sim \mu} \max_{a_i \in \MA_i} \E_{\ba_{-i} \sim \pi_{-i}(s)} \left[ Q^\pi_i(s, (a_i, \ba_{-i})) \right] - V_i^\pi(s) \leq \ep.
    \end{align}
  \end{defn}

  The below lemma reduces the problem of computing an $\ep$-(P)NE-SG to computing an $\ep$-(perfect) NE.
  \begin{lemma}
    \label{lem:reduce-to-sg}
    Consider a stationary product policy $\pi \in \Delta(\MA)^\MS$. Then:
    \begin{itemize}
    \item If $\pi$ is an $\ep$-perfect NE, then it is an $\ep$-PNE-SG.
    \item If $\pi$ is an $\ep$-NE, then it is an $\ep$-NE-SG.
    \end{itemize}
\end{lemma}
\begin{proof}
  For each $i \in [m], s \in \MS$, define 
  \begin{align}
    \pi_i^\dagger(s) \in  &\argmax_{a_i \in \MA_i} \E_{\ba_{-i} \sim \pi_{-i}(s)} [ Q_i^\pi(s, (a_i, \ba_{-i}))]  \label{equ:pi_BR_Q_i}\\ 
\rho_{i,s} :=&  \max_{a_i \in \MA_i} \E_{\ba_{-i} \sim \pi_{-i}(s)} [ Q_i^\pi(s, (a_i, \ba_{-i}))] - V_i^\pi(s).\nonumber
  \end{align}
  Since $V_i^\pi(s) = \E_{\ba \sim \pi(s)} [ Q_i^\pi(s, \ba)]$ and $\pi_i(s)$ is a product distribution for all $i,s$, it holds that for all $i,s$,
  \begin{align}
\rho_{i,s} = \E_{\ba \sim (\pi_i^\dagger \times \pi_{-i})(s)} [Q_i^\pi(s, \ba) - V_i^\pi(s)] = \E_{\ba_{-i} \sim \pi_{-i}(s)} [ Q_i^\pi(s, (\pi_i^\dagger(s), \ba_{-i})) - V_i^\pi(s)] \geq 0.\label{eq:pidagger-improve}
  \end{align}
  
  By the performance difference lemma (Lemma \ref{lem:pd}), we have, for all $i \in [m], s \in \MS$, 
  \begin{align}
    V_i^{\pi_i^\dagger \times \pi_{-i}}(s) - V_i^\pi(s) =& \frac{1}{1-\gamma} \cdot \E_{s' \sim d_s^{\pi_i^\dagger \times \pi_{-i}}} \E_{\ba \sim \pi_i^\dagger \times \pi_{-i}(s')} [ Q_i^\pi(s', \ba) - V_i^\pi(s')]\nonumber\\
    \geq & \E_{\ba \sim (\pi_i^\dagger \times \pi_{-i})(s)} [ Q_i^\pi(s, \ba) - V_i^\pi(s)] = \rho_{i,s},\nonumber
  \end{align}
  where the second inequality follows from (\ref{eq:pidagger-improve}) and the fact that $d_s^\pi(s) \geq 1-\gamma$ by definition of $d_s^\pi$.

  If $\pi$ is an $\ep$-perfect NE, then we have that $V_i^{\pi_i^\dagger \times \pi_{-i}}(s) - V_i^\pi(s) \leq \ep$ for all $i,s$, which therefore, implies that for all $i,s$, we have $\rho_{i,s} \leq \ep$, i.e., (\ref{eq:stage-condition}) holds, and therefore $\pi$ is an $\ep$-PNE-SG.

  If $\pi$ is an $\ep$-NE, then for all agents $i \in [m]$, we know from Lemma \ref{lem:pd} that  
  \begin{align}
   \ep\geq  \EE_{s\sim\mu}\Big[V_i^{\pi_i^\dagger \times \pi_{-i}}(s) - V_i^\pi(s)\Big] =& \frac{1}{1-\gamma} \cdot \E_{s' \sim d_{\mu}^{\pi_i^\dagger \times \pi_{-i}}} \E_{\ba \sim \pi_i^\dagger \times \pi_{-i}(s')} [ Q_i^\pi(s', \ba) - V_i^\pi(s')]\nonumber\\
    \geq & \EE_{s\sim\mu}\E_{\ba \sim (\pi_i^\dagger \times \pi_{-i})(s)} [ Q_i^\pi(s, \ba) - V_i^\pi(s)] = \EE_{s\sim\mu}\big[\rho_{i,s}\big],\nonumber
  \end{align}
  where we use the fact that $d_\mu^\pi(s)\geq (1-\gamma)\mu(s)$  for all $s$. This shows  that $\pi$ is an $\ep$-NE-SG.   
    \end{proof}

    Next we introduce a well-supported variant  of the stage-game Nash equilibrium of Definition \ref{def:ne-sg}.
    \begin{defn}
  \label{def:wsne-sg}
  Consider a product Markov policy $\pi \in \Delta(\MA)^\MS$. For each $i \in [m]$ and $s \in \MS$, define
  \begin{align}
\ep_{i,s} := \max_{a_i' \in \MA_i} \E_{\ba_{-i} \sim \pi_{-i}(s)} \left[ Q_i^\pi(s, (a_i', \ba_{-i})) \right] - \min_{a_i \in \MA_i : \pi_i(a_i | s) > 0} \E_{\ba_{-i} \sim \pi_{-i}(s)} \left[ Q_i^\pi(s, (a_i, \ba_{-i})) \right]\label{eq:wsne-sg}.
  \end{align}
  We say that:
  \begin{itemize}
  \item $\pi$ is an \emph{$\ep$-perfect well-supported Nash equilibrium in stage games} (abbreviated $\ep$-PWSNE-SG) if $\max_{i\in[m],s\in\cS} \ep_{i,s} \leq \ep$;
  \item $\pi$ is an \emph{$\ep$-well-supported Nash equilibrium in stage games} (abreviated $\ep$-WSNE-SG) if $\max_{i \in [m]} \E_{s \sim \mu} [ \ep_{i,s}] \leq \ep$.
  \end{itemize}
\end{defn}

We next reduce the problem of computing a well-supported Nash equilibrium in stage games to that of computing a Nash equilibrium in stage games. 
    \begin{lemma}
      \label{lem:convert-to-ws} 
      Suppose the stochastic game $\BG$ satisfies the property that at each state, all but $p$ players have trivial action space (i.e., equal to a singleton). 
      Given a product stationary policy $\pi : \MS \ra \Delta(\MA)$, then:
      \begin{itemize}
      \item If $\pi$ is an $\ep$-PNE-SG, we can construct in polynomial time a policy $\pi' : \MS \ra \Delta(\MA)$ which is an $6 \cdot \sqrt{\frac{p\ep}{1-\gamma}}$-PWSNE-SG.
      \item If $\pi$ is an $\ep$-NE-SG, we can construct in polynomial time a policy $\pi' : \MS \ra \Delta(\MA)$ which is an $6 \cdot \sqrt{\frac{p\ep}{1-\gamma}}$-WSNE-SG.
      \end{itemize}
\end{lemma}
\begin{proof}
  We will use the following shorthand notation in the proof: for a Markov product  policy $\pi$, a state $s \in \MS$, an agent $i \in [m]$, and an action $a_i \in \MA_i$, we write (with a slight abuse of notation): 
  \begin{align}
Q_i^\pi(s, a_i) := \E_{\ba_{-i} \sim \pi_{-i}(s)} \left[ Q_i^\pi(s, (a_i, \ba_{-i})) \right]\nonumber.
  \end{align}
  
  Fix a stationary product policy $\pi : \MS \ra \Delta(\MA_1) \times \cdots \times \Delta(\MA_m)$. For each $i \in [m]$ and $s \in \MS$, define
  \begin{align}
\rho_{i,s} := \max_{a_i \in \MA_i} \E_{\ba_{-i} \sim \pi_{-i}(s)} [ Q_i^\pi(s, (a_i, \ba_{-i})) ]  - V_i^\pi(s) = \max_{a_i \in \MA_i} Q_i^\pi(s,a_i) - \E_{a_i \sim \pi_i(s)} [Q_i^\pi(s, a_i)] \nonumber.
  \end{align}
  In the case that $\pi$ is an $\ep$-PNE-SG, we have that $\max_{i,s} \rho_{i,s} \leq \ep$, and in the case that $\pi$ is an $\ep$-NE-SG, we have that $\max_i \E_{s \sim \mu} [\rho_{i,s}] \leq \ep$. 

  Fix some $k > 1$ to be specified later. We construct a new product policy $\pi'$ as follows: for each $i \in [m]$, $s \in \MS$, and $a_i \in \MA_i$,
  \begin{align}
    \pi_i'(a_i | s) := \begin{cases}
      \frac{\pi_i(a_i | s)}{1 - \bar \pi_i(s)} & : \quad Q_i^\pi(s, a_i) \geq \max_{a_i' \in \MA_i} \{ Q_i^\pi(s, a_i') \} - k \cdot \rho_{i,s} \\
      0 & : \quad \mbox{otherwise},
      \end{cases}\nonumber
  \end{align}
  where $\bar \pi_i(s)$ is the sum, over all $a_i$ so that $Q_i^\pi(s,a_i) < \max_{a_i' \in \MA_i} \{ Q_i^\pi(s, a_i') \} - k\rho_{i,s}$, of $\pi_i(a_i | s)$.

 Next, using the fact that $\pi$ is an $\ep$-NE-SG, we have the following claim:
  \begin{lemma}[\cite{daskalakis2009complexity}, Claim 6]
    For all $i \in [m],\ s \in \MS$, it holds that
    \begin{align}
\sum_{a_i \in \MA_i} \left| \pi_i'(a_i | s) - \pi_i(a_i | s) \right| \leq \frac{2}{k-1}\nonumber.
    \end{align}
  \end{lemma}
  Thus
  \begin{align}
\max_{s \in \MS} \| \pi(s) - \pi'(s) \|_1 \leq \max_{s \in \MS} \sum_{i\in [m]} \| \pi_i(s) - \pi_i'(s) \|_1  \leq \frac{2p}{k-1}\nonumber.
  \end{align}
  Hence, for all $s \in \MS,i \in [m], a\in \MA$, by Lemma \ref{lem:lip-dirpar}, %
  \begin{align}
\left| Q_i^\pi(s,a) - Q_i^{\pi'}(s,a) \right|  \leq \frac{2p \gamma}{(k-1)(1-\gamma)}\nonumber,
  \end{align}
  which implies that for all $a_i \in \MA_i$, we have
  \begin{align}
| Q_i^\pi(s,a_i) - Q_i^{\pi'}(s, a_i)| \leq & \| \pi_{-i}(\cdot | s) - \pi_{-i}'(\cdot | s) \|_1 + \frac{2p \gamma}{(k-1)(1-\gamma)} \leq \frac{2p}{(k-1)(1-\gamma)}\nonumber.
  \end{align}
  Thus, for all $s \in \MS,i \in [m]$, and $a_i \in \MA_i$ so that $\pi_i'(a_i | s) > 0$, we have
  \begin{align}
    Q_i^{\pi'}(s,a_i) \geq & Q_i^\pi(s,a_i) - \frac{2p}{(k-1)(1-\gamma)} \nonumber\\
    \geq & \max_{a_i' \in \MA_i} \{ Q_i^\pi(s, a_i') \} - k\rho_{i,s} - \frac{2p}{(k-1)(1-\gamma)} \nonumber\\
    \geq & \max_{a_i' \in \MA_i} \{ Q_i^{\pi'}(s, a_i')\} - k\rho_{i,s} - \frac{8p}{k(1-\gamma)}\nonumber.
  \end{align}
  Let us now choose $k = \sqrt{\frac{8p}{(1-\gamma)\ep}}$. Then in the case that $\pi$ is an $\ep$-PNE-SG, we have $\rho_{i,s} \leq \ep$ for all $i,s$, and we immediately obtain the desired result.

  If $\pi$ is only an $\ep$-NE-SG, then for all $i \in [m]$,
  \begin{align}
    \E_{s \sim \mu} \left[ \max_{a_i' \in \MA_i} Q_i^\pi(s, a_i') - \min_{a_i \in \MA_i : \pi_i(a_i | s) > 0} Q_i^\pi(s, a_i) \right] \leq & \E_{s \sim \mu} \left[ k \rho_{i,s} + \frac{8p}{k(1-\gamma)} \right]\nonumber\\
    \leq & k \ep + \frac{8p}{k(1-\gamma)}\nonumber,
  \end{align}
  which gives that $\pi'$ is an $6 \cdot \sqrt{\frac{p\ep}{1-\gamma}}$-WSNE-SG. 
\end{proof}

Combining the results presented in this section, we have the following:
\begin{lemma}
  \label{lem:reduce-to-ws}
  Consider a turn-based stochastic game $\BG$. Given a product stationary policy $\pi \in \Delta(\MA)^\MS$, the following statements hold:
  \begin{itemize}
  \item If $\pi$ is an $\ep$-perfect NE, then we can construct in polynomial time a policy $\pi' \in \Delta(\MA)^\MS$ which is a $6 \cdot \sqrt{\frac{\ep}{1-\gamma}}$-PWSNE-SG.
  \item If $\pi$ is an $\ep$-NE, then we can construct in polynomial time a policy $\pi' \in \Delta(\MA)^\MS$ which is a $6 \cdot \sqrt{\frac{\ep}{1-\gamma}}$-WSNE-SG.
  \end{itemize}
\end{lemma}
\begin{proof}
The lemma is an immediate consequence of Lemmas \ref{lem:reduce-to-sg} and \ref{lem:convert-to-ws}, noting that since $\BG$ is turn-based we may take $p=1$ in Lemma \ref{lem:convert-to-ws}. 
\end{proof}

\subsection{Implementing the gates with ``stochastic game gadgets''}
\label{sec:game-gadgets}
In this section, we introduce several ``game gadgets'' which show how to implement each of the arithmetic gates of Definition \ref{def:gcircuit} using a constant number of states in a turn-based stochastic game.

\paragraph{Notation for turn-based games.} Throughout this section, we will consider an $m$-player turn-based stochastic game $\BG=  (\MS, (\MA_i)_{i \in [m]}, \BP, (r_i)_{i \in [m]}, \gamma, \mu)$. (Though we will eventually take $m = 2$, we will introduce the gadgets in this section for games with an arbitrary number of players.) In our construction, we will have that for each player $i \in [m]$, $\MA_i = \{0,1\}$. Since, at each state $s \in \MS$, there is a single agent (namely, $\ag{s}$) whose action affects the reward and transition at that state, the value functions induced by a stationary policy $\pi : \MS \ra \Delta(\MA)$ depend only on the values of $\pi_{\ag{s}}(s)$, for each $s \in \MS$. Thus, we may represent such a policy as a mapping from $\MS$ to $[0,1] = \Delta(\{0,1\})$; with slight abuse of notation, we denote this mapping also as $\pi : \MS \ra [0,1]$. In particular, $\pi(s)$ is to be interpreted as the probability that agent $\ag{s}$ plays the action $1 \in \MA_{\ag{s}} = \{0,1\}$ at state $s$.

We will furthermore work with games $\BG$ that have a designated sink state $\sinkz \in \MS$, so that $\BP(\sinkz | \sinkz, \ba) = 1$ for all $\ba \in \MA$, and $r_i(\sinkz, \ba) = 0$ for all $i \in [m], \ba \in \MA$. We allow $\ag{\sinkz}$ to be arbitrary; the value of $\ag{\sinkz}$ will have no relevance to any of our results. %
In particular, whenever the system reaches state $\sinkz$, it stays there for all future steps and all agents accumulate 0 additional reward.

\paragraph{Unimprovable states.} Given a turn-based game $\BG$ and a policy $\pi : \MS \ra [0,1]$, note that, for each $s \in \MS$ and $\ba = (a_1, \ldots, a_m) \in \MA$, $Q_{\ag{s}}^\pi(s, \ba)$ depends only on $a_{\ag{s}} \in \MA_{\ag{s}}$. Thus, to simplify notation, we will use  $Q_{\ag{s}}^\pi(s, a_{\ag{s}}) := Q_{\ag{s}}^\pi(s, \ba)$. We say that a state $s \in \MS$ is \emph{$\ep$-unimprovable under $\pi$} if
\begin{align}
\max_{a' \in \{0,1\}} Q_{\ag{s}}^\pi(s, a') - \min_{a \in \{0,1\} : \pi(s) \neq 1-a } Q_{\ag{s}}^\pi(s,a) \leq \ep.\label{eq:ep-satisfied}
\end{align}
Note that the set of actions $a \in \{0,1\}$ so that $\pi(s) \neq 1-a$ is exactly the set of actions on which $\pi(s)$ puts positive probability. Hence the quantity in (\ref{eq:ep-satisfied}) is identical to the quantity $\ep_{i,s}$ defined more generally in (\ref{eq:wsne-sg}). Thus, if $\pi$ is an $\ep$-PWSNE-SG (Definition \ref{def:wsne-sg}), it holds that all states are $\ep$-unimprovable. Furthermore, if $\pi$ is an $\ep$-WSNE-SG, then by Markov's inequality, for all $k \geq 1$, a fraction $1-1/k$ of states are $\ep \cdot k$-unimprovable. 

\paragraph{Implementing the $G_{\times, +}$ gate.}
We begin by defining a gadget which implements the $G_{\times, +}(\xi, \zeta | v_1, v_2 | v_3)$ gate: in particular, the gadget will consist of states $v_1, v_2, v_3$ of the stochastic game $\BG$ together with a helper state $w$ of $\BG$. We will choose the transitions and rewards of $\BG$ so that, roughly speaking, for any equilibrium policy $\pi : \MS \ra [0,1]$, $\pi(v_3)$ is close to $\max\{ \min\{ \xi \cdot \pi(v_1) + \zeta \cdot \pi(v_2), 1 \}, 0\}$. 
\begin{defn}
  \label{def:g-times-plus}
Consider any $\alpha, \psi, \beta \in \BR$, each with absolute value at most $1-\gamma$. We say that a $G_{\times,+}(\frac{\alpha}{2\beta}, \frac{\psi}{2\beta}|v_1, v_2 | v_3)$ gate \emph{embeds in a stochastic game $\BG$ via the states $(v_1, v_2, v_3, w)$ and the constants $(\alpha, \psi, \beta)$}, for states $v_1, v_2,v_3, w \in \MS$ if the following holds: %
  \begin{enumerate}
  \item The transitions out of $v_3$ and $w$ satisfy the following 
    \begin{itemize}
    \item $\BP(v_1 | w, 0) = \min \{ \frac 12, \frac{|\alpha|}{2|\beta|} \}$, $ \BP(v_2 | w, 0) = \min \{ \frac 12, \frac{|\psi|}{2|\beta|}\}$, $\BP(\sinkz | w,0) = 1 - \BP(v_1 | w,0) - \BP(v_2 | w,0)$, and $\BP(v_3 | w, 1) = 1$;
    \item $\BP(w | v_3,  0) = 1$, and $\BP(\sinkz | v_3, 1) = 1$. 
    \end{itemize}
  \item The rewards to the players controlling $v_3, w$ at states $v_1, v_2, v_3, w$ satisfy the following: 
    \begin{itemize}
    \item $r_\contr{w}(v_1, 1) = \frac{\alpha \cdot \max \{ 1, \frac{|\beta|}{|\alpha|}\} }{1-\gamma}$, $r_\ag{w}(v_2, 1) = \frac{\psi \cdot \max \{ 1, \frac{|\beta|}{|\psi|}\}}{1-\gamma}$ 
, and $r_\contr{w}(v_3, 1) = \frac{\beta}{1-\gamma}$;
      \item  $r_\contr{v_3}(w,1) = \frac{\beta}{1-\gamma}$ and $r_\contr{v_3}(w, 0) = -\frac{\beta}{1-\gamma}$;
      \item %
        For all $a \in \{0,1\}$, it holds that
        \begin{align}
r_{\ag{w}}(v_1, 0) = r_{\ag{w}}(v_2, 0)= r_{\ag{w}}(v_3, 0)= r_{\ag{w}}(w,a) =  r_{\ag{v_3}}(v_3, a) = 0.\nonumber
        \end{align}
    \end{itemize}
  \end{enumerate}
Recall from Definition \ref{def:gcircuit} that we allow for $v_1 = v_2$ in the gate $G_{\times, +}(\xi, \zeta | v_1, v_2 | v_3)$. If this is the case, we require that $\alpha = \psi$ and instead require that $\BP(v_1|w,0) = \BP(v_2|w,0) = 2 \cdot \min \left\{ \frac 12, \frac{|\alpha|}{2|\beta|} \right\}$ above. 
\end{defn}
In the context of Definition \ref{def:g-times-plus}, we will at times refer to $w$ as the \emph{helper node} for the embedded gate.
\begin{lemma}
  \label{lem:g-times-plus}
  Suppose $G_{\times,+}(\frac{\alpha}{2\beta}, \frac{\psi}{2\beta} | v_1, v_2 | v_3)$  embeds in a stochastic game $\BG$ via the tuple $(v_1, v_2, v_3, w)$ and the vector $(\alpha, \psi, \beta)$, and consider any $\ep, \ep' \in (0,1)$. Suppose that  $\gamma |\beta| \ep - 2\gamma^2 > \ep'$.  Then for any policy $\pi : \MS \ra [0,1]$ for which $v_3, w$ are each $\ep'$-unimprovable under $\pi$, it holds that
  \begin{align}
    \pi( v_3) = \max\left\{ \min \left\{ \frac{\alpha}{2\beta} \cdot \pi( v_1) + \frac{\psi}{2\beta} \cdot \pi(v_2), 1 \right\}, 0 \right\} \pm \ep.\nonumber
  \end{align}
\end{lemma}
\begin{proof}
  Since $\gamma |\beta| \ep - 2\gamma^2 > \ep'$ and $\ep < 1$, we have that $\gamma |\beta| - \gamma^2 > \ep'>0$ and  $\beta\neq 0$.  

  Consider any  policy $\pi : \MS \ra [0,1]$. %
  First note that, since $r_{\ag{w}}(v_1, 0) = r_{\ag{w}}(v_2, 0) = 0$, %
  it holds that  
  \begin{align}
    V_{\contr{w}}^\pi(v_1) =&  (1-\gamma) \cdot \pi(v_1) \cdot r_{\contr{w}}(v_1, 1) \pm \gamma = \alpha \cdot \max\left\{ 1, \frac{|\beta|}{|\alpha|} \right\} \cdot \pi(v_1) \pm \gamma\label{eq:vwpi-1}\\
    V_\ag{w}^\pi(v_2) =& (1-\gamma) \cdot \pi(v_2) \cdot r_\ag{w}(v_2, 1) \pm \gamma = \psi \cdot \max \left\{ 1, \frac{|\beta|}{|\psi|} \right\} \cdot \pi(v_2) \pm \gamma\label{eq:vwpi-2}.
  \end{align}
  In particular, we have used that the total contribution of rewards of $\ag{w}$ to $V_{\ag{w}}^\pi(v_1)$ (respectively, to $V_{\ag{w}}^\pi(v_2)$) at states at least 1 step out from $v_1$ (respectively, $v_2$) is at most $(1-\gamma) \cdot (\gamma + \gamma^2 + \cdots) = \gamma$. 

  We first compute $Q_{\ag{w}}^\pi(w,b)$ for $b \in \{0,1\}$, as follows. 
  Using that $r_{\ag{w}}(v_3, 0) = r_{\ag{w}}(w,0) = r_{\ag{w}}(w,1) = 0$, we have:
  \begin{itemize}
\item $Q_{\ag{w}}^{\pi}(w,1) =  \gamma \cdot \pi(v_3) \cdot \beta \pm \gamma^2$. %
\item $Q_{\ag{w}}^{\pi}(w,0) = \frac 12 \gamma \alpha \cdot \pi(v_1) + \frac 12 \gamma \psi \cdot \pi(v_2) \pm \gamma^2$, which may be seen as follows:
  \begin{align}
    Q_{\ag{w}}^{\pi}(w,0) =& \gamma \cdot \min \left\{ \frac 12, \frac{|\alpha|}{2|\beta|} \right\} \cdot V_{\ag{w}}^{\pi}(v_1) + \gamma \cdot \min \left\{ \frac{1}{2}, \frac{|\psi|}{2|\beta|} \right\} \cdot V_{\ag{w}}^{\pi}(v_2) \nonumber\\ %
    =& \gamma \cdot \min \left\{ \frac 12, \frac{|\alpha|}{2|\beta|} \right\} \cdot \left( \alpha \cdot \max\left\{ 1, \frac{|\beta|}{|\alpha|} \right\} \cdot \pi(v_1) \pm \gamma \right) \nonumber\\
    & + \gamma \cdot \min \left\{ \frac{1}{2}, \frac{|\psi|}{2|\beta|} \right\} \cdot \left(\psi \cdot \max \left\{ 1, \frac{|\beta|}{|\psi|} \right\} \cdot \pi(v_2) \pm \gamma\right) \nonumber\\
    =& \frac 12 \gamma \alpha \cdot \pi(v_1) + \frac 12 \gamma \psi \cdot \pi(v_2) \pm \gamma^2.\nonumber
  \end{align}
\end{itemize}
For computation of $Q_{\ag{w}}^{\pi}(w,0)$ above, we have used (\ref{eq:vwpi-1}) and (\ref{eq:vwpi-2}), as well as the fact that $\min \left\{ \frac 12, \frac{|\alpha|}{2|\beta|} \right\} \cdot \max \left\{ 1, \frac{|\beta|}{|\alpha|} \right\} = \frac{1}{2}$ (and an analogous equality clearly holds with $\psi$ replacing $\alpha$). 

  We next compute $Q_{\ag{v_3}}^{\pi}(v_3,b)$ for $b \in \{0,1\}$ in the particular case where  $\pi(w) \in \{0,1\}$, using that $r_{\ag{v_3}}(v_3, a) = 0$ for each $a \in \{0,1\}$: %
  \begin{itemize}
  \item If $\pi(w) = 1$, then:
    \begin{itemize}
    \item $Q_{\ag{v_3}}^{\pi}(v_3,0) = {\gamma \cdot \beta} \pm \gamma^2$;
    \item $Q_{\ag{v_3}}^{\pi}(v_3,1) = 0$.
    \end{itemize}
  \item If $\pi(w) = 0$, then:
    \begin{itemize}
    \item $Q_{\ag{v_3}}^{\pi}(v_3,0) = {-\gamma \cdot \beta} \pm \gamma^2$;
    \item $Q_{\ag{v_3}}^{\pi}(v_3,1) = 0$.      
    \end{itemize}
  \end{itemize}

  We consider two cases, depending on the sign of $\beta$.

\vspace{7pt}
 \noindent {\bf Case 1: $\beta > 0$.} If $\pi(v_3) >\max\{ \frac{\alpha}{2\beta} \cdot \pi(v_1 ) + \frac{\psi}{2\beta} \cdot \pi(v_2),0 \} + \ep$, then $\beta \cdot \pi(v_3) > \frac{\alpha}{2} \cdot \pi(v_1) + \frac{\psi}{2} \cdot \pi(v_2) + \beta\ep$, and so we have %
  \begin{align}
    Q_{\ag{w}}^{\pi}(w,1) - Q_{\ag{w}}^{\pi}(w,0) \geq& \gamma \cdot \left(\beta \cdot \pi(v_3) - \frac 12 \alpha \cdot \pi(v_1) - \frac 12 \psi \cdot \pi(v_2)\right) -2\gamma^2  >  \gamma \beta \cdot \ep - 2\gamma^2 > \ep'\nonumber,
  \end{align}
  which implies, since $w$ is $\ep'$-unimprovable under $\pi$, that $\pi(w) = 1$. %
  But then
  \begin{align}
Q_{\ag{v_3}}^{\pi}(v_3,0) - Q_{\ag{v_3}}^{\pi}(v_3,1) \geq \gamma \beta - \gamma^2 > \ep'\nonumber,
  \end{align}
  which implies that, since $v_3$ is $\ep'$-unimprovable under $\pi$, $\pi(v_3) = 0$. But we have assumed above that $\pi(v_3) > \max\{\frac{\alpha}{2\beta} \cdot \pi(v_1) + \frac{\psi}{2\beta} \cdot \pi(v_2), 0\} + \ep > 0$, which is a contradiction. 

  Next suppose that $\pi(v_3) < \min \{ \frac{\alpha}{2\beta} \cdot \pi(v_1) + \frac{\psi}{2\beta} \cdot \pi(v_2), 1 \} - \ep$, which implies that $\beta \cdot \pi(v_3) \leq \frac{\alpha}{2} \cdot \pi(v_1) + \frac{\psi}{2} \cdot \pi(v_2) - \beta\ep$.  Then 
  \begin{align}
    Q_{\ag{w}}^{\pi}(w,0) - Q_{\ag{w}}^{\pi}(w,1) \geq & \gamma \cdot \left(\frac 12 \alpha \cdot \pi(v_1) + \frac 12 \psi \cdot \pi(v_2) - \beta \cdot \pi(v_3)\right) - 2 \gamma^2 \geq \gamma \beta \cdot \ep  - 2 \gamma^2 > \ep'\nonumber,
  \end{align}
  which implies that, since $w$ is $\ep'$-unimprovable under $\pi$, $\pi(w) = 0$. But then
  \begin{align}
Q_{\ag{v_3}}^{\pi}(v_3,1) - Q_{\ag{v_3}}^{\pi}(v_3,0) \geq \gamma \beta - \gamma^2 > \ep',\nonumber
  \end{align}
  which implies that, since $v_3$ is $\ep'$-unimprovable under $\pi$, $\pi(v_3) = 1 > 1-\ep$, a contradiction to $\pi(v_3) < \min \{ \frac{\alpha}{2\beta} \cdot \pi(v_1) + \frac{\psi}{2\beta} \cdot \pi(v_2), 1 \} - \ep$.

\vspace{7pt}
 \noindent{\bf Case 2: $\beta < 0$.} Roughly speaking, this case is similar to Case 1, except that some inequalities are reversed. We work out the details for completeness. If $\pi(v_3) >\max\{ \frac{\alpha}{2\beta} \cdot \pi(v_1 ) + \frac{\psi}{2\beta} \cdot \pi(v_2),0 \} + \ep$, then $\beta \cdot \pi(v_3) < \frac{\alpha}{2} \cdot \pi(v_1) + \frac{\psi}{2} \cdot \pi(v_2) + \beta\ep$, and so we have %
    \begin{align}
    Q_{\ag{w}}^{\pi}(w,0) - Q_{\ag{w}}^{\pi}(w,1) \geq & \gamma \cdot \left(\frac 12 \alpha \cdot \pi(v_1) + \frac 12 \psi \cdot \pi(v_2) - \beta \cdot \pi(v_3)\right) - 2 \gamma^2 \geq -\gamma \beta \cdot \ep  - 2 \gamma^2 > \ep'\nonumber,
  \end{align}
  which implies that, since $w$ is $\ep'$-unimprovable under $\pi$, then $\pi(w) = 0$. But then
  \begin{align}
Q_{\ag{v_3}}^{\pi}(v_3,0) - Q_{\ag{v_3}}^{\pi}(v_3,1) \geq -\gamma \beta - \gamma^2 > \ep',\nonumber
  \end{align}
  which implies that, since $v_3$ is $\ep'$-unimprovable under $\pi$, $\pi(v_3) = 0  < \ep$, a contradiction to $\pi(v_3) > \max \{ \frac{\alpha}{2\beta} + \frac{\psi}{2\beta} \cdot \pi(v_2), 0 \} + \ep$, which we assumed above.

  Next suppose that $\pi(v_3) < \min \{ \frac{\alpha}{2\beta} \cdot \pi(v_1) + \frac{\psi}{2\beta} \cdot \pi(v_2), 1 \} - \ep$, which implies that $\beta \cdot \pi(v_3) \geq \frac{\alpha}{2} \cdot \pi(v_1) + \frac{\psi}{2} \cdot \pi(v_2) - \beta \ep$. Then
    \begin{align}
    Q_{\ag{w}}^{\pi}(w,1) - Q_{\ag{w}}^{\pi}(w,0) \geq& \gamma \cdot \left(\beta \cdot \pi(v_3) - \frac 12 \alpha \cdot \pi(v_1) - \frac 12 \psi \cdot \pi(v_2)\right) -2\gamma^2  >  -\gamma \beta \cdot \ep - 2\gamma^2 > \ep'\nonumber,
  \end{align}
  which implies that, since $w$ is $\ep'$-unimprovable under $\pi$, $\pi(w) = 1$. But then
  \begin{align}
Q_{\ag{v_3}}^{\pi}(v_3,1) - Q_{\ag{v_3}}^{\pi}(v_3,0) \geq -\gamma \beta - \gamma^2 > \ep',\nonumber
  \end{align}
  which implies that, since $v_3$ is $\ep'$-unimprovable under $\pi$, $\pi(v_3) = 1 > 1-\ep$, a contradiction to $\pi(v_3) < \min \{ \frac{\alpha}{2\beta} \cdot \pi(v_1) + \frac{\psi}{2\beta} \cdot \pi(v_2), 1 \} - \ep$.

  In all possible cases, we have established that $\pi(v_3) \geq \min \{ \frac{\alpha}{2\beta} \cdot \pi(v_1) + \frac{\psi}{2\beta} \cdot \pi(v_2), 1 \} - \ep$ and $\pi(v_3) \leq \max \{ \frac{\alpha}{2\beta} \cdot \pi(v_1) + \frac{\psi}{2\beta} \cdot \pi(v_2), 0 \} + \ep$, which establishes the statement of the lemma.
\end{proof}

\paragraph{Implementing the $G_\gets$  gate.} 
Next we define a gadget which implements the  $G_\gets(b | | v)$ gate, for a constant $b \in \{0,1\}$.
\begin{defn}
  \label{def:g-gets}
For $b \in \{0,1\}$, we say that a $G_\gets(b | | v)$ gate \emph{embeds in a stochastic game $\BG$ via the state $v \in \MS$ and the constant $b$}, if the following holds:
  \begin{enumerate}
  \item The transitions out of $v$ satisfy $\BP(\sinkz | v, 0) = \BP(\sinkz | v,1) = 1$.
  \item The rewards to player $\ag{v}$ at the state $v$ satisfy $r_{\ag{v}}(v,1) = {b}$ and $r_{\ag{v}}(v,0) = {1-b}$. 
  \end{enumerate}
\end{defn}
\begin{lemma}
\label{lem:g-gets}
Suppose that $G_\gets(b || v)$ embeds in a stochastic game $\BG$ via the state $v$ and the constant $b \in \{0,1\}$, and consider a policy $\pi : \MS \ra [0,1]$. Then if the state $v$ is $\ep$-unimprovable under $\pi$ with $\ep < (1-\gamma)/2$, it holds that $\pi(v) = b$. 
\end{lemma}
\begin{proof}
  For $b \in \{0,1\}$, it is clear that $Q_{\ag{v}}^\pi(v,1) = (1-\gamma) \cdot b$ and $Q_{\ag{v}}^{\pi}(v,0) = (1-\gamma) \cdot (1-b)$. 
  Thus, since $v$ is $\ep$-unimprovable under $\pi$ and $\ep < (1-\gamma)/2$, we must have that $\pi(v) = b$.
\end{proof}

\paragraph{Implementing the $G_<$ gate. }
Finally, we define a gadget which implements the $G_<(|v_1, v_2 | v_3)$ gate, for states $v_1, v_2, v_3$ of the stochastic game $\BG$. 
\begin{defn}
  \label{def:g-lt}
Consider any $\beta \in \BR$ with absolute value at most $1-\gamma$.   We say that a $G_<(|v_1, v_2 | v_3)$ gate  \emph{embeds in a stochastic game $\BG$ via the states $(v_1, v_2, v_3, w)$ and the constant $\beta$}, for states $v_1, v_2, v_3, w \in \MS$ if the following holds:
  \begin{enumerate}
  \item The transitions out of $v_3$ and $w$ satisfy the following:
    \begin{itemize}
    \item $\BP(v_1 | w, 0) = 1$ and $\BP(v_2 | w, 1) = 1$;

    \item $\BP(w | v_3, 1) = 1$ and $\BP(\sinkz | v_3, 0) = 1$.
    \end{itemize}
  \item The reward of the players controlling $w, v_3$ at states $v_1, v_2, w$ satisfy the following:
    \begin{itemize}
  \item $r_{\ag{w}}(v_1, 1) = r_{\ag{w}}(v_2, 1) = \frac{\beta}{1-\gamma}$;
  \item $r_{\ag{w}}(v_1, 0)  = r_{\ag{w}}(v_2, 0) =  0$;
  \item For each $a \in \{0,1\}$, $r_{\ag{w}}(w,a) = r_{\ag{v_3}}(v_3, a) = 0$;
  \item $r_{\ag{v_3}}(w,1) = \frac{\beta}{1-\gamma}$ and $r_{\ag{v_3}}(w,0) = -\frac{\beta}{1-\gamma}$.
  \end{itemize}
  \end{enumerate}
\end{defn}
In the context of Definition \ref{def:g-lt}, we will at times refer to $w$ as the \emph{helper node} for the gate.

\begin{lemma}
  \label{lem:g-lt}
  Suppose that $G_<( | v_1, v_2 | v_3)$ embeds in a stochastic game $\BG$ via the states $(v_1, v_2, v_3, w)$ and the constant $\beta$, and consider any $\ep, \ep' \in (0,1)$. Suppose that $\gamma | \beta | \ep -2 \gamma^2 > \ep'$. Then for any policy $\pi : \MS \ra [0,1]$ so that $v_3, w$ are $\ep'$-unimprovable under $\pi$, it holds that
  \begin{align}
    \pi(v_3) = \begin{cases}
      1 \pm \ep  &: \pi(v_1) \leq \pi(v_2) - \ep \\
      0 \pm \ep &: \pi(v_1) \geq \pi(v_2) + \ep.
    \end{cases}\nonumber
  \end{align}
\end{lemma}
\begin{proof}
  Consider any policy $\pi : \MS \ra [0,1]$. We first compute $Q_{\ag{w}}^\pi(w,b)$ for $b \in \{0,1\}$: %
  \begin{itemize}
  \item $Q_{\ag{w}}^{\pi}(w,1) = \gamma \cdot V_{\ag{w}}^{\pi}(v_2) = \gamma \cdot (\beta \cdot \pi(v_2) \pm \gamma) = \gamma \cdot \beta \cdot \pi(v_2) \pm \gamma^2$;
  \item $Q_{\ag{w}}^{\pi}(w,0) = \gamma \cdot V_{\ag{w}}^{\pi}(v_1) = \gamma \cdot (\beta \cdot \pi(v_1) \pm \gamma) = \gamma \cdot \beta \cdot \pi(v_1) \pm \gamma^2$.
  \end{itemize}

  We next compute $Q_{\ag{v_3}}^{\pi}(v_3,b)$ for $b \in \{0,1\}$ in the particular case where  $\pi(w) \in \{0,1\}$:
  \begin{itemize}
  \item If $\pi(w) = 1$, then:
    \begin{itemize}
    \item $Q_{\ag{v_3}}^{\pi}(v_3,1) = \gamma\beta  \pm \gamma^2$. 
    \item $Q_{\ag{v_3}}^{\pi}(v_3,0) = 0$.
    \end{itemize}
  \item If $\pi(w) = 0$, then
    \begin{itemize}
    \item $Q_{\ag{v_3}}^{\pi}(v_3,1) = -\gamma\beta  \pm \gamma^2$. 
    \item $Q_{\ag{v_3}}^{\pi}(v_3,0) = 0$.      
    \end{itemize}
  \end{itemize}

Notice that $\gamma | \beta | \ep -2 \gamma^2 > \ep'$ implies $\beta\neq 0$. Suppose that $\pi(v_1) \leq \pi(v_2) - \ep$ and that $\beta > 0$. We have
  \begin{align}
Q_{\ag{w}}^{\pi}(w,1) - Q_{\ag{w}}^{\pi}(w,0) \geq \gamma \beta \cdot (\pi(v_2) - \pi(v_1)) - 2\gamma^2 \geq \gamma \beta \ep - 2 \gamma^2 > \ep',\nonumber
  \end{align}
  which implies that, since $w$ is $\ep'$-unimprovable under $\pi$, we must have that $\pi(w) = 1$. Then we have
  \begin{align}
Q_{\ag{v_3}}^{\pi}(v_3,1) - V_{\ag{v_3}}^{\pi}(v_3,0) \geq \gamma\beta - \gamma^2 > \ep',\nonumber
  \end{align}
  meaning that $\pi(v_3) = 1$ since $v_3$ is $\ep'$-unimprovable under $\pi$, which is what we wanted to show in this case. In a similar manner, if $\pi(v_1) \leq \pi(v_2) - \ep$ but $\beta < 0$, then we see that $\pi(w) = 0$ and since $-\gamma \beta - \gamma^2 > \ep'$, we again get that $\pi(v_3) = 1$.

 Similarly, if $\pi(v_1) \geq \pi(v_2) + \ep$, then we have $\pi(w) = 0$ if $\beta > 0$ and $\pi(w) = 1$ if $\beta < 0$. In the case that $\beta > 0$, we get that $\pi(v_3) = 0$, and in the case that $\beta < 0$, we also get that $\pi(v_3) = 0$, as desired. 
\end{proof}

\subsection{Gluing gadgets via valid colorings}
\label{sec:valid-colorings}

Next, we discuss how to combine the gadgets introduced in the previous section into a 2-player turn-based stochastic game such that an approximate Nash equilibrium yields an approximate assignment to a given instance of the generalized circuit problem. 

\paragraph{Thought experiment \& challenges.} 
If we were willing to allow each player to control a different node (so that the number of players would be polynomial in the input length), then this procedure would be quite straightforward. Since we aim to show hardness for \emph{2-player} games,  however, we have to be more careful, since some of the constraints induced by the embedding of gates in Definitions \ref{def:g-times-plus}-\ref{def:g-lt} may conflict with each other when multiple states are controlled by a single player.

For instance, Definition \ref{def:g-times-plus} requires that for the embedded gate $G = G_{\times,+}(\frac{\alpha}{2\beta}, \frac{\psi}{2\beta} | v_1, v_2 | v_3)$ with helper node $w$, we must have $r_{\ag{w}}(v_1,1) = \frac{\alpha \cdot \max\left\{ 1, \frac{|\beta|}{|\alpha|}\right\}}{1-\gamma}$. Now suppose we were to attempt to embed gate $G'=G_{\times,+}(\frac{\alpha'}{2\beta'}, \frac{\psi'}{2\beta'} | v_1', v_2' | v_3')$ with some helper node $w'$. Suppose further that the output node $v_3'$ of $G'$ equals $v_1$ (which corresponds to the output of $G'$ feeding into the gate $G$). Then the constraints of Definition \ref{def:g-times-plus} for this gate would require that $r_{\ag{w'}}(v_3',1) = r_{\ag{w'}}(v_1,1) = \frac{\beta'}{1-\gamma}$. It is possible that $\frac{\beta'}{1-\gamma} \neq \frac{\alpha \cdot \max\left\{ 1, \frac{|\beta|}{|\alpha|}\right\}}{1-\gamma}$, which implies that $\ag{w} \neq \ag{w'}$. It is a straightforward consequence of Definition \ref{def:g-times-plus} that we must also have that $\ag{w'} \neq \ag{v_3'} = \ag{v_1}$. Similar constraints may arise involving $\ag{v_2}$ and $\ag{v_3}$, and it is evident that the task of assigning a controller to each node becomes quite nontrivial. If we assign all non-helper nodes (denoted by $v, v_1, v_2, v_3$ in Definitions \ref{def:g-times-plus}, \ref{def:g-gets}, \ref{def:g-lt}) to a 
single player, it is possible to show that, assuming the given generalized circuit instance has fan-out 2 (which is without loss of generality by \cite{rubinstein2018inapproximability}), by greedily assigning each of the helper nodes to one of 4 players (for a total of 5 players), we may satisfy all constraints of the embedded gadgets.

To obtain hardness for 2-player (as opposed to 5-player) games, we instead take a different approach. As discussed above, we will assign all non-helper states to a single player, which we call $\VP$, and all helper states (denoted by $w$ in Definitions \ref{def:g-times-plus} and \ref{def:g-lt}) to a second player, which we call $\WP$. As mentioned above, this will cause conflicts; however, it is straightforward to check that the only type of conflict that arises is that for states $w,w'$ controlled by $\WP$, there is some state $v$ controlled by player $\VP$ so that one embedded gate requires $r_{\ag{w}}(v,1) = c$ and $r_{\ag{w'}}(v,1) = c'$ for some $c \neq c'$. To avoid this type of conflict, we will show how to convert a given generalized circuit instance into an equivalent instance for which such conflicts cannot arise. In particular, we introduce a notion of \emph{valid coloring} of the nodes of a generalized circuit, so that a circuit equipped with a valid coloring has the property that no conflicts of the above type can arise when embedding the circuit into a 2-player turn-based stochastic game.

\paragraph{Valid colorings.} Given a generalized circuit $\MC=  (V, \MG)$, we will say that the  assignment of a real number to each node, denoted by $\phi : V \ra \BR$, is a \emph{coloring} of $V$. Below we define the notion of \emph{valid coloring}, which requires, loosely speaking, that the colorings of nodes are consistent with the rewards given to the $\ag{w}$ player in each of the gate gadgets defined in the previous section: %
\begin{defn}
  \label{def:validity}
  Given a coloring $\phi : V \ra \BR$, we say that $\phi$ is \emph{valid} if the following holds:
  \begin{enumerate}
  \item For each gate $G_<( | v_1, v_2 | v_3)$, it holds that $\phi(v_1) = \phi(v_2)$;
  \item For each gate $G_{\times, +}(\xi, \zeta | v_1, v_2 | v_3)$, it holds that
\begin{align}
  \frac{\phi(v_1)}{\phi(v_3)} = \begin{cases}
    2 \cdot \xi &: |\xi| \geq 1/2 \\
    \sign(\xi) &: |\xi| < 1/2 
  \end{cases},
             \mbox{ and }    \frac{\phi(v_2)}{\phi(v_3)} = \begin{cases}
    2 \cdot \zeta &: |\zeta| \geq 1/2 \\
    \sign(\zeta) &: |\zeta| < 1/2 .
  \end{cases}\nonumber
\end{align}
  \end{enumerate}
  We say that a gate $G$ is \emph{valid} if, in the case that it is one of the above types of gates, the respective condition above is met. (If $G$ is not one of the above types of gates, i.e., the $G_\gets$ gate, then it is automatically defined to be valid.)
\end{defn}

We further define the \emph{range} of $\phi$ to be the set $\{ \phi(v) \ : \ v \in V \}\subset \BR$. The below lemma shows, loosely speaking, how to convert a circuit with some coloring $\phi$ to an equivalent circuit with a valid coloring.  For simplicity, we use the following terminology: given a generalized circuit $\MC = (V, \MG)$, we say that an assignment $\pi : V \ra [0,1]$ is an \emph{$(\ep, \delta)$-assignment} of $\MC$ if at least a $1-\delta$ fraction of the gates are $\ep$-approximately satisfied by $\pi$ (see Definition \ref{def:gcircuit}). We say that $\pi$ is an \emph{$\ep$-assignment} if it is an $(\ep, 0)$-assignment.

\begin{lemma}
  \label{lem:construct-valid-coloring}
  There is an absolute constant $C_0 > 0$ so that the following holds. Let $\MC = (V, \MG)$ be a generalized circuit and $\ep > 0$. %
  Then one can construct, in polynomial time, a circuit $\MC' = (V', \MG')$ together with a \emph{valid} coloring $\phi : V' \ra \BR$, so that:
  \begin{enumerate}
  \item $V \subset V'$; %
  \item The range of $\phi$ is contained in $[1/4,1/2] \cup [-1/2, -1/4]$;
  \item For any $\delta \geq 0$, given an $(\ep,\delta)$-approximate assignment $\pi : V' \ra [0,1]$ of $\MC'$, the restriction of $\pi$ to $V$ constitutes a $(133\sqrt{\ep}, C_0 \delta / \sqrt \ep)$-approximate assignment of $\MC$.
  \end{enumerate}
\end{lemma}
We first sketch the proof of Lemma \ref{lem:construct-valid-coloring}.  For an appropriate choice of $V'$ which contains $V$ as a subset, we will define $\phi : V' \ra \BR$ so that $\phi(v) = 1/4$ for all $v \in V$. Notice that we would be able to choose $V' = V, \MG' = \MG$ and would immediately have a valid coloring if it were not for gates of the type $G_{\times, +}(\xi, \zeta | v_1, v_2 | v_3)$, which can require that different nodes have different colors under $\phi$ (i.e., when $\xi \neq 1/2$ or $\zeta \neq 1/2$).

  To circumvent this obstacle, for each gate of the form $G_{\times, +}(\xi, \zeta | v_1, v_2 | v_3)$, we introduce a sequence of gates and nodes connecting each of $v_1$ (respectively, $v_2$) to some (new) node $v_1'$ (respectively, $v_2'$), which approximately implements the identity map. Importantly, this sequence of gates and nodes will have the property that there is a cut (i.e., a separating set) consisting solely of gates of the type $G_<(|u_1, u_2 | u_3)$, which place no restriction on $\phi(u_3)$ for a valid coloring $\phi$. We will be able to use this cut to ensure that the value of $\phi$ at  vertices on one side of the cut differs from the value of $\phi$ at vertices on the other side of the cut, thus ensuring that $\phi(v_1')$ (respectively, $\phi(v_2')$) can differ from $\phi(v_1)$ (respectively, $\phi(v_2)$), while maintaining validity. 
\begin{proof}[Proof of Lemma \ref{lem:construct-valid-coloring}]
  We follow the outline sketched above. In particular, we build up the circuit $(V', \MG')$ according to the following procedure. We initialize $V' = V$ and set $\phi(v) = 1/4$ for all $v \in V$. Moreover, for all gates apart from those of the type $G_{\times, +}$, we add the same gate (with the same input and output nodes) to $\MG'$. It is immediate that any such gate satisfies the validity constraint in Definition \ref{def:validity} under the coloring $\phi$ (if applicable). Now define the function  $f : \BR \ra \BR$ by $$
  f(x) = \begin{cases} 2x &: |x| \geq 1/2 \\ \sign(x) &: |x| < 1/2 \end{cases}.
  $$ 
 
  Consider each gate of the form $G_{\times, +}(\xi, \zeta | v_1, v_2 | v_3)$ in turn. For each such gate, we perform the following steps: We add nodes $v_1', v_2'$ to $V'$, and add the gate $G_{\times, +}(\xi, \zeta | v_1', v_2' | v_3)$ to $\MG'$. Furthermore, we set $\phi(v_1') := f(\xi) \cdot \phi(v_3) = f(\xi) \cdot \frac{1}{4}$ and $\phi(v_2') := f(\zeta) \cdot \phi(v_3) = f(\zeta) \cdot \frac{1}{4}$, thus ensuring that the validity constraint for $G_{\times, +}(\xi, \zeta | v_1', v_2' | v_3)$ is satisfied. Furthermore, since $|f(x)| \geq 1$ for all $x \in \BR$, we have that $|\phi(v_1')| \geq 1/4$ and $|\phi(v_2')| \geq 1/4$. Moreover, since $|\zeta| \leq 1$ and $|\xi| \leq 1$ (by Definition \ref{def:gcircuit}), it holds that $|\phi(v_1')| \leq 1/2$ and $|\phi(v_2')| \leq 1/2$.

  In Claim \ref{clm:real-unary} below, we add a sequence of gates and nodes to $\MG', V'$, respectively (together with respective colors ensuring validity), that lie between $v_1$ and $v_1'$, which ensure that in any $\ep$-approximate assignment $\pi$, $|\pi(v_1) - \pi(v_1')| \leq O(\sqrt{\ep})$. By symmetry,  the same construction can be implemented for the nodes $v_2, v_2'$; as we discuss following the proof of Claim \ref{clm:real-unary} the result of Lemma \ref{lem:construct-valid-coloring} will follow in a straightforward manner.
  \begin{claim}
    \label{clm:real-unary}
    Consider two nodes $a, a' \in V'$ so that $\phi(a), \phi(a')$ are defined. Then it is possible to add a set $\til V$ of $O(1/\sqrt \ep)$ nodes to $V'$ and a set $\til \MG$ of $O(1/\sqrt \ep)$ gates to $\MG'$ so that the gates in $\til \MG$ have all their input and output nodes in $\til V \cup \{a,a'\}$, and the following holds:
    \begin{enumerate}
    \item $\phi$ may be extended to a mapping on $\til V$ so that for all $v \in \til V$, $\phi(v) \in \{ \phi(a), \phi(a') \}$. Furthermore, the resulting $\phi$ is so that all gates in $\til \MG$ are valid. \label{it:range-phi}
    \item Given any assignment $\pi$ of the circuit (including that of the nodes in $\til V$), if all gates in $\til \MG$ are satisfied under $\pi$, then $|\pi(a) - \pi(a')| \leq 66 \cdot \sqrt \ep$.\label{it:pia-pia}
    \end{enumerate}
  \end{claim}
  \begin{proof}
    The construction we introduce mirrors that of Algorithms 6 and 7 of \cite{rubinstein2018inapproximability},  with a few differences.   Choose $\ep' \geq \sqrt{\ep}$ as small as possible so that $4/\ep'$ is a power of 2 (so that $\ep' \leq 2 \sqrt{\ep}$). Initialize $\til V, \til \MG$ to be empty sets. We now introduce the following nodes and gates, which are added to $\til V$ and $\til \MG$, respectively:
    \begin{enumerate}
    \item Add a gate $G_{\gets}(1 || \sigma)$ to $\til \MG$, whose output node $\sigma$ is added to $\til V$. Define $\phi(\sigma) := \phi(a)$.
    \item For each $k \in [4/\ep']$:
      \begin{enumerate}
      \item Add a gate $G_{\times, +}(k\ep'/8, k\ep'/8 | \sigma, \sigma | \sigma_k)$ to $\til \MG$, whose output node $\sigma_k$ is added to $\til V$. Define $\phi(\sigma_k) := \phi(\sigma)$. \emph{(Validity of this gate is ensured since $0 < k\ep'/8 \leq 1/2$ and we have $\phi(\sigma_k) = \phi(\sigma)$.)} 
      \item Add a gate $G_<(|\sigma_k, a | b_k)$ to $\til \MG$, whose output node $b_k$ is added to $\til V$.  Define $\phi(b_k) := \phi(a')$.  \emph{(Validity of this gate is ensured since $\phi(\sigma_k) = \phi(\sigma) = \phi(a)$; importantly, we are allowed to set $\phi(b_k)$ to something which does \emph{not} equal $\phi(a)$.)}
      \end{enumerate}
    \item For each $j \in [\log_2(4/\ep')]$:
      \begin{enumerate}
      \item For each $k \in [(4/\ep')/2^j]$:
        \begin{enumerate}
        \item If $j =1$, add a gate $G_{\times, +}(1/2,1/2 | b_{2k-1}, b_{2k} | d_{1,k})$ to $\til \MG$, whose output node $d_{1,k}$ is added to $\til V$. Define $\phi(d_{1,k}) := \phi(a')$. \emph{(Validity of this gate is ensured since $\phi(b_{2k-1}) = \phi(b_{2k}) = \phi(d_{1,k}) = \phi(a')$.)}
        \item If $j > 1$, add a gate $G_{\times,+}(1/2,1/2 | d_{j-1, 2k-1}, d_{j-1,2k} | d_{j,k})$ to $\til \MG$, whose output node $d_{j,k}$ is added to $\til V$. Define $\phi(d_{j,k}) := \phi(a')$. \emph{(Validity of this gate is ensured since $\phi(d_{j-1,2k-1}) = \phi(d_{j-1,2k}) = \phi(d_{j,k}) = \phi(a')$.)}
        \end{enumerate}
      \end{enumerate}
    \item Add a gate $G_{\times, +}(1/2,1/2 | d_{\log_2(4/\ep'), 1}, d_{\log_2(4/\ep'), 1} | a')$ to $\til \MG$. \emph{(Validity of this gate is ensured since $\phi(d_{\log_2(4/\ep'),1}) = \phi(a')$.)}
    \end{enumerate}
    It is straightforward to see that $|\til V|, |\til \MG|$ are bounded above by $O(1/\ep') = O(1/\sqrt \ep)$ at the end of the above procedure.

    It is clear that at all nodes $v$ added to $\til V$ in the above construction, we have $\phi(v) \in \{\phi(a), \phi(a') \}$, thus verifying the first item in the claim's statement. To see the second item, let $\pi$ denote an assignment for the generalized circuit $(V', \MG')$, after $\til V$ has been added to $V'$ and $\til \MG$ has been added to $\MG'$, and suppose that $\pi$ $\ep$-approximately satisfies all gates in $\til \MG$. By definition of the gate $G_{\times, +}$, we must have that for each $k \in [4/\ep']$, $\pi(\sigma_k) = k\ep'/4 \pm \ep$. Thus, the number of integers $k \in [4/\ep']$ so that $\pi(b_k) \geq 1-\ep$ lies in the range $\left[ \frac{4}{\ep'} \cdot (\pi(a) - 3\ep), \frac{4}{\ep'} \cdot (\pi(a) + 3\ep) \right]$, and the number of integers $k \in [4/\ep']$ so that $\pi(b_k) \leq \ep$ lies in the range $\left[ \frac{4}{\ep'} \cdot (1 - \pi(a) - 3\ep), \frac{4}{\ep'} \cdot (1 - \pi(a) + 3\ep) \right]$.

    It follows that
    \begin{align}
\sum_{k=1}^{4/\ep'} \pi(b_k) = \frac{4}{\ep'} \cdot \pi(a) \pm \left( \ep \cdot \frac{4}{\ep'} + \frac{4}{\ep'} \cdot 6\ep \right) = \frac{4}{\ep'} \cdot \pi(a) \pm 28 \ep'.\nonumber
    \end{align}
    By the definition of the gate $G_{\times, +}$ and the triangle inequality, it holds that $\pi(d_{\log_2(4/\ep'), 1}) = \frac{\ep'}{4} \sum_{k=1}^{4/\ep'} \pi(b_k) \pm \log_2(4/\ep') \cdot \ep = \frac{\ep'}{4} \sum_{k=1}^{4/\ep'} \pi(b_k) \pm 4\ep'$, since $\log_2(4/\ep') \leq 4/\ep'$. Since also $\pi(a') = \pi(d_{\log_2(4/\ep'), 1}) \pm \ep$, we conclude that
    \begin{align}
\pi(a') = \frac{\ep'}{4} \sum_{k=1}^{4/\ep'} \pi(b_k) \pm (4\ep' + \ep) = \pi(a) \pm (4\ep' + \ep + 28 (\ep')^2) = \pi(a) \pm 33 \ep' = \pi(a) \pm 66 \sqrt{\ep}\nonumber,
    \end{align}
    where the last step uses that $\ep \leq \ep' \leq 2 \sqrt \ep$. 
    \end{proof}
    Given Claim \ref{clm:real-unary}, we complete the proof of Lemma \ref{lem:construct-valid-coloring}. For the gate $G_{\times, +}(\xi, \zeta | v_1, v_2 | v_3)$ (as was introduced above), we apply Claim \ref{clm:real-unary} once with $a = v_1, a' = v_1'$, adding sets $\til V_1, \til \MG_1$ to $V', \MG'$, respectively, and once with $a = v_2, a' = v_2'$, adding sets $\til V_2, \til \MG_2$ to $V', \MG'$, respectively. After this procedure, it still holds that for all $v \in V'$, $|\phi(v)| \in [1/4,1/2]$ by item \ref{it:range-phi} of Claim \ref{clm:real-unary}. Furthermore, item \ref{it:pia-pia} of Claim \ref{clm:real-unary} gives that in any assignment $\pi : V' \ra [0,1]$ for which all gates in $\til \MG_1 \cup \til \MG_2 \cup \{ G_{\times, +}(\xi, \zeta | v_1', v_2' | v_3)\}$ are $\ep$-approximately satisfied,
    \begin{align}
      \pi(v_3) =& \max \left\{ \min \left \{ \xi \cdot \pi(v_1') + \zeta \cdot \pi(v_2'), 1 \right\}, 0 \right\} \pm \ep \nonumber\\
      =& \max \left\{ \min \left \{ \xi \cdot (\pi(v_1) \pm 66 \sqrt{\ep}) + \zeta \cdot (\pi(v_2) \pm 66 \sqrt{\ep}), 1 \right\}, 0 \right\} \pm \ep\nonumber\\
      =& \max \left\{ \min \left \{ \xi \cdot \pi(v_1) + \zeta \cdot \pi(v_2), 1 \right\}, 0 \right\} \pm 133 \sqrt \ep\nonumber,
    \end{align}
    where the final step uses that $|\xi|, |\zeta| \leq 1$. Note that $|\til \MG_1 \cup \til \MG_2| \leq O(1/\sqrt \ep)$ by Claim \ref{clm:real-unary}. Thus, after applying Claim \ref{clm:real-unary} for each gate $G_{\times, +}$ in the original circuit $\MC$, we note that for an $(\ep, \delta)$-approximate assignment $\pi : V' \ra [0,1]$ of $\MC'$, it must hold that, for some constant $C > 1$, for at least a fraction $1-C \delta / \sqrt \ep$ fraction of the gates $G$ of the original circuit $\MC$, the gate $G$ is $133\sqrt \ep$-satisfied. (This holds because at least a $1-C\delta/\sqrt \ep$ fraction of the gates in $\MC$ are either not of the type $G_{\times, +}$ or have all of the $O(1/\sqrt \ep)$ supplementary gates added in course of Claim \ref{clm:real-unary} $\ep$-satisfied by $\pi$.) This completes the proof of the lemma.
\end{proof}

Next we show that the problem of finding an approximate assignment of a circuit which has a valid coloring can be reduced to the problem of finding an approximate (P)WSNE-SG of an infinite-horizon  discounted stochastic game.

\begin{lemma}
  \label{lem:valid-circuit-games}
  Fix any $\ep \in (0,\frac{1}{12})$ and $\delta \in (0,1)$. 
  Set $\gamma =\ep^2$, $\ep' = \ep^4$, and $\ep'' = \ep' \cdot \delta$. Then the following statements hold. %
  \begin{itemize}
  \item The problem of finding an $\ep$-assignment to a generalized circuit instance equipped with a valid coloring with range contained in $[-1/2,-1/4] \cup [1/4,1/2]$ %
    has a polynomial-time reduction to the problem of computing an  $\ep'$-PWSNE-SG in 2-player  turn-based $\gamma$-discounted stochastic games.    %
  \item The problem of finding an $(\ep,3\delta)$-assignment to a generalized circuit instance equipped with a valid coloring with range contained in $[-1/2, -1/4] \cup [1/4,1/2]$ %
    has a polynomial-time  reduction to the problem of computing a $\ep''$-WSNE-SG in 2-player turn-based $\gamma$-discounted stochastic games.    %
  \end{itemize}
\end{lemma}
\begin{proof}
  Let $\MC = (V, \MG)$ be a generalized circuit together with some valid coloring $\phi : V \ra \BR$, and $\ep \in (0,1)$. We construct a $\gamma$-discounted 2-player turn-based stochastic game $\BG$, as follows:   the two players are denoted  by $\WP$ and $\VP$, the action spaces of each player of $\BG$ satisfiy $\MA_\VP = \MA_\WP = \{0,1\}$, and the state space $\MS$ satisfies:
  \begin{align}
\MS = V \sqcup W \sqcup \sinkz,\nonumber
  \end{align}
  where $\sinkz$ is a special sink state which transitions to itself indefinitely and at which all players receive 0 reward, and $W$ is in bijection with $\MG$, consisting of a designated node $w_G$ for each gate $G \in \MG$. The  ownership of the states is as follows: for all $v \in V$, we have $\ag{v} = \VP$, for all $w \in W$, we have $\ag{w} = \WP$. Finally, we arbitrarily set $\ag{\sinkz} = \VP$. 

  For each gate of the form $G(\ell | v_1, v_2 | v_3)$, we will ensure that $G$ embeds in $\BG$  via the tuple $(v_1, v_2, v_3, w_G)$ or the tuple $(v_1, v_2, v_3)$ (depending on the type of gate $G$), and via an appropriate vector of constants (if applicable, again depending on the type of gate $G$). %
  To do so, we construct the transitions and rewards of $\BG$ as follows: intially set the reward at each state (for all agents and actions) to be 0, and define the transitions so that each state transitions to $\sinkz$ under any action. We will then make several modifications to the rewards and transitions: First, for each state $v \in V$, define
  \begin{align}
    \label{eq:phi-rewards}
    r_{\WP}(v,1) = \frac{\phi(v)}{1-\gamma}.
  \end{align}
  Next, for each state $v \in V$ which is the outgoing node of a gate of the form $G_\gets(b || v)$, set
  \begin{align}
r_{\VP}(v,b) = b, \qquad r_{\VP}(v,0) = 1-b.\nonumber
  \end{align}

  Next, for each gate $G \in \MG$, we make the following modifications to $\BG$'s transitions and rewards, depending on the type of gate $G$.   
  \begin{enumerate}
  \item If $G$ is of the form $G_{\times, +}(\xi, \zeta | v_1, v_2| v_3)$, then define
    \begin{align}
      \alpha = \begin{cases}
        \phi(v_1) &: |\xi | \geq 1/2 \\
        \phi(v_1) \cdot 2 |\xi| &: |\xi| < 1/2
      \end{cases}, \qquad \psi = \begin{cases}
        \phi(v_2) &: |\zeta| \geq 1/2 \\
        \phi(v_2) \cdot 2 |\zeta| &: |\zeta| < 1/2,
      \end{cases}, \qquad \beta = \phi(v_3),\nonumber
    \end{align}
    and modify the outgoing transitions from the states $v_3, w_G$ and the rewards at state $v_3$ to satisfy the requirements of Definition \ref{def:g-times-plus} with the above values of $\alpha, \psi, \beta$. %
    The above definitions and the validity of $\phi$ ensure that $\frac{\alpha}{2\beta} = \xi$ and $\frac{\psi}{2\beta} = \zeta$,  as in Definition \ref{def:g-times-plus} (which ensures, when applying Lemma \ref{lem:g-times-plus}, that the gate $G$ implements the right transformation, the condition for which will be verified shortly).  
Furthermore, we do not have to modify $r_\WP(v_1, 1)$, $r_\WP(v_2, 1)$, or $r_\WP(v_3,1)$ since they are set to $\frac{\phi(v_1)}{1-\gamma}$, $\frac{\phi(v_2)}{1-\gamma},$ and $\frac{\phi(v_3)}{1-\gamma}$, respectively, in (\ref{eq:phi-rewards}), and it holds by validity of $\phi$ and the definitions of $\alpha, \beta, \psi$ above that %
    \begin{align}
\alpha \cdot \max \left\{ 1, \frac{|\beta|}{|\alpha|} \right\} = \phi(v_1), \quad \psi \cdot \max \left\{ 1, \frac{|\beta|}{|\psi|} \right\} = \phi(v_2), \quad \beta = \phi(v_3).
    \end{align}
  \item If $G$ is of the form $G_\gets(1 || v)$, then modify the outgoing transitions from the state $v$ to satisfy the requirements of Definition \ref{def:g-gets}.
  \item If $G$ is of the form  $G_<(|v_1, v_2 | v_3)$, %
    then modify the outgoing transitions from the states $v_3, w_G$ and the rewards at state $w_G$ to satisfy the requirements of  Definition \ref{def:g-lt}, %
    with $\beta = \phi(v_1)$. (We do not have to modify $r_{\WP}(v_1,1)$ or $r_\WP(v_2,1)$ since both are set to $\frac{\beta}{1-\gamma} = \frac{\phi(v_1)}{1-\gamma} = \frac{\phi(v_2)}{1-\gamma}$, by validity of $\phi$  and (\ref{eq:phi-rewards}).)  
  \end{enumerate}
Note that each gate  $G$, whose output node is denoted by $v$, in the above procedure we have modified only the outgoing transitions at $v$ and $w_G$, and the two players' rewards at state $w_G$. Since each node is the output node of a unique gate, this process ensures that neither the transitions nor reward at any state are modified twice in the below procedure, thus ensuring that the embedding requirements for each gate are still satisfied at the end of the above procedure.

  Since $\max_{v \in V} |\phi(v)| \leq 1/2$, all nonzero rewards assigned to state-action pairs in the above procedure are bounded in magnitude by $\frac{\max_{v \in V} |\phi(v)|}{1-\gamma}$, and $\gamma < 1/2$, it holds that all rewards of $\BG$ have absolute value at most $1$.

Now we verify the condition in Lemma \ref{lem:g-times-plus}. Let us write $\beta_0 := \min_{v \in V} |\phi(v)| \geq 1/4$.  We claim that $\gamma \cdot |\beta_0| \cdot \ep - 2 \gamma^2 > \ep'$.   This holds since $\gamma \ep / 4 - 2\gamma^2 > \ep'$, which is guaranteed by our choice of $\gamma = \ep^2$ and $\ep' = \ep^4$ and since $\ep^3/4 - 2 \ep^4 > \ep^4$, which holds as long as $\ep < 1/12$, which was assumed in the lemma statement. 
  
Consider a policy $\pi$ of $\BG$, is represented by a function $\pi : \MS \ra [0,1]$. If $\pi$ is an $\ep'$-PWSNE-SG of $\BG$, then by Definition \ref{def:wsne-sg} and (\ref{eq:ep-satisfied}), all states $s$ of $\BG$ are $\ep'$-unimprovable under $\pi$.  By Lemmas \ref{lem:g-times-plus}, \ref{lem:g-gets}, and \ref{lem:g-lt}, %
since $\gamma \cdot |\beta_0 | \cdot \ep - 2 \gamma^2 > \ep'$, in any $\ep'$-PWSNE-SG $\pi$ of $\BG$, it holds that each gate is $\ep$-approximately satisfied by the restriction of $\pi$ to $V$. Thus, the restriction of an $\ep'$-PWSNE-SG $\pi$ to the nodes $V \subset \MS$ furnishes an $\ep$-approximate assignment to the $\ep$-\Gcircuit instance $\MC$, as desired.

Next, suppose that $\pi$ is an $\ep''$-WSNE-SG of $\BG$. Then by Definition \ref{def:wsne-sg}, (\ref{eq:ep-satisfied}), and Markov's inequality, a fraction $1-\delta$ of states of $\BG$ are $\ep'' /\delta = \ep'$-unimprovable. %
Since each node $v \in V$ is the output node of some (unique) gate in $\MG$, a fraction $1-3\delta$ of gates $G$ in $\MG$ have the following property:\kz{can we double-check and explain a bit more? a bit uncertain}\noah{basically the worst case is that each state which is not $\ep'$-unimprovable belongs to a different gate, and each gate has 2 nodes that we want to be $\ep'$-unimprovable. This would lead to a fraction $1-2\delta$, but due to the sink state (which messes things up slightly), I wrote $1-3\delta$. If time I will write something more formally in the text.}
\begin{itemize}
\item If $G = G_{\times, +}(\xi, \zeta | v_1, v_2 | v_3)$, then $v_3, w_G$ are $\ep'$-unimprovable.
\item If $G = G_\gets(1||v)$, then $v$ is $\ep'$-unimprovable.
\item If $G = G_<(|v_1, v_2 | v_3)$, then $v_3, w_G$ are $\ep'$-unimprovable.
\end{itemize}
By Lemmas \ref{lem:g-times-plus}, \ref{lem:g-gets}, and \ref{lem:g-lt}, it follows that a fraction $1-3\delta$ of the gates of $\MG$ are $\ep$-approximately satisfied by the restriction of $\pi$ to $V$. 
\end{proof}

Finally, we are ready to prove Theorems \ref{thm:perfect-hardness} and \ref{thm:non-perfect-hardness}.
\begin{proof}[Proof of Theorem \ref{thm:perfect-hardness}]
  By Theorem \ref{thm:gcircuit-ppad}, it suffices to show that there is a constant $c > 0$ so that for all $\ep_0 < 1/12$, the $\ep_0$-\Gcircuit problem has a polynomial-time reduction to the problem of computing $c \cdot \ep_0^{16}$-perfect NE in $1/2$-discounted 2-player stochastic games.

  Fix any $\ep \in (0,1/12)$, and write $\gamma = \ep^2$. 
  Consider an instance $\MC = (V, \MG)$ of $\ep$-\Gcircuit. We construct a circuit $\MC' = (V', \MG')$, together with a valid coloring $\phi : V' \ra \BR$, as guaranteed in the statement of Lemma \ref{lem:construct-valid-coloring} (where we take $\delta = 0$). By Lemma \ref{lem:valid-circuit-games}, we may further construct in polynomial time, given $\MC'$ together with $\phi$, a  $\gamma$-discounted 2-player turn-based stochastic game $\BG$ so that given an $\ep^4$-PWSNE-SG of $\BG$, we may compute an $\ep$-approximate assignment to $\MC'$, which, by Lemma \ref{lem:construct-valid-coloring}, yields a $133\sqrt{\ep}$-approximate assignment of $\MC$.

  By Lemma \ref{lem:reduce-to-ws}, the problem of computing an $\ep^4$-PWSNE-SG of $\BG$ reduces to the problem of computing a $\frac{\ep^8}{144}$-perfect NE of $\BG$. Finally, noting that the game $\BG$ is $\gamma$-discounted and we wish to reduce to the problem of computing equilibria in $1/2$-discounted games, we argue as follows: given the game $\BG$, we may construct a $1/2$-discounted stochastic game $\BG'$ whose states, actions, and rewards are identical to that of $\BG$, and whose transitions $\BP'(\cdot | s,a)$ are determined from the transitions $\BP(\cdot | s,a)$ of $\BG$ as follows:
  \begin{align}
    \BP'(s' | s,a) = \begin{cases}
      \frac{\gamma}{1/2} \cdot \BP(s' | s,a)  &: s' \neq \sinkz \\
      \frac{\gamma}{1/2} \cdot \BP(s' | s,a) + (1 - 2\gamma) &: s' = \sinkz.
    \end{cases}\nonumber
  \end{align}
Let $V_i^{\BG, \pi}$ denote the value function of $\BG$ and $V_i^{\BG',\pi}$ denote the value function of $\BG'$. It is clear that for all $\pi$ and all $i \in [m]$, $V_i^{\BG, \pi}(\sinkz) = V_i^{\BG', \pi}(\sinkz) = 0$. It is now straightforward to see that for any joint stationary policy $\pi \in \Delta(\MA)^\MS$, we have, for all $s \in \MS$ and $i \in [m]$,
  \begin{align}
    V_i^{\BG', \pi}(s) =& \E_{\ba \sim \pi(s)} \left[ r(s,\ba) + \frac 12 \cdot \sum_{s' \in \MS} \frac{\gamma}{1/2} \cdot V_i^{\BG', \pi}(s') \right]\nonumber\\
    V_i^{\BG, \pi}(s) =& \E_{\ba \sim \pi(s)} \left[ r(s, \ba) + \gamma \cdot \sum_{s' \in \MS} V_i^{\BG, \pi}(s')\right]\nonumber,
  \end{align}
  which immediately implies that $V_i^{\BG, \pi} \equiv V_i^{\BG', \pi}$. Hence the $\frac{\ep^8}{144}$-perfect Nash equilibria of $\BG$ and $\BG'$ coincide.  Choosing $\ep_0 = \sqrt\ep$ shows that the $\ep_0$-\Gcircuit problem reduces to the problem of finding $c \cdot \ep_0^{16}$-perfect NE in 1/2-discounted 2-player stochastic games, as desired. 

  Finally, to show \PPAD-hardness of computing approximate stationary CCE in 2-player 1/2-discounted stochastic games, we note that in any turn-based stochastic game, any stationary policy $\pi$ is equivalent to some product policy $\pi'$ (in the sense that $V_i^\pi \equiv V_i^{\pi'}$ for all $i$): in particular, $\pi'$ is the policy where at each state $s$, all players except $\ag{s}$ take some fixed action in their action set and $\ag{s}$ plays according to their marginal in $\pi(s)$. Thus, for any $\ep > 0$, an $\ep$-perfect CCE may be converted into an $\ep$-perfect NE in polynomial time. 
\end{proof}

\begin{proof}[Proof of Theorem \ref{thm:non-perfect-hardness}]
  The first part of the theorem is an immediate consequence of Theorem \ref{thm:perfect-hardness}, as we proceed to explain. Consider a turn-based stochastic game $\BG$, and $\ep > 0$. Note that an $\ep/S$-stationary NE $\pi$ of $\BG$ must satisfy $\max_{i \in [m]} \E_{s \sim \mu} \left[ V_i^{\dagger, \pi_{-i}}(s) - V_i^\pi(s) \right] \leq \ep/S$. Since $\pi$ is a product policy, we have that $V_i^{\dagger, \pi_{-i}}(s) - V_i^\pi(s) \geq 0$ for all $i \in [m]$. Thus, for all $i \in [m], s \in \MS$, we have $V_i^{\dagger, \pi_{-i}}(s) - V_i^\pi(s) \leq \ep$, i.e., $\pi$ is an $\ep$-perfect NE of $\BG$, which is \PPAD-hard to compute by Theorem \ref{thm:perfect-hardness}.

  We proceed to prove the second part of Theorem \ref{thm:perfect-hardness}. Since we assume Conjecture \ref{con:pcp-ppad}, it suffices to show that there is a constant $c > 0$ so that for all $\ep_0 < 1/12$ and $\delta_0 < 1$, the $(\ep_0, \delta_0)$-\Gcircuit problem has a polynomial-time reduction to the problem of computing $c \cdot \ep_0^{18} \delta_0^2$-stationary NE in 1/2-discounted 2-player stochastic games.

  To do so, fix $\ep \in (0,1/12), \delta \in (0,1)$ and write $\gamma = \ep^2$. Consider an instance $\MC = (V, \MG)$ of the $(\ep, \delta)$-\Gcircuit problem. We construct a circuit $\MC'$, together with a valid coloring $\phi : V' \ra \BR$, as guaranteed in the statement of Lemma \ref{lem:construct-valid-coloring}: in particular, given an $(\ep, \delta)$-approximate assignment $\pi : V' \ra [0,1]$ of $\MC'$, the restriction of $\pi$ to $V$ constitutes a $(133 \sqrt \ep, C_0 \delta / \sqrt \ep)$-approximate assignment of $\MC$ (for some constant $C_0 > 1$).

  By Lemma \ref{lem:valid-circuit-games}, we may further construct in polynomial time, given $\MC'$ together with $\phi$, a $\gamma$-discounted 2-player turn-based stochastic game $\BG$ so that, given a $\ep^4 \delta /3$-WSNE-SG of $\BG$, we may compute an $(\ep, \delta)$-assignment to $\MC'$, which thus yields a $(133 \sqrt \ep, C_0 \delta / \sqrt \ep)$-approximate assignment of $\MC$.  By Lemma  \ref{lem:reduce-to-ws}, the problem of computing an $\ep^4 \delta/3$-WSNE-SG of $\BG$ reduces to computing an $\frac{\ep^8 \delta^2}{144 \cdot 9}$-stationary NE of $\BG$. The same construction as in the proof of Theorem \ref{thm:perfect-hardness} allows us to reduce further to the problem of computing an  $\frac{\ep^8 \delta^2}{144 \cdot 9}$-stationary NE of a 2-player 1/2-discounted game $\BG'$. Choosing $\ep_0 = \sqrt \ep$ and $\delta_0 = \delta / \sqrt \ep$, we have shown that the $(\ep_0, \delta_0)$-\Gcircuit problem reduces to computing a $c \ep_0^{18} \delta_0^2$-stationary NE in 1/2-discounted 2-player stochastic games, for some constant $c > 0$.
\end{proof}

\section{Proofs for Section \ref{sec:ub-res}}\label{sec:proof-ub}

In this section, we prove Theorem \ref{thm:main-ub}, which gives a PAC-RL guarantee for \algname. In Section \ref{sec:finite-horizon-prelim} we introduce, for completeness, some basic preliminaries for finite-horizon stochastic games (closely mirroring the analogous definitions in Section \ref{sec:prelim}). In Section \ref{sec:conc-ineq} we introduce some parameters and concentration inequalities. In Section \ref{sec:intermed-game}, we introduce an intermediate stochastic game that is used in the analysis, which is reminiscient of the analysis of \texttt{Rmax} \cite{brafman2002r,jin2020reward}.  In Section \ref{sec:ub-proof-final}, we complete the proof of Theorem \ref{thm:main-ub}.

\subsection{Preliminaries for finite-horizon stochastic games}
\label{sec:finite-horizon-prelim}
We first introduce the requisite notation and terminology regarding finite horizon games: a finite-horizon $m$-player stochastic game $\BG$ is defined as a tuple $(\MS, (\MA_i)_{i \in [m]}, \BP, (r_i)_{i \in [m]}, H, \mu)$, which have the same interpretations as in the infinite-horizon discounted case, with the following exceptions:
\begin{itemize}
\item $H \in \BN$ denotes the horizon (replacing the discount factor $\gamma$); in particular, a trajectory proceeds for a total of $H$ steps, at which point it terminates.
\item The reward and transitions are allows to depend on the step $h \in [H]$: in particular, $r_i$ is to be interpreted as a tuple $r_i = (r_{i,1}, \ldots, r_{i,H})$, where each $r_{i,h} : \MS \times \MA \ra [-1,1]$, and $\BP$ is to be interpreted as a tuple $\BP = (\BP_1, \ldots, \BP_H)$, where each $\BP_h : \MS \times \MA \ra \Delta(\MS)$. 
\end{itemize}
When discussing the finite-horizon case, we consider only \emph{nonstationary} policies, which are sequences of maps $\pi = (\pi_1, \ldots, \pi_H)$, where each $\pi_h : \MS \ra \Delta(\MA)$; we will therefore drop the descriptor ``nonstationary''. The space of such policies is denoted $\Delta(\MA)^{[H] \times \MS}$. With a slight abuse of notation we will denote the value function of a nonstationary policy $\pi$ by $V_{i,h}^\pi : \MS \ra \BR$, $h \in [H]$, which is defined similarly to (\ref{eq:vihpi}) except with no discount factor; in particular, in the finite-horizon setting, we have, for all $i \in [m], h \in [H], s \in \MS$,
\begin{align}
V_{i,h}^\pi(s) = \E_{(s_h, \ba_h, \ldots, s_H, \ba_H) \sim (\BG, \pi)}\left[ \sum_{h'=h}^H r_{i,h}(s_{h'}, \ba_{h'}) | s_h = s\right],\nonumber
\end{align}
and $V_{i,h}^\pi(\mu) := \E_{s \sim \mu} \left[ V_{i,h}^\pi(s)\right]$. We also write $V_i^\pi := V_{i,1}^\pi$ for simplicity, as in the infinite-horizon case.  Given a policy $\pi \in \Delta(\MA)^{[H] \times \MS}$ and $i \in [m]$, the best response policy $\pi_i^\dagger(\pi_{-i})$ is defined exactly as in the infinite-horizon case, so that, in particular, $V_{i,h}^{\dagger, \pi_{-i}}(s) = \sup_{\pi_i' \in \Delta(\MA_i)^{[H] \times \MS}} V_{i,h}^{\pi_i' \times \pi_{-i}}(s)$ for all $(h,s) \in [H] \times \MS$. Finally, the notion of $\ep$-(nonstationary) CCE is defined exactly as in Definitions \ref{def:cce}, recalling that $V_i^\pi(\mu) = V_{i,1}^\pi(\mu)$ and $V_i^{\dagger,\pi_{-i}}(\mu) = V_{i,1}^{\dagger, \pi_{-i}}(\mu)$ by definition. When we wish to clarify the SG $\BG$ that corresponds to a value function, we will write $V_i^{\BG, \pi}$ in place of $V_i^\pi$.

Given a policy $\pi \in \Delta(\MA)^\MS$, the \emph{state visitation distribution} $d_h^\pi$ at step $h$ for the policy $\pi$ is defined similarly to in the infinite-horizon discounted case, except with different normalization: for all $s \in \MS$, $ d_h^\pi(s) := \BP_{(s_1, \ldots, s_H) \sim (\BG, \pi)} \left( s_h = s \right).$ 
Here the trajectory $(s_1, \ldots, s_H)$ drawn from $(\BG, \pi)$, is drawn with initial state state $s_1 \sim \mu$. 

Given an infinite-horizon discounted game $\BG'$ and a desired accuracy level $\ep$, we consider the following finite-horizon game $\BG$ with horizon $H := \frac{\log 1/\ep}{1-\gamma}$, so that $\gamma^H \leq \ep$. The state space, action space, initial state distribution, and transitions at each step of $\BG$ are the same as those of $\BG'$. %
Letting $r_i' : \MS \ra [-1,1]$ denote the reward function of $\BG'$ and $r_{i,h} : \MS \ra [-1,1]$ (for $h \in [H]$) denote the reward function of $\BG$, we define $r_{i,h} (s, \ba) := \gamma^{h-1} \cdot r_i(s,\ba)$. It is straightforward to see that for all nonstationary policies $\pi' \in \Delta(\MA)^{\BN \times \MS}$, the truncation of $\pi'$ to the first $H$ steps, which we denote by $\pi \in \Delta(\MA)^{[H] \times \MS}$ satisfies, for all $i \in [m], s \in \MS$, $\left|\frac{V_i^{\BG', \pi'}(s)}{1-\gamma} - V_i^{\BG, \pi}(s) \right| \leq \frac{\ep}{1-\gamma}$. Thus, given an $\ep$-CCE of $\BG$, we may readily construct a $2\ep$-CCE of $\BG'$. Furthermore, if our algorithm is given access to $\BG'$ in the episodic PAC-RL model of Section \ref{sec:pac-rl-sg}, we may readily simulate access to $\BG$ by drawing the first $H$ steps of a trajectory of $\BG'$ and discounting the reward received at each step $h \in [H]$ by a factor of $\gamma^{h-1}$. Thus, for the remainder of the section, we proceed to discuss the problem of learning approximate CCE in finite-horizon general-sum stochastic games (i.e., the proof of Theorem \ref{thm:main-ub}).

\subsection{Parameters \& concentration inequalities}
\label{sec:conc-ineq}
\nc{\eptvd}{\varepsilon^{\rm tvd}}
\nc{\epval}{\varepsilon^{\rm val}}
\nc{\epreg}{\varepsilon^{\rm reg}}

Fix a finite-horizon stochastic game $\BG = (\MS, (\MA_i)_{i \in [m]}, \BP, (r_i)_{i \in [m]}, H, \mu)$, and an error parameter $\ep > 0$ as well as a failure probability $\delta > 0$. 
For fixed values of the above, we introduce the following notation and parameters for use throughout this section:

\begin{itemize}
\item Choose $p = \frac{\ep}{16 SH^2}$.
\item Choose $J = C_J \cdot \frac{H^6 \iota^2 \cdot \max_{i \in [m]} A_i}{\ep^2}$, for some sufficiently large constant $C_J > 2$ (to be specified below).
\item Choose $K = \frac{8J}{p}$.
\item Choose $\epval = \frac{\ep}{4H}$.
\item Choose $\epreg = \frac{\ep}{8H}$.
\item Choose $\eptvd = p/2$.
\item Choose $N_{\rm visit} = C_N \cdot \frac{S \iota}{(\eptvd)^2}$, for some sufficiently large constant $C_N > 1$ (to be specified below). 
\item Set $\iota := \log\left( \frac{SH \max_{i \in [m]} A_i}{\ep \delta} \right)$.
\item Let $\wh q$ denote the value of $q$ at termination of \algname (i.e., $\wh q$ denotes the total number of stages completed by the algorithm).
\end{itemize}
Also, recall the following parameters introduced in \algname:
\begin{itemize}
\item For $q \geq 1$ and $k \geq 1$ we let $(s_{1,k}^q, \ba_{1,k}^q, \br_{1,k}^q, \ldots, s_{H,k}^q, \ba_{H,k}^q, \br_{H,k}^q)$ denote the trajectory drawn in step \ref{it:draw-trajectory} of \algname at stage $q$.
\item For $h \in H$, we write $\wh s_{h,k} := s_{h,k}^{\wh q}, \wh\ba_{h,k} := \ba_{h,k}^{\wh q}, \wh\br_{h,k} = \br_{h,k}^{\wh q}$ to denote the trajectory at the final stage $\wh q$.  
\item For $j \geq 1, h \in [H]$, and $q \geq 1$, let $k_{j,h,s}^q$ and $J_{h,s}^q$ denote the values of the parameters $k_{j,h,s}$ and $J_{h,s}$ defined in steps \ref{it:define-vbar} and \ref{it:define-jhs} at stage $q$.
\item For all $j,h,s$, write $\wh k_{j,h,s} := k_{j,h,s}^{\wh q}$ and $\wh J_{h,s} := J_{h,s}^{\wh q}$ to denote the values at the final stage.
\item For $q \geq 1$, let $\MV^q$ denote the value of the set $\MV$ at the beginning of stage $q$ of \algname. Furthermore we will write $\wh\MV$ to denote $\MV^{\wh q}$, which is the value of $\MV$ at the termination of \algname (by the termination criterion, the value of the set $\MV$ does not change during the final stage $\wh q$. 
\end{itemize}
Throughout the proof, we let $C > 1$ denote a constant whose value may change from line to line. 

Our first basic lemma states that \algname always terminates (in particular, the for loop at step \ref{it:loop-q} terminates). 
\begin{lemma}
  \label{lem:bad-stages-bounded}
The algorithm \algname terminates after at most $SH$ stages (i.e., $\wh q \leq SH$). %
\end{lemma}
\begin{proof}
If \algname does not terminate at some stage $q$, then it must add some pair $(h', s)$ to $\MV^q$ at that stage, which did not previously belong to $\MV^q$. Since elements of $\MV^q$ are never removed, the total number of stages is bounded above by $SH$. 
\end{proof}

The next lemma states that the state visitation estimates constructed in step \ref{it:state-visitation-estimate} of \algname are accurate with high probability.
\begin{lemma}
  \label{lem:dist-estimation}
  There is an event $\ME^{\rm visitation}$ that occurs with probability at least $1-\delta$ so that under the event $\ME^{\rm visitation}$, for all stages $q \geq 1$, and all $h' \in [H]$, it holds that
  \begin{align}
\left\| d_{h'}^{\til \pi^q} - \wh d_{h'}^{q} \right\|_1 \leq \eptvd\nonumber.
  \end{align}
\end{lemma}
\begin{proof}
  Consider any call to $\EST(\pi, N)$, which produces outputs $(\wh d_1, \ldots, \wh d_H)$. Then by \cite[Theorem 1]{canonne2020short}, for any $h \in [H]$, with probability $1-\delta/ (H^2 S)$, as long as $N \geq C \cdot \frac{S + \log(H^2 S/\delta)}{(\eptvd)^2}$ (for a sufficiently large constant $C$), it holds that $\left\| d_h^\pi - \wh d_h \right\|_1 \leq \eptvd$. Taking a union bound over all $h \in [H]$ and the at most $SH$ stages $q$ at which \EST is called at step \ref{it:state-visitation-estimate} of \algname, we obtain the claim of the lemma as long as $N_{\rm visit} \geq  C \cdot \frac{S + \log(H^2 S/\delta)}{(\eptvd)^2}$; but this inequality is ensured by our choice of $N_{\rm visit}$ in Section \ref{sec:conc-ineq}. 
\end{proof}

\begin{lemma}
  \label{lem:coverage}
  There is an event $\ME^{\rm coverage}$ that occurs with probability at least $1-\delta$ so that under the event $\ME^{\rm coverage} \cap \ME^{\rm visitation}$, for all stages $q$, $h \in [H]$, and $s \in \MS$, then if $(h,s) \in \MV^q$, it holds that $J_{h,s}^q \geq J$. 
\end{lemma}
\begin{proof}
  If $(h,s) \in \MV^q$, then for some $q' < q$, we must have that $\wh d_h^{q'}(s) \geq p$, and $\pi_{h,s}^\cover$ was set to $\til \pi^{q'}$. By Lemma \ref{lem:dist-estimation} and since $\eptvd \leq p/2$, it  must hold that $d_h^{\til \pi^{q'}}(s) \geq p/2$  under  the event $\ME^{\rm visitation}$. Let us now condition on the event $\ME^{\rm visitation}$; then in the for-loop in step \ref{it:loop-over-pi} corresponding to the policy $\pi_{h,s}^\cover = \til \pi^{q'}$, 
  for each of the $K$ episodes (in the loop in step \ref{it:K-episodes}), $(h,s)$ is visited with probability at least $p/2$. Thus, by the Chernoff bound, with probability at least $1- e^{-Kp/8}$, the number of episodes $k$ in the loop in step \ref{it:loop-over-pi} at which $(h,s)$ is visited is at least $Kp/2$, i.e., we must have $J_{h,s}^q \geq Kp/2$. Then by the union bound, since $Kp/2 \geq J$ and $Kp/8 > 2 \iota \geq \log(S^2 H^2 / \delta)$ (see Section \ref{sec:conc-ineq}),  under an event $\ME^{\rm coverage}$ occuring with probability at least $1-\delta$, for all stages $q$ and all $(h,s) \in \MV^q$, the state $(h,s)$ will be visited at least $Kp/2 \geq J$ times at stage $q$. %
  This completes the proof of the lemma.  %
\end{proof}

\subsection{Intermediate game}
\label{sec:intermed-game}

Recall that $\wh \MV = \MV^{\wh q}$ denotes the value of the set $\MV$ at termination of \algname. Define a stochastic game $\BG_{\wh \MV}$ as follows:
\begin{itemize}
\item The action space of $\BG_{\wh \MV}$ is $\MA$ and the state space of $\BG_{\wh \MV}$ is $\MS \cup \sinko$, where $\sinko$ is a special state which always transitions to itself deterministically and at which all players receive reward 1 at each step. 
\item For all $(h,s) \in {\wh \MV}$, the transitions and reward of $\BG_{{\wh \MV}}$ are identical to that of $\BG$.
\item For all $(h,s) \not \in {\wh \MV}$, all joint actions at $(h,s)$ yield reward 1 to all players and transition to to the state $\sinko$. %
\end{itemize}
In the remainder of the section, we will be working with the parameters of both $\BG$ and $\BG_{\wh \MV}$: to avoid amibuity, we denote their value functions $V_{i,h}^{\BG_{\wh \MV}, \pi}$ and $V_{i,h}^{\BG, \pi}$; we denote their transitions as $\BP_h^{\BG_{\wh \MV}}$ and $\BP_h^{\BG}$; and we denote their reward functions as $r_{i,h}^{\BG_{\wh \MV}}$ and $r_{i,h}^{\BG}$. 

Recall that $\wh \pi$ denotes the policy output by \algname. %
The below lemma shows that the value functions $\Vo_{i,h}^{\wh q}$ constructed at the final stage of \algname are close  to those of $\BG_{\wh \MV}$ under  $\wh \pi$. 
\begin{lemma}
  \label{lem:conc-onpolicy}
  There is an event $\ME^{\rm val}$ that occurs with probability at least $1-2\delta$, so that under the event $\ME^{\rm val}$, for all $s \in \MS,\ h \in [H],\ i \in [m]$, it holds that
  \begin{align}
\left| \Vo_{i,h}^{\wh q}(s) - V_{i,h}^{\BG_{{\wh \MV}}, \wh \pi}(s) \right| \leq \epval\label{eq:conc-onpolicy}.
  \end{align}
\end{lemma}
\begin{proof}
  We use reverse induction on $h$, noting that the base case $h=H+1$ is immediate since all value functions are identically 0. Now suppose that there is some constant $C$ so that for all $h' > h$, it holds that, for some event $\ME^{\rm val}_{h'}$, under the event $\ME^{\rm val}_{h'} \cap \ME^{\rm coverage}$, for all $s \in \MS$ and $i \in [m]$,%
  \begin{align}
\left| \Vo_{i,h'}^{\wh q}(s) - V_{i,h'}^{\BG_{{\wh \MV}}, \wh \pi}(s) \right| \leq (H+1-h') \cdot CH \sqrt{\frac{\iota}{J}}.\label{eq:vgprime-induction}
  \end{align}
We will show that (\ref{eq:vgprime-induction}) holds with $h' = h$, for an appropriate choice of the event $\ME^{\rm val}_h \supset \bigcup_{h' \geq h} \ME^{\rm val}_{h'}$. To do so, consider any state $s \in \MS$ and any agent $i \in [m]$; we consider the following two cases regarding $s$:

\vspace{7pt}
  \noindent{\bf Case 1.} $(h,s) \in {\wh \MV}$. By Lemma \ref{lem:coverage}, under the event $\ME^{\rm coverage}$, we have that $\wh J_{h,s} = J_{h,s}^{\wh q} \geq J$. For each $j \geq 1$, recall that we have defined $\wh k_{j,h,s} = k_{j,h,s}^{\wh q} \in [\KMAX+1]$ (and $k_{j,h,s}^{\wh q}$ is defined in step \ref{it:define-vbar} of \algname).  As $h,s$ are fixed, we will write $k_j := \wh k_{j,h,s}$.  It is evident that $k_j$ is a stopping time for each $j$. Note that, for any $t \geq 1$, the sequence
  \begin{align}
\left( \One[k_j \leq \KMAX] \cdot \left( \wh r_{i,h,k_j} + \Vo_{i,h+1}^{\wh q}(\wh s_{h+1,k_j}) - \E_{s' \sim \BP_h^\BG(\cdot | \wh s_{h,k_j}, \wh \ba_{h,k_j})} \left[ r_{i,h}^\BG(\wh s_{h,k_j}, \wh\ba_{h,k_j}) + \Vo_{i,h+1}^{\wh q}(s') \right] \right) \right)_{1 \leq j \leq t}
  \end{align}
  is a martingale difference sequence with respect to the filtration $\MF_j$, where $\MF_j$ denotes the sigma-field generated by all random variables up to step $h+1$ of episode $k_j$ (i.e., of stage $\wh q$). By the Azuma-Hoeffding inequality and a union bound, it follows that, with probability at least  $1-\delta/(HmS)$, for all $1 \leq t \leq \KMAX$,
  \begin{align}
    & \frac{1}{t} \cdot \left| \sum_{j=1}^t \One[k_j \leq \KMAX] \cdot  \left( \wh r_{i,h,k_j} + \Vo_{i,h+1}^{\wh q}(\wh s_{h+1,k_j}) - \E_{s' \sim \BP^{\BG}_h(\cdot | \wh s_{h,k_j}, \wh \ba_{h,k_j})} \left[ r_{i,h}^\BG(\wh s_{h,k_j}, \wh \ba_{h,k_j}) + \Vo_{i,h+1}^{\wh q}(s') \right] \right) \right|\nonumber\\
    \leq & CH \sqrt{\frac{\iota}{t}}\nonumber,
  \end{align}
  where $C > 1$ denotes some constant. 
  Let $\ME^{\rm val}_h$ denote the intersection of $\ME_{h+1}^{\rm val}$ and all instances of this probability $1-\delta/(Hm S)$ event, over $i \in [m], s \in \MS$.   
  In particular, the above inequality holds for $t = \wh J_{h,s}$ under $\ME_h^{\rm val}$,  
  which gives that, under the event $\ME^{\rm coverage} \cap \ME^{\rm val}_h$, $\wh J_{h,s} \geq J$, and so %
  \begin{align}
\left| \Vo_{i,h}^{\wh q}(s) - \frac{1}{\wh J_{h,s}} \sum_{j=1}^{\wh J_{h,s}} \E_{s' \sim \BP^{\BG}_h(\cdot | s, \wh \ba_{h,k_j})} \left[ r_{i,h}^\BG(s, \wh \ba_{h,k_j}) + \Vo_{i,h+1}^{\wh q}(s')\right] \right| \leq CH \sqrt{\frac{\iota}{J}}\label{eq:vbar-azuma}. 
  \end{align}
  (Here we have also used that $\wh s_{h,k_j} = s$ for all $j \leq \wh J_{h,s}$ by the definition of $k_j$.)

  By definition of $\wh \pi$, we have that, again for the fixed value of $(h,s,i)$,
  \begin{align}
V_{i,h}^{\BG_{{\wh \MV}}, \wh \pi}(s) =  \frac{1}{\wh J_{h,s}} \cdot \sum_{j=1}^{\wh J_{h,s}} \E_{s' \sim \BP_h^{\BG}(\cdot | s, \wh \ba_{h,k_j})} \left[ r_{i,h}^{\BG}(s, \wh \ba_{h,k_j}) + V_{i,h+1}^{\BG_{{\wh \MV}}, \wh \pi}(s') \right] \nonumber.
  \end{align}
  Here we have used that since $(h,s) \in \wh \MV$, it holds that for all $\ba \in \MA$, $\BP_h^{\BG}(\cdot | s, \ba) = \BP_h^{\BG_{\wh \MV}}(\cdot | s,\ba)$ and $r_{i,h}^{\BG}(s,\ba) = r_{i,h}^{\BG_{\wh \MV}}(s,\ba)$. 
  By the inductive hypothesis (\ref{eq:vgprime-induction}) with $h'=h+1$, it holds that under the event $\ME^{\rm coverage} \cap \ME^{\rm val}_{h+1}$, 
  \begin{align}
\left| V_{i,h+1}^{\BG_{{\wh \MV}}, \wh \pi}(s) - \Vo_{i,h+1}^{\wh q}(s) \right| \leq & (H-h) \cdot CH \sqrt{\frac{\iota}{J}}\label{eq:use-inductive-j}.
  \end{align}
  Combining (\ref{eq:vbar-azuma}) and (\ref{eq:use-inductive-j}), we get that, under the event $\ME^{\rm coverage} \cap \ME_h^{\rm val}$,
  \begin{align}
\left| \Vo_{i,h}^{\wh q}(s) - V_{i,h}^{\BG_{{\wh \MV}}, \wh \pi}(s) \right| \leq (H+1-h) \cdot CH \sqrt{\frac{\iota}{J}}\nonumber,
  \end{align}
  thus completing the inductive step in the case that $(h,s) \in {\wh \MV}$. 
   
\vspace{7pt}
\noindent{\bf Case 2.} $(h,s) \not \in {\wh \MV}$. Here we note that, by (\ref{eq:define-vbar}), $\Vo_{i,h}^{\wh q}(s) = H+1-h$. %
Furthermore, it is immediate from the definition of $\BG_{{\wh \MV}}$ that $V_{i,h}^{\BG_{{\wh \MV}}}(s) = H+1-h$ since $(h,s) \not \in {\wh \MV}$. %
Hence, we have $\left| \Vo_{i,h}^{\wh q}(s) - V_{i,h}^{\BG_{\wh \MV},\wh \pi}(s)\right| = 0$ in this case. 

  Thus we have verified that in all cases, (\ref{eq:vgprime-induction}) holds at step $h$, thus completing the inductive step.

  Summarizing, if we set $\ME^{\rm val} = \ME_1^{\rm val} \cap \ME^{\rm coverage}$, then the guarantee (\ref{eq:conc-onpolicy}) holds as long as we have $\epval \geq CH^2 \sqrt{\frac{\iota}{J}}$, which is ensured by choosing the constant $C_J$ large enough (see Section \ref{sec:conc-ineq}).   Furthermore, by a union bound (over all values of $h \in [H], i \in [m], s \in \MS$, $\ME^{\rm val}$ holds with probability at least  $1-2\delta$. 
\end{proof}

For each $s \in \MS, i \in [m],\ a_i \in \MA_i$, define  
\begin{align}\label{equ:def_Q_bar}
  \Qo_{i,h}(s,a_i) =  \E_{\ba_{-i} \sim \wh \pi_{-i,h}(s)} \E_{s' \sim \BP_h^\BG(\cdot | s, (a_i, \ba_{-i}))} \left[ r_{i,h}^\BG(s,(a_i, \ba_{-i}))+  \Vo_{i,h+1}^{\wh q}(s') \right].
\end{align}

In Lemma \ref{lem:bandit-no-reg} below, we use Theorem \ref{thm:adv-bandit-external} (giving a no-regret property for the bandit learners used at each state in \algname) to bound the difference between $\Qo_{i,h}$ and $\Vo_{i,h}^{\wh q}$. 
\begin{lemma}
  \label{lem:bandit-no-reg}
  There is an event $\ME^{\rm reg}$ that occurs with probability at least $1-\delta$, so that under the event $\ME^{\rm reg}$, for all $s \in \MS, \ h \in [H],\ i \in [m],$ it holds that
  \begin{align}
\max_{a_i \in \MA_i} \left(\Qo_{i,h}(s,a_i) - \Vo_{i,h}^{\wh q}(s) \right)\leq \epreg. \label{eq:no-reg-bars}
  \end{align}
\end{lemma}
\begin{proof}
  Fix any $(s,h,i) \in \MS \times [H] \times [m]$. First we treat the case that $(h,s) \not \in {\wh \MV}$. Note that by the definition (\ref{eq:define-vbar}) we have that $\Vo_{i,h}^{\wh q}(s) \leq H+1-h$ for all $i,h,s$. It then follows that $\Qo_{i,h}(s,a_i) \leq H+1-h$,  
  for all $i,h,s,a_i$. Furthermore, we have that $\Vo_{i,h}^{\wh q}(s) = H+1-h$ when $(h,s) \not \in {\wh \MV}$ (by the definition (\ref{eq:define-vbar}) in \algname),  
  meaning that $\max_{a_i \in \MA_i} \left( \Qo_{i,h}(s,a_i) - \Vo_{i,h}^{\wh q}(s) \right) \leq 0$.

  For the remainder of the proof treat those pairs $(h,s)$ so that $(h,s) \in {\wh \MV}$. Fix some value of $(h,s) \in \wh \MV$, and set, for each $j \geq 1$, $k_j := \wh k_{j,h,s}$. It is evident that for each $j$, $k_j$ is a stopping time with respect to the filtration $\MH_k$, where $\MH_k$ denotes the sigma-field generated by all states and actions taken up to (and including) step $h+1$ of episode $k$.  %
  
  Note that each agent $i$ runs the adversarial bandit algorithm at state $s$ and step $h$ at each episode $k_j$ (as $\wh s_{h,k_j} = s$ by definition of $k_j$). Furthermore, for each $j$ so that $k_j \leq \KMAX$, the expected reward that agent $i$ would receive upon playing action $a_i \in \MA_i$, conditioned on the actions $\wh\ba_{-i,h,k_j}$ taken by all other agents at step $h$ of episode $k_j$, is given by   
  \begin{align}
\ell_j(a_i) := \E_{s' \sim \BP_h^\BG(\cdot | s, (a_i, \wh \ba_{-i,h,k_j}))} \left[ \frac{H - r_{i,h}^\BG(s, (a_i, \wh \ba_{-i,h,k_j})) - \Vo_{i,h+1}^{\wh q}(s')}{H} \right]\label{eq:def-ellj}.
  \end{align}
  Furthermore, it is evident that, for each $j \geq 1$, for the choice of action $\wh a_{i,h,k_j}$ by agent $i$'s bandit algorithm at $(s,h)$, the feedback
  $$
  \til \ell_j(\wh a_{i,h,k_j}) := \frac{H - r_{i,h,k_j} - \Vo_{i,h+1}^{\wh q}(\wh s_{h+1,k_j})}{H}
  $$
  fed to the bandit algorithm satisfies to $\E[\til \ell_j(\wh a_{i,h,k_j}) | \MF_{i,j-1} ] =\ell_j(\wh a_{i,h,k_j})$, where $\MF_{i,j}$ denotes the the sigma-field generated by all states and actions taken up to (and including) step $h$ of episode $k_{j+1}$. It is straightforward to see that $\MF_{i,j}$ is well-defined, as $k_{j+1}$ is a stopping time. 
  \kz{$\MG$ has been used before for the set of gates. maybe we could  use $\cF$ (even more standard)? 
  Finally, when defining the subscript of $\cF$, do we need to have index $h$ (in addition to $j$) to avoid confusion (or maybe it is  clear)?}\noah{will add subscript $h$ if time} It is also evident that $\til \ell_j(\wh a_{i,h,k_j})$ is $\MF_{i,j}$-measurable, meaning that $\sum_{j=1}^t \One[k_j \leq \KMAX] \cdot (\til \ell_j(\wh a_{i,h,k_j}) - \ell_j(\wh a_{i,h,k_j}))$ is a martingale difference sequence adapted to the filtration $\MF_{i,t}$. 
  Thus, by the Azuma-Hoeffding inequality, with probability at least $1-\delta/(SHm)$, for all $t \in [\KMAX]$, we have that for some constant $C > 0$,
  \begin{align}
\frac{1}{t} \cdot \left| \sum_{j=1}^t \One[k_j \leq \KMAX] \cdot (\til \ell_j(\wh a_{i,h,k_j}) - \ell_j(\wh a_{i,h,k_j})) \right| \leq C \sqrt{\frac{\iota}{t}}\label{eq:till-l-azuma}.
  \end{align} 
  
  Next,  we have from Theorem \ref{thm:adv-bandit-external} %
  that for some constant $C'$, with probability at least $1-\delta/(SHm)$, for all $t \in [\KMAX]$, 
  \begin{align}
\max_{a_i \in \MA_i} \sum_{j=1}^t \ell_j(\wh a_{i,h,k_j}) - \ell_j(a_i) \leq C' \iota \cdot \sqrt{t A_i}.\label{eq:use-bandit-noregret}%
  \end{align}
  From (\ref{eq:till-l-azuma}) and (\ref{eq:use-bandit-noregret}), and choosing $t = \wh J_{h,s}$, we get that, with probability at least $1-2\delta/(SHm)$,
  \begin{align}
\max_{a_i \in \MA_i} \sum_{j=1}^{\wh J_{h,s}} \til \ell_j(\wh a_{i,h,k_j}) - \ell_j(a_i) \leq C''\iota \cdot \sqrt{\wh J_{h,s} \cdot A_i },\label{eq:noreg-tils}
  \end{align}
  for some constant $C'' > 0$. 
  Let the event that (\ref{eq:noreg-tils}) holds be denoted as  $\ME_{i,h,s}$. 
  
  By definition we have $\wh \pi_h(s) \in \Delta(\MA)$ is given by the following distribution: for $a \in \MA$,
  \begin{align}
\wh \pi_h(a | s) = \frac{1}{\wh J_{h,s}} \cdot \sum_{j=1}^{\wh J_{h,s}} \One[\wh \ba_{h,k_j} = a]\nonumber.
  \end{align}
  Therefore using (\ref{equ:def_Q_bar}) and (\ref{eq:def-ellj}), we have that, for each $a_i \in \MA_i$,
  \begin{align}
\frac{1}{\wh J_{h,s}} \cdot \sum_{j=1}^{\wh J_{h,s}} \ell_j(a_i) = 1 - \frac{\Qo_{i,h}(s, a_i) }{H}.\label{eq:tils-qbar}
  \end{align}
From the definition of $\Vo_{i,h}^{\wh q}(s)$ in (\ref{eq:define-vbar}), we have 
  \begin{align}
    \frac{1}{\wh J_{h,s}} \cdot \sum_{j=1}^{\wh J_{h,s}} \til\ell_j(\wh a_{i,h,k_j}) = \frac{1}{\wh J_{h,s}} \cdot \sum_{j=1}^{\wh J_{h,s}} \left( 1 - \frac{\wh r_{i,h,k_j} + \Vo_{i,h+1}^{\wh q}(\wh s_{h+1,k_j})}{H} \right) = 1 - \frac{\Vo_{i,h}^{\wh q}(s)}{H}\label{eq:ells-vbar}.
  \end{align}
  From (\ref{eq:noreg-tils}), (\ref{eq:tils-qbar}), and (\ref{eq:ells-vbar}), we have that, under the event $\ME^{\rm coverage} \cap \ME_{i,h,s}$, 
  \begin{align}
\max_{a_i \in \MA_i}  \left(   \Qo_{i,h}(s,a_i) - \Vo_{i,h}^{\wh q}(s) \right) \leq C''\iota \cdot H\sqrt{\frac{ A_i }{J}}\label{eq:coverage-ihs}.
  \end{align}
  (In particular, we work under the event $\ME^{\rm coverage}$ to ensure that $\wh J_{h,s} \geq J$). Thus, taking a union bound  over all $i,h,s$, and letting $\ME^{\rm reg} := \ME^{\rm coverage} \cap \bigcap_{i,h,s} \ME_{i,h,s}$, which has probability at least $1-3\delta$,  we get that under the event $\ME^{\rm reg}$, for all $s,h,i$ so that $(h,s) \in {\wh \MV}$, (\ref{eq:no-reg-bars}) holds as long as we have $\epreg \geq C''\iota H \sqrt{\frac{A_i}{J}}$, which holds as long as $C_J$ is sufficiently  large (see Section \ref{sec:conc-ineq}).  
\end{proof}

Next we combine the previous lemmas in the section to show that the policy $\wh \pi$ is a coarse correlated equilibrium for the game $\BG_{\wh \MV}$. 
\begin{lemma}
  \label{lem:no-gprime-regret}
  Under the event $\ME^{\rm reg} \cap \ME^{\rm val}$, for all $i \in [m]$,  for any policy $\pi_i$ of player $i$, it holds that %
  \begin{align}
    V_{i,1}^{\BG_{{\wh \MV}}, (\pi_i, \wh \pi_{-i})}(\mu) - V_{i,1}^{\BG_{{\wh \MV}}, \wh \pi}(\mu) \leq H \cdot (\epreg + 2 \cdot \epval).\nonumber
  \end{align}
\end{lemma}
\begin{proof}
  For each $i \in [m], h \in [H], s \in \MS, a \in \MA$, we will write, for any policy  
  $\pi \in \Delta(\MA)^{[H] \times \MS}$ and any $\ba \in \MA$,
  \begin{align}
Q_{i,h}^{\BG_{{\wh \MV}}, \pi}(s,\ba) = \E_{s' \sim \BP_h^{\BG_{{\wh \MV}}}(\cdot | s,\ba)} \left[ r_{i,h}^{\BG_{{\wh \MV}}}(s,\ba) + V_{i,h+1}^{\BG_{{\wh \MV}}, \pi}(s') \right].\nonumber
  \end{align}
  Now fix any $(h,s) \in {\wh \MV}$, so that $\BP_h^\BG(\cdot | s,\ba) = \BP_h^{\BG_{{\wh \MV}}}(\cdot | s,\ba)$ and $r_{i,h}^{\BG}(s,\ba) = r_{i,h}^{\BG_{{\wh \MV}}}(s,\ba)$ for all $a \in \MA$. From \eqref{equ:def_Q_bar} and the fact that $(h,s) \in \wh \MV$, we have
  \begin{align}
    \Qo_{i,h}(s,a_i) =  \E_{\ba_{-i} \sim \wh \pi_{-i,h}(s)} \E_{s' \sim \BP_h^{\BG_{\wh \MV}}(\cdot | s, (a_i, \ba_{-i}))} \left[ r_{i,h}^{\BG_{\wh \MV}}(s,(a_i, \ba_{-i}))+  \Vo_{i,h+1}^{\wh q}(s') \right].\nonumber
  \end{align}

  Then for any fixed action $a_i \in \MA_i$ of player $i$, we have, under the event $\ME^{\rm val}$ (see Lemma \ref{lem:conc-onpolicy}),  
  \begin{align}
    & \left| \E_{\ba_{-i} \sim \wh \pi_{-i, h}(s)} \left[ Q_{i,h}^{\BG_{{\wh \MV}}, \wh \pi}(s,(a_i, \ba_{-i})) \right] - \Qo_{i,h}(s,a_i) \right| \nonumber\\
    =& \left| \E_{\ba_{-i} \sim \wh \pi_{-i,h}(s)} \left[ Q_{i,h}^{\BG_{{\wh \MV}}, \wh \pi}(s, (a_i, \ba_{-i})) - \E_{s' \sim \BP_h^{\BG_{{\wh \MV}}}(\cdot | s, (a_i, \ba_{-i}))} \left[ r_{i,h}^{\BG_{{\wh \MV}}}(s, (a_i, \ba_{-i})) + \Vo_{i,h+1}^{\wh q}(s') \right]\right] \right|\nonumber\\
    =& \left| \E_{\ba_{-i} \sim \wh \pi_{-i,h}(s)} \E_{s' \sim \BP_h^{\BG_{{\wh \MV}}}(\cdot | s, (a_i, \ba_{-i}))} \left[ V_{i,h+1}^{\BG_{{\wh \MV}}, \wh \pi}(s') - \Vo_{i,h+1}^{\wh q}(s')\right] \right|\nonumber\\
    \leq & \epval\label{eq:use-epval-qs},
  \end{align}
  where the final inequality uses (\ref{eq:conc-onpolicy}). Therefore, for all $(h,s) \in {\wh \MV}$, it holds that, under the event $\ME^{\rm val} \cap \ME^{\rm reg}$,
  \begin{align}
    & \E_{\ba_{-i} \sim \wh \pi_{-i,h}(s)} \left[ Q_{i,h}^{\BG_{{\wh \MV}}, \wh \pi}(s, (a_i, \ba_{-i})) - V_{i,h}^{\BG_{{\wh \MV}}, \wh \pi}(s) \right]\nonumber\\
    \leq &  \Qo_{i,h}(s, a_i) - \Vo_{i,h}^{\wh q}(s)  + 2 \cdot \epval\label{eq:apply-bar-closeness}\\
    \leq & \epreg + 2\cdot\epval \label{eq:apply-epreg-lemma},
  \end{align}
  where  (\ref{eq:apply-bar-closeness}) follows from (\ref{eq:use-epval-qs}) as well as $\left| V_{i,h}^{\BG_{{\wh \MV}}, \wh \pi}(s) - \Vo_{i,h}^{\wh q}(s) \right| \leq \epval$ under $\ME^{\rm val}$, and (\ref{eq:apply-epreg-lemma}) follows from Lemma \ref{lem:bandit-no-reg}.

  Now consider any $(h,s) \not \in {\wh \MV}$. Since a reward of at most 1 can be received at each step in $\BG_{{\wh \MV}}$, it holds that $Q_{i,h}^{\BG_{{\wh \MV}}, \wh \pi}(s, \ba) \leq H+1-h$ for all $i \in [m]$ and $\ba \in \MA$. Furthermore, since it still holds that  $\left| V_{i,h}^{\BG_{{\wh \MV}}, \wh \pi}(s) - \Vo_{i,h}^{\wh q}(s) \right| \leq \epval$ under $\ME^{\rm val}$, we see that
  \begin{align}
    & \E_{\ba_{-i} \sim \wh \pi_{-i,h}(s)} \left[ Q_{i,h}^{\BG_{{\wh \MV}}, \wh \pi}(s, (a_i, \ba_{-i})) - V_{i,h}^{\BG_{{\wh \MV}}, \wh \pi}(s) \right]\nonumber\\
    \leq & (H+1-h) - \Vo_{i,h}^{\wh q}(s) + \epval \leq \epval \label{eq:not-visited-regret},
  \end{align}
  where the final inequality follows from $\Vo_{i,h}^{\wh q}(s) = H+1-h$ for $(h,s) \not \in {\wh \MV}$ (see (\ref{eq:define-vbar})). 
  
Now fix any player $i$ and any policy $\pi_i$ of player $i$. Since the policy $\wh \pi_{-i}$ is a Markov policy, the value function of the game $\BG_{\wh \MV}$ as a function of player $i$'s policy is equivalent to that of a MDP. Thus, we may apply the finite horizon version of the  performance difference lemma \cite{Kakade02approximatelyoptimal},  which gives that 
  \begin{align}
V_{i,1}^{\BG_{{\wh \MV}}, (\pi_i, \wh \pi_{-i})} (\mu) - V_{i,1}^{\BG_{{\wh \MV}}, \wh \pi}(\mu) =& \E_{s_{1:H}, \ba_{1:H} \sim (\BG_{{\wh \MV}}, (\pi_i, \wh \pi_{-i}))} \left[ \sum_{h=1}^H Q_{i,h}^{\BG_{{\wh \MV}}, \wh \pi}(s_h, \ba_h) - V_{i,h}^{\BG_{{\wh \MV}}, \wh \pi}(s_h) \right]\label{eq:pd-reg-decomp}.
  \end{align}
  For each $h \in [H]$, we bound the $h$th term in the above expression as follows:
  \begin{align}
    & \E_{s_{1:H}, \ba_{1:H} \sim (\BG_{{\wh \MV}}, (\pi_i, \wh \pi_{-i}))} \left[ Q_{i,h}^{\BG_{{\wh \MV}}, \wh \pi}(s_h, \ba_h) - V_{i,h}^{\BG_{{\wh \MV}}, \wh \pi}(s_h) \right] \nonumber\\
    =& \E_{s_h \sim (\BG_{{\wh \MV}}, (\pi_i, \wh \pi_{-i}))} \E_{a_i \sim \pi_{i,h}(s_h)} \E_{\ba_{-i} \sim \wh \pi_{-i,h}(s_h)} \left[ Q_{i,h}^{\BG_{{\wh \MV}}, \wh \pi}(s_h, (a_i, \ba_{-i})) - V_{i,h}^{\BG_{{\wh \MV}}, \wh \pi}(s_h) \right]\nonumber\\
    \leq & \epreg + 2 \cdot \epval, \label{eq:replace-with-bars}
  \end{align}
  where (\ref{eq:replace-with-bars}) follows from (\ref{eq:apply-epreg-lemma}) and (\ref{eq:not-visited-regret}). Thus, from (\ref{eq:pd-reg-decomp}), we get that, under $\ME^{\rm reg} \cap \ME^{\rm val}$, 
  \begin{align}
V_{i,1}^{\BG_{{\wh \MV}}, (\pi_i, \wh \pi_{-i})} (\mu) - V_{i,1}^{\BG_{{\wh \MV}}, \wh \pi}(\mu) \leq H \cdot (\epreg + 2 \cdot \epval).\nonumber
  \end{align}
  This completes the proof. 
\end{proof}

\subsection{Completion of proof of Theorem \ref{thm:main-ub}}
\label{sec:ub-proof-final}

Lemma \ref{lem:no-gprime-regret} comes close to showing that $\wh \pi$ is an $\ep$-CCE, except that it applied to the game $\BG_{\wh \MV}$ as opposed to the game $\BG$. To get guarantees for the game $\BG$, we need to bound the probability (under $\wh \pi$ that a trajectory visits a state not in $\wh \MV$), which is done in Lemma \ref{lem:escape-prob} below.
\begin{lemma}
  \label{lem:escape-prob}
  For the output policy $\wh \pi$ of \algname, we have that, for all $h \in [H]$, under the event $\ME^{\rm visitation}$,
  \begin{align}
\Pr_{s_h \sim (\BG, \wh \pi)} \left[ (h,s_h) \not \in {\wh \MV} \right] \leq pS + \eptvd \nonumber.
  \end{align}
\end{lemma}
\begin{proof}
  Recall that  $\wh q$ denotes the index of the final stage of \algname. Under the event $\ME^{\rm visitation}$ of Lemma \ref{lem:dist-estimation}, since we have that $\wh \pi = \til \pi^{\wh q}$, 
  it holds that for all $h \in [H]$, 
  \begin{align}
    \left\| d_{h}^{\wh \pi} - \wh d_h^{\wh q} \right\|_1 \leq \eptvd.\nonumber
  \end{align}
  Since $\wh q$ is the final stage, it must be the case that for all $s \in \MS$ and $h \in [H]$ so that $(h,s) \not \in \wh \MV = \MV^{\wh q}$, it holds that $\wh d_h^{\wh q}(s) < p$. 
  In particular, for each $h \in [H]$, $\sum_{s \in \MS: (h,s) \not \in {\wh \MV}} \wh d_h^{\wh q}(s) < p S$. Thus, under the event $\ME^{\rm visitation}$, it holds that $\sum_{s \in \MS: (h,s)  \not \in {\wh \MV}} d_h^{\wh \pi}(s) < pS + \eptvd$. %

  By noting that for each $h \in [H]$
  \begin{align}
\Pr_{s_h \sim (\BG, \wh \pi)} [(h, s_h) \not \in {\wh \MV}] = \sum_{s \in \MS : (h,s) \not \in {\wh \MV}} d_h^{\wh \pi}(s),\nonumber
  \end{align}
  which concludes the proof. 
\end{proof}

Combining the previous lemmas, we now show that the policy $\wh \pi$ is an approximate CCE of $\BG$ with high probability.
\begin{lemma}
  \label{lem:cce-almost}
  Under the event $\ME^{\rm visitation} \cap \ME^{\rm reg} \cap \ME^{\rm val}$, the output policy $\wh \pi$ of \algname satisfies the following: %
  for all $i \in [m]$ and policies $\pi_i$ for player $i$, we have
  \begin{align}
V_{i,1}^{\BG, (\pi_i, \wh \pi_{-i})}(\mu) - V_{i,1}^{\BG, \wh \pi}(\mu) \leq  H \cdot (\epreg + 2 \cdot \epval) + 2H^2 \cdot (pS + \eptvd) \leq \ep\nonumber.
  \end{align} 
\end{lemma}
\begin{proof}
  We first show the following two facts hold under the joint event $\ME^{\rm visitation} \cap \ME^{\rm reg} \cap \ME^{\rm val}$:
  \begin{enumerate}
  \item For all joint policies $\pi$ and all players $i$, it holds that $V_{i,1}^{\BG, \pi}(\mu) \leq V_{i,1}^{\BG_{{\wh \MV}}, \pi}(\mu)$.\label{it:gv-optimistic}
    \item %
      It holds that $V_{i,1}^{\BG_{{\wh \MV}}, \wh \pi}(\mu) \leq V_{i,1}^{\BG, \wh \pi}(\mu) + H^2 \cdot (pS + \eptvd)$.\label{it:gv-not-too-optimistic}
    \end{enumerate}

    To see the above facts, fix any joint policy $\pi$ and note that, by definition,%
  \begin{align}
    V_{i,1}^{\BG,  \pi}(\mu) =& \E_{s_{1:H}, \ba_{1:H} \sim (\BG,  \pi)} \left[ \sum_{h=1}^H r_{i,h}^{\BG}(s_h, \ba_h) \right]\nonumber\\    V_{i,1}^{\BG_{{\wh \MV}},  \pi}(\mu) =&  \E_{s_{1:H}', \ba_{1:H}' \sim (\BG_{{\wh \MV}},  \pi)} \left[ \sum_{h=1}^H r_{i,h}^{\BG_{{\wh \MV}}}(s_h', \ba_h') \right] \nonumber.
  \end{align}
  We now construct a coupling between trajectories $(s_{1:H},\ba_{1:H}) \sim (\BG,  \pi)$ and $(s_{1:H}', \ba_{1:H}') \sim (\BG_{{\wh \MV}},  \pi)$, as follows:
  \begin{enumerate}
  \item First set $s_1 = s_1'$ to be the initial state, drawn from the initial state distribution $\mu$. 
  \item Set a parameter $\tau = 1$ and $h = 1$.
  \item While $\tau = 1$ and $h \leq H$:
    \begin{enumerate}
    \item Draw a sample $\ba_h = \ba_h' \sim  \pi_h(s_h)$.
    \item Draw a sample $s_{h+1} = s_{h+1}' \sim \BP_h^\BG(\cdot | s_h, \ba_h)$.
    \item If $(h+1, s_{h+1})\in {\wh \MV}$, 
    then increment $h$ by 1, and continue.
    \item If $(h+1, s_{h+1}) \not \in {\wh \MV}$, 
    set $\tau = 0$ and increment $h$ by 1.
    \end{enumerate}
  \item If $h \leq H$ (which means that the above loop was terminated early and we must have $\tau = 0$):
    \begin{enumerate}
    \item Draw independent samples $(\ba_{h:H}, s_{h+1:H}) \sim (\BP_h^\BG,  \pi)$, 
    and $(\ba_{h:H}', s_{h+1:H}') \sim (\BP_h^{\BG_{{\wh \MV}}},  \pi)$, conditioned on starting at state $s_h = s_h'$ at step $h$.\footnote{In the case that $h=H$, we only draw the joint action profiles $\ba_H$ and $\ba_H'$.}
    \end{enumerate}
  \end{enumerate}
  It is immediate to see that the above joint distribution of $(s_{1:H}, \ba_{1:H}, s_{1:H}', \ba_{1:H}')$ constitutes a coupling between the trajectories induced by the pairs $(\BG,  \pi)$ and $(\BG_\cV,  \pi)$. Let the distribution of this coupling be denoted by $\nu$. Note that for a pair of trajectories $(s_{1:H}, \ba_{1:H}, s_{1:H}', \ba_{1:H}')$ drawn from the distribution $\nu$, we must have, with probability 1, $s_h = s_h', \ba_h = \ba_h'$ if for all $h' \leq h$, $(h', s_{h'}) \in {\wh \MV}$.   Let $\MJ_h$ %
  denote the event that for all $h' \leq h$, $(h', s_{h'}) \in {\wh \MV}$, and let $\chi_{\MJ_h} \in \{0,1\}$ denote the indicator of $\MJ_h$.

  Next, we claim that for all $h \in [H]$ and $i \in [m]$, with probability 1, $r_{i,h}^\BG(s_h, \ba_h) \leq r_{i,h}^{\BG_{{\wh \MV}}}(s_h', \ba_h')$: this is evident under the event $\MJ_h$, since then we have $(s_h, \ba_h) = (s_h', \ba_h') \in {\wh \MV}$. Furthermore, if $\MJ_h$ does not hold, then for some (random) $h' \leq h$, we have $(h', s_{h'}) \not \in {\wh \MV}$, meaning that, regardless of the policy $\pi$, $r_{i,h}^{\BG_{{\wh \MV}}}(s_h, \ba_h) = 1$ (since we have either $s_h = \sinko$, and $r_{i,h}^{\BG_{{\wh \MV}}}(\sinko, \ba) = 1$ for all $\ba$, or $h = h'$ in which case we have defined $r_{i,h}^{\BG_{{\wh \MV}}}(s_h', \ba) = 1$ for all $\ba$). It follows that, for any policy $\pi$,%
  \begin{align}
     V_{i,1}^{\BG_{{\wh \MV}}, \pi}(\mu) - V_{i,1}^{\BG, \pi}(\mu) 
    = \E_{(s_{1:H}, \ba_{1:H}, s_{1:H}', \ba_{1:H}') \sim \nu} \left[ \sum_{h=1}^H \left( r_{i,h}^{\BG_{{\wh \MV}}}(s_h',\ba_h') - r_{i,h}^{\BG}(s_h, \ba_h) \right) \right]\geq 0,\nonumber
  \end{align}
  where $\nu$ is the  joint distribution of $(s_{1:H}, \ba_{1:H}, s_{1:H}', \ba_{1:H}')$  corresponding to $\pi$,  
  establishing the first of our claims (item \ref{it:gv-optimistic}) above. 

Next we establish item \ref{it:gv-not-too-optimistic}, for which we only need to consider the policy $\pi = \wh \pi$ output by \algname. Under the event $\ME^{\rm visitation}$, we have %
  \begin{align}
    & \left| V_{i,1}^{\BG, \wh \pi}(\mu) - V_{i,1}^{\BG_{{\wh \MV}}, \wh \pi}(\mu) \right|\nonumber\\
    =& \left|\E_{(s_{1:H}, \ba_{1:H}, s_{1:H}', \ba_{1:H}') \sim \nu} \left[ \sum_{h=1}^H \left( r_{i,h}^{\BG}(s_h, \ba_h) - r_{i,h}^{\BG_{{\wh \MV}}}(s_h', \ba_h') \right) \right]\right|\nonumber\\
    \leq & \sum_{h=1}^H \left| \E_\nu \left[ \chi_{\MJ_h} \cdot \left( r_{i,h}^\BG(s_h, \ba_h) - r_{i,h}^{\BG_{{\wh \MV}}}(s_h', \ba_h')\right) \right]\right| + \sum_{h=1}^H \E_\nu \left[ \left| (1 - \chi_{\MJ_h}) \cdot  \left( r_{i,h}^\BG(s_h, \ba_h) - r_{i,h}^{\BG_{{\wh \MV}}}(s_h', \ba_h')\right) \right| \right]\nonumber\\
    \leq & 2 \sum_{h=1}^H \E_\nu \left[ 1 - \chi_{\MJ_h} \right] \label{eq:chifh}\\
    \leq & 2H \cdot \BP_{s_{1:H}, \ba_{1:H} \sim (\BG, \wh \pi)}\left[ \exists h \in [H] \ : \ (h, s_h) \not \in {\wh \MV} \right] \nonumber\\
    \leq & 2H \sum_{h=1}^H \BP_{s_{1:H}, \ba_{1:H} \sim (\BG, \wh \pi)} \left[ (h,s_h) \not \in {\wh \MV} \right] \nonumber\\
    \leq & 2H^2 \cdot (pS + \eptvd)\label{eq:use-escape-prob},
  \end{align}
  where %
  (\ref{eq:chifh}) follows because $r_{i,h}^\BG(s_h, \ba_h) - r_{i,h}^{\BG_{{\wh \MV}}}(s_h', \ba_h') = 0$ whenever $\chi_{\MJ_h} = 1$ (as then $(s_h, \ba_h) = (s_h', \ba_h') \in \wh \MV$),  
  and (\ref{eq:use-escape-prob}) uses the conclusion of Lemma \ref{lem:escape-prob} and the fact that $\ME^{\rm visitation}$ is assumed to hold. 

  Using items \ref{it:gv-optimistic} (with the policy $(\pi_i, \wh \pi_{-i})$) and \ref{it:gv-not-too-optimistic} above, we obtain that, under the event $\ME^{\rm visitation} \cap \ME^{\rm reg} \cap \ME^{\rm val}$, 
  \begin{align}
    V_{i,1}^{\BG, (\pi_i, \wh \pi_{-i})}(\mu) - V_{i,1}^{\BG, \wh \pi}(\mu) \leq &  V_{i,1}^{\BG_{{\wh \MV}}, (\pi_i, \wh \pi_{-i})}(\mu) - V_{i,1}^{\BG_{{\wh \MV}}, \wh \pi}(\mu) + H^2 \cdot (pS + \eptvd) \nonumber\\
    \leq & H \cdot (\epreg + 2 \cdot \epval) + 2H^2 \cdot (pS + \eptvd) \leq \ep\nonumber,
  \end{align}
  where the second-to-last inequality follows from Lemma \ref{lem:no-gprime-regret} and the final inequality follows by our choices of $\epreg, \epval, p, \eptvd$ in Section \ref{sec:conc-ineq}. 
\end{proof}

Finally, we may prove Theorem \ref{thm:main-ub} as a consequence of Lemma \ref{lem:cce-almost} and the choices of our parameters.
\begin{proof}[Proof of Theorem \ref{thm:main-ub}]
  Consider any stochastic game $\BG$ and any $\ep, \delta > 0$. Lemma \ref{lem:cce-almost} gives that under the event $\ME^{\rm visitation} \cap \ME^{\rm reg} \cap \ME^{\rm val}$ (which has probability at least $1- 4\delta$), we have that the output policy $\wh \pi$ of \algname satisfies
  \begin{align}
\max_{i \in [m]} \left\{ V_{i,1}^{\BG, (\dagger, \wh \pi_{-i})}(\mu) - V_{i,1}^{\BG, \wh \pi}(\mu) \right\} \leq \ep\nonumber,
  \end{align}
  which implies that $\wh \pi$ is an $\ep$-(nonstationary) CCE (Definition \ref{def:cce}). It remains to bound the number of trajectories collected by \algname: it is seen by inspection to be bounded above by
  \begin{align}
    \sum_{q=1}^{\wh q} \left( \sum_{h=1}^H |\Pi_h^q| \cdot K + N_{\rm visit} \right)\leq & HS \cdot (SHK + N_{\rm visit}) \nonumber\\
    \leq & HS \cdot O\left( SH \cdot \frac{J}{p} + \frac{S \iota}{p^2} \right)\nonumber\\
    \leq & HS \cdot O \left( \frac{SH \cdot H^6 \iota^2 \cdot \max_i A_i \cdot SH^2}{\ep^3} + \frac{S\iota \cdot SH^2}{\ep^2} \right)\nonumber\\
    \leq & O \left( \frac{H^{10} S^3  \iota^2\max_{i \in [m]} A_i}{\ep^3} \right)\nonumber.
  \end{align}
  Finally, the proof is completed by rescaling $\delta$ to be $\delta/4$. 
\end{proof}

\kz{TO DO conclusion}\noah{do we have open problems we want to state (obvious one is improving sample complexity of our alg + using UCB bonuses/making it more natural)?}\kz{yea, also maybe the complexity for normal-form CCE in SGs? learning Markov CE (this might be immediate)? the case with function approximation?} 

\bibliographystyle{alpha}
\bibliography{refs,refsb}

\newpage  
\appendix

\section{Related Work}\label{sec:append_related_work}

In this section, we summarize the most related work in the literature.

\paragraph{Equilibrium computation complexity in games.} The computational complexity of finding equilibria has been extensively studied in normal-form games. It is known that for two-player zero-sum normal-form games, Nash equilibria can be computed efficiently using either linear programming \cite{Dantzig1951}  or decentralized no-regret learning algorithms \cite{CBL06}. For general-sum normal-form games, however, computing a Nash equilibrium is known to be \PPAD-complete  \cite{daskalakis2009complexity,CDT06,R16}, even for the two-player case. In contrast, (coarse) correlated equilibria \cite{aumann1987correlated}  can be found via either linear programming  \cite{gilboa1989nash,papadimitriou2008computing} or no-regret learning \cite{CBL06}  efficiently, even in this general-sum setting. The study of the  complexity of equilibria computation in stochastic games has been comparatively scarce.  Since stochastic games generalize normal-form games, the complexity of computing Markov perfect Nash equilibrium in general-sum SGs is thus at least \PPAD-hard. Very recently, \cite{deng2021complexity} confirmed  that computing Markov perfect NE is \PPAD-complete (meaning that the problem of computing perfect CCE is also in \PPAD). Other computational complexity results for stochastic games include the following: determining whether a pure-strategy NE exists in a SG is \PSPACE-hard \cite{conitzer2008new};   determining if there exists a memoryless $\ep$-NE in reachability SGs is \NP-hard \cite{chatterjee2004nash};   in \emph{simple stochastic games} \cite{condon1990algorithms}, a  special class of zero-sum SGs introduced in \cite{Shapley53}, deciding which player has the greater chance of winning  is in  $\NP\cap\coNP$ \cite{condon1992complexity,zwick1996complexity}, and computing an equilibrium is in \textsf{UEOPL}, a subclass of \textsf{CLS}, which is in turn a subclass of \PPAD (see \cite{etessami2010complexity,fearnley2018unique}). %

\paragraph{Multi-agent RL in stochastic games.} Stochastic games \cite{Shapley53} have served as the foundational framework of multi-agent reinforcement learning since \cite{Littman94}. There is a rich literature on multi-agent RL in two-player zero-sum SGs, including the early studies  \cite{Littman94,brafman2002r} as well as more recent ones with finite-sample complexity guarantees \cite{WeiHL17,XieCWY20,ZhangKBY20,SidfordWYY20,BaiJ20,DaskalakisFG20,BaiJY20,LiuTBJ21,cui2022offline,zhong2022pessimistic}. On the general-sum front, $Q$-learning based algorithms, e.g., Nash $Q$-learning  \cite{HuW03} and Friend-or-Foe $Q$-learning \cite{littman2001friend}, have been shown to converge to the Nash equilibrium asymptotically under certain restrictive assumptions. Another variant, Correlated $Q$-learning \cite{greenwald2003correlated}, which also aims to finding a correlated equilibrium in similar spirit to the present work, was shown to converge empirically in several SGs. Related to our findings,  \cite{zinkevich2005cyclic} demonstrated that {\it value-based} RL methods cannot find
{\it stationary} equilibria in arbitrary general-sum SGs, and advocated instead for an alternative {\it nonstationary} equilibrium concept -- cyclic equilibria.  Finally, {decentralized} multi-agent RL has attracted increased attention recently \cite{DaskalakisFG20,sayin2021decentralized,JinLWY21,SongMB21,MaoB21}, due to the fact that decentralized algorithms are more natural, require fewer assumptions, and typically avoid exponential dependence on the number of agents. Most relevant to our paper are the works \cite{JinLWY21,SongMB21,MaoB21}, which developed the \emph{V-learning} algorithm for learning nonstationary (C)CE in general-sum SGs. These algorithms have tighter sample complexity than ours, but the output policies are not Markovian.

\section{On the Gates Used in \Gcircuit}
\label{sec:gcircuit-gates}
The definition of $\ep$-\Gcircuit in \cite{rubinstein2018inapproximability} uses some gates not introduced in Definition \ref{def:gcircuit}, namely $G_{\times}(\zeta | v_1 | v_2)$, $G_=(|v_1 | v_2)$, $G_+(|v_1, v_2 | v_3)$, $G_-(|v_1, v_2 | v_3)$, $G_\vee(|v_1, v_2 | v_3)$,  $G_{\wedge}(|v_1, v_2 | v_3)$, $G_\neg(|v_1|v_2)$, and $G_\gets(\zeta | v)$, for $\zeta \in [0,1]$. However, it is straightforward to see that these gates may be implemented as follows:
\begin{itemize}
\item $G_\times, G_=, G_+, G_-$ may each be implemented using the gate $G_{\times, +}$ (for appropriate choices of $\xi, \zeta$): in particular, we may implement $G_\times(\zeta | v_1 | v_2)$ as $G_{\times, +}(\zeta/2, \zeta/2 | v_1, v_1 | v_2)$, $G_=(|v_1|v_2)$ as $G_{\times, +}(1/2, 1/2 | v_1, v_1 | v_2)$, $G_-(|v_1, v_2 | v_3)$ as $G_{\times, +}(1, -1 | v_1, v_2 | v_3)$, and $G_+$ as $G_{\times, +}(1,1| v_1, v_2 | v_3)$.
\item The gate $G_\gets ( \zeta | v)$, for $\zeta \in [0,1]$, may be implemented using 
$G_\gets( 1 || u)$, $G_{\times, +}(\zeta/2, \zeta/2 | u, u | v)$; since any $\ep$-approximate assignment $\pi$ must satisfy $\pi(u) = 1$, we get that $\pi(v) = \zeta \cdot \pi(u) \pm \ep = \zeta \pm \ep$. 
\item The gate $G_\vee(|v_1, v_2 | v_3)$ may be implemented using the following gates: $G_{\times, +}(1/2, 1/2 | v_1, v_2 | u_1)$, $G_\gets(1| | u_2)$, $G_{\times, +}(1/8, 1/8 | u_2, u_2 | u_3)$, and $G_<(| u_3, u_1 | v_3)$, where $u_1, u_2, u_3$ are supplementary nodes. Any $\ep$-approximate assignment $\pi$ must satisfy $\pi(u_1) = \frac{\pi(v_1) + \pi(v_2)}{2} \pm \ep$, $\pi(u_3) = \frac 14 \pm \ep$. Thus, when $\pi(v_1) = 1 \pm \ep$ or $\pi(v_2) = 1 \pm \ep$, as long as $1 - 2\ep > 1/4 + 2\ep$   (which holds when $\ep < 1/16$), $\pi(v_3) = 1 \pm \ep$. Furthermore, when $\pi(v_1) = 0 \pm \ep$ and $\pi(v_2) = 0 \pm \ep$, again as long as $\ep < 1/16$, we have $\pi(v_3) = 0\pm \ep$.
\item The gate $G_\neg(|v_1|v_2)$ may be implemented using the following gates: $G_\gets(1||u_1)$, $G_{\times, +}(1, -1 | u_1, v_1 | v_2)$, where $u_1$ is a supplementary node. Any $\ep$-approximate assignment $\pi$ must satisfy $\pi(u_1) = 1$ and so $\pi(v_2) = \max\{\pi(u_1) - \pi(v_1), 0\} \pm \ep= 1 - \pi(v_1) \pm \ep$.
\item The gate $G_\wedge$ may be implemented exactly as $G_\vee(|v_1, v_2 | v_3)$ above, except the gate with output node $u_3$ being replaced with $G_{\times, +}(3/8, 3/8 | u_2, u_2 | u_3)$. (Alternatively, we may use the gates $G_\neg$ and $G_\vee$.)
\end{itemize}
We also remark that our requirement that $\pi(v) = \zeta$ for $G_\gets (\zeta || v)$ when $\zeta \in \{0,1\}$ is stronger than that in \cite{rubinstein2018inapproximability}, which allows for error $\ep$, and that the gate $G_{\times, +}$ is not considered in  \cite{rubinstein2018inapproximability}. However, these modifications only make the problem harder. 
Summarizing, the $\ep$-\Gcircuit problem with the set of gates listed above is still \PPAD-complete for some constant $\ep$. %

\section{Adversarial bandit guarantees}
\label{sec:bandit-derivations}
In the context of the adversarial no-regret bandit learning setting described in Section \ref{sec:adv-bandit}, it was shown in \cite[Theorem 1]{neu2015explore} (see also \cite[Theorem 12.1]{lattimore2020bandit}, which is not quite sufficient for us since it requires $T$ to be known ahead of time) that for any $T_0 \in \BN$, $\delta \in (0,1)$, with probability at least $1-\delta$, we have that for all $T \leq T_0$, 
\begin{align}
\max_{b \in \MB} \sum_{t=1}^T \left( \til \ell_t(b_t) - \til\ell_t(b) \right) \leq O \left(\sqrt{TB} \cdot  \log(T_0 B/\delta) \right)\nonumber. 
\end{align}
To obtain Theorem \ref{thm:adv-bandit-external} as a consequence of the above, let $\MF_t$ denote the sigma-field generated by $b_1, \ldots, b_{t+1}, \til \ell_1, \ldots, \til \ell_{t}, \ell_1, \ldots, \ell_{t+1}$. We now note that for each $t$, $\E[\til \ell_t(b_t) | \MF_{t-1}] = \ell_t(b_t)$ and for all $t,b$, $\E[\til \ell_t(b) | \MF_{t-1} ] =\ell_t(b)$. We then apply the Azuma-Hoeffding inequality (followed by a union bound) to each of $\left( \til \ell_t(b_t) - \ell_t(b_t) \right)_{t \in [T_0]}$ and, for all $b \in \MB$, $\left( \til \ell_t(b) - \ell_t(b) \right)_{t \in [T_0]}$, which are martingale difference sequences with respect to the filtration $\MF_t$. %

\end{document}